\documentclass[11pt]{article}

\usepackage[T1]{fontenc}
\usepackage[utf8]{inputenc} % se compili con pdfLaTeX
\usepackage{lmodern}

\usepackage{amsmath,amssymb,amsthm}
\usepackage{algorithm}
\usepackage{algpseudocode}
\usepackage[colorlinks=true,linkcolor=blue,citecolor=blue,urlcolor=blue]{hyperref}
\usepackage[nameinlink,noabbrev]{cleveref}

\usepackage{amsmath}
\usepackage{amssymb}
\usepackage{amsthm}
\usepackage{mathtools}
\usepackage{braket}
\usepackage{relsize}
\usepackage{color,soul}
\usepackage{dsfont}
\usepackage{comment}
\usepackage{indentfirst}
\usepackage{microtype}
\usepackage{braket}
\usepackage{graphicx}
\usepackage[shortlabels]{enumitem}
\usepackage{verbatim}
\usepackage{listings}
\usepackage{color}
\usepackage{inconsolata}
\usepackage{tikz}
\usepackage{pgfplots}
\pgfplotsset{compat=1.18}
\usepackage{booktabs}
\usepackage{caption}
\usepackage{bm}
\usepackage{enumitem}
\usepackage{mathrsfs}
\usepackage{longtable}
\usepackage{environ}
\usepackage{xspace}
\usepackage{setspace}

\newtheorem{theorem}{Theorem}
\newtheorem{proposition}[theorem]{Proposition}
\newtheorem{corollary}[theorem]{Corollary}
\newtheorem{lemma}[theorem]{Lemma}
\newtheorem{remark}[theorem]{Remark}

\newcommand{\numberset}{\mathbb}

\newcommand{\R}{\numberset{R}}

\newcommand{\F}{\mathcal{F}}
\newcommand{\G}{\mathcal{G}}
\newcommand{\Q}{\numberset{Q}}

\renewcommand{\P}{\mathscr{P}}

\newcommand{\E}{\numberset{E}}

\renewcommand{\epsilon}{\varepsilon}
\newcommand{\supp}{\operatorname{supp}}

\DeclareMathOperator{\diag}{diag}
\DeclareMathOperator{\rank}{rank}

\title{Triangular-Reference Schrödinger Bridges for Time Series Generation}
\author{Gabriele Bocchi\thanks{%
  \textit{Arakne S.r.l., Roma}.
  E-mail: \texttt{gabriele.bocchi@arakne.it}.}}
  \date{}

\begin{document}
\maketitle

\begin{abstract}
Schrödinger bridges for time series (SBTS) generate synthetic paths by
projecting, in relative entropy, a Brownian reference onto the path laws that
match the joint distribution of the data on the observation grid. The
Brownian reference, however, fixes the quadratic variation of the generated
paths, which is restrictive when stochastic volatility, correlated noise, or
rank-deficient covariance structures must be reproduced. We introduce
\emph{Triangular-Reference Schrödinger Bridges for Time Series} (TR-SBTS),
which keeps the entropy-projection backbone of SBTS but replaces the Brownian
reference by a triangular, volatility-informed, intervalwise frozen reference
on a state augmented with latent covariance descriptors. The construction
remains a single entropy projection on the augmented state: the minimiser is
the \(h\)-transform of the reference, and on each frozen interval the optimal
drift has the logarithmic-gradient form \(b^\star(t,x)=A\,\nabla\log H(t,x)\),
intrinsic to the active covariance directions when the frozen covariance
\(A\) is degenerate. We prove stability of the frozen approximation and
consistency of the associated regularised kernel estimators, describe a
reference-aware Nadaraya--Watson implementation of the conditional
next-increment law, and evaluate the construction on numerical experiments.
\end{abstract}

\section{Introduction}
\label{sec:introduction}

The generation of realistic synthetic time series arises whenever observed
data are too scarce or too sensitive to be used directly. In finance, which
motivates this work, synthetic trajectories serve for stress testing,
scenario generation, data augmentation, and the assessment of downstream
decision rules such as hedging procedures.

Generating time series is harder than generating static data. Matching
one-time marginals, or even a static joint distribution, is not enough: a
useful generator must preserve temporal dependence, cross-sectional
dependence, volatility and covariance structure, and pathwise variability.
Classical parametric models offer interpretable dynamics but are exposed to
misspecification, which has motivated data-driven generative models
\cite{goodfellow2014generative, kingma2013autoencoding, rezende2015variational,
arjovsky2017wasserstein, song2021score, ho2020denoising} and, for sequential
data, specialised architectures and metrics \cite{yoon2019time,
wiese2020quant, xu2020cot, kidger2021neural, buehler2020data,
chevyrev2016primer, choudhary2023funvol}.

Schrödinger bridge methods provide a natural framework for this problem. The
classical Schrödinger bridge seeks the path law closest, in relative entropy,
to a reference measure under prescribed marginal constraints; in diffusion
settings the solution adds a drift to the reference dynamics, with the energy
of the control measured by relative entropy \cite{leonard2014survey,
daipra1991reciprocal, cuturi2013sinkhorn, pavon2021schrodinger,
debortoli2021diffusion, shi2023diffusion, gushchin2024light}. The Schrödinger
Bridge for Time Series (SBTS) of \cite{hamdouche2023generative} adapts this
idea to sequential data: rather than two endpoint marginals, it constrains
the full joint law of the observations on a fixed time grid, and the
minimiser is a controlled diffusion whose path-dependent drift is represented
by backward Schrödinger potentials and estimated from data. SBTS thus targets
the joint time-series distribution directly while producing a continuous-time
interpolation between observation dates.

The benchmarks of \cite{alouadi2025robust} support SBTS as a competitive and
robust generator, but they also expose a structural limitation: the Brownian
reference fixes the quadratic variation of the generated paths. For financial
series this is restrictive, since stochastic volatility, correlated noise,
and possibly rank-deficient covariance structures are part of the phenomenon
to be reproduced rather than secondary modelling details. Extensions in this
direction include the jump-diffusion reference of \cite{demarco2026sbjts},
which captures heavy tails and abrupt moves on càdlàg path space, and the
Schrödinger--Bass bridge for time series of \cite{alouadi2026sbbts}, which
optimises jointly over drift and volatility.

The present paper takes a conservative route: the entropy minimisation
problem of SBTS is kept unchanged, and only the reference law is modified.
TR-SBTS preserves the SBTS entropy-projection backbone but replaces the
Brownian reference by a triangular, volatility-informed, intervalwise frozen
reference. The reference lives on an augmented state
\[
Z=(Z^0,Z^1,\ldots,Z^L),
\]
where \(Z^0\) is the principal time series, such as prices or returns, and
the lower levels \(Z^1,\ldots,Z^L\) encode volatility or covariance
descriptors constructed on the coarse observation grid. Given the joint
coarse-grid law \(\mu=\mathcal L(Z_{t_0},\ldots,Z_{t_N})\), the problem is
the same as in SBTS: find the path law closest to the reference in relative
entropy among those matching \(\mu\), the constraint being joint across time
and across the levels.

The reference is \emph{triangular}: on each interval \((t_i,t_{i+1}]\) its
transition kernels are adapted to the common coarse past
\(\mathcal G_i=\sigma(Z_{t_0},\ldots,Z_{t_i})\) and, ordering the components
from the deepest descriptor to the observed process, each level evolves
conditionally on the common past and on the lower levels already constructed.
The lower levels are thus generated first, and on each interval their
realisation determines the frozen Gaussian reference, possibly with
degenerate covariance, used by the levels above. The frozen covariance may
encode volatility information learned from high-frequency data, which thereby
enriches the coarse-grid state rather than refining the output grid.

Triangularity does not split the target into independent marginal bridges.
The constraint remains the joint law of the augmented process; the triangular
reference only makes the conditional structure of the single global entropy
projection explicit, through the chain rule for relative entropy, which in
two-level notation reads
\[
H(P^{X,Y}\mid R^{X,Y})
=
H(P^Y\mid R^Y)
+
\mathbb E_{P^Y}\bigl[H(P^{X\mid Y}\mid R^{X\mid Y})\bigr].
\]
The augmented bridge can accordingly be generated bottom-up, and the price
bridge is the top component of one augmented entropy projection.

Under the finite-entropy condition \(\mu\ll\mu_T^R\),
\(H(\mu\mid\mu_T^R)<\infty\), where \(\mu_T^R\) is the coarse-grid law of the
reference, the minimiser is the usual \(h\)-transform
\(dP^\star/dR=(d\mu/d\mu_T^R)(Z_{t_0},\ldots,Z_{t_N})\). The local
consequence is the main structural result of the paper: conditionally on the
coarse past and on the lower-level environment, each conditional problem is a
frozen Gaussian Schrödinger bridge, possibly degenerate, and its optimal
drift has the logarithmic-gradient form
\[
b^\star(t,x)=A\,\nabla\log H(t,x),
\]
where \(A\) is the frozen covariance and \(H\) a backward heat potential on
the affine leaf generated by the active covariance directions; when \(A\) is
rank-deficient the formula is understood intrinsically on those directions.
This is the degenerate analogue of the SBTS heat-drift formula.

Two further results complete the analysis. First, the frozen reference
approximates an ideal reference whose covariance varies continuously inside
each interval: we prove stability of the corresponding bridges, with the
approximation error measured by cumulative covariances on the active leaf, so
that refining the volatility grid recovers the continuous formulation.
Second, we prove convergence of regularised kernel estimators of the backward
potentials and the drifts, giving
consistency guarantees for the full approximation chain. At the computational
level, the only object to estimate is the conditional next-increment law,
approximated by a single reference-aware Nadaraya--Watson estimator acting on
finite-dimensional causal summaries of the augmented past.

The paper is organised as follows.
Section~\ref{sec:reference-based-sbts-projection} develops the variational
backbone: the entropy projection under a grid-law constraint and its
characterisation as a coherent family of conditional projections.
Section~\ref{sec:exact-intervalwise-dynamics} introduces the triangular
frozen reference, solves the resulting leaf problems, deriving the
logarithmic-gradient drift on possibly degenerate Gaussian leaves, and
proves the stability of the frozen approximation.
Section~\ref{sec:analytic-statistical-terminal-datum} establishes the
statistical approximation and the main convergence theorem.
Section~\ref{sec:computational-architecture} describes the computational
estimator, and Section~\ref{sec:numerical-experiments} reports the numerical
experiments. Appendix~\ref{sec:appendix-statistical-approximation} contains
the proof of the statistical theorem,
Appendix~\ref{sec:appendix-predictive-energy-score} the predictive scoring
procedure, and
Appendix~\ref{sec:appendix-entropic-reference-selection-and-validation} the
entropic criterion used to select and validate the reference families.

\section{Reference-based entropy projection and conditional structure}
\label{sec:reference-based-sbts-projection}

We isolate the variational backbone used throughout: the entropy minimisation
on path space under a grid-law constraint, with respect to an arbitrary
reference. Subsection~\ref{subsec:abstract-entropy-projection} identifies the
minimiser of the projection, and
Subsection~\ref{subsec:conditional-characterization} proves that the global
projection is stable under conditioning---the global optimiser is equivalently
characterised as a coherent family of conditional entropy projections. Both
results are purely variational and reference-agnostic: no specific structure of the reference is used or assumed.
The specific reference that makes the conditional problems explicit and exactly solvable is introduced
only in Section~\ref{sec:exact-intervalwise-dynamics}.

\subsection{Abstract entropy projection under a grid-law constraint}
\label{subsec:abstract-entropy-projection}

\medskip
Let
\[
0=t_0<t_1<\dots<t_N=T,
\qquad
\Omega:=C([0,T];\R^d),
\]
and let \(X=(X_t)_{0\le t\le T}\) be the canonical process on \(\Omega\),
with canonical filtration \(\F=(\F_t)_{0\le t\le T}\). We denote by
\[
\Pi:\Omega\to(\R^d)^{N+1},
\qquad
\Pi(\omega):=\bigl(X_{t_0}(\omega),\dots,X_{t_N}(\omega)\bigr)
\]
the observation map at the coarse times.

Let \(R\in\P(\Omega)\) be an arbitrary reference probability law on path
space, and denote by
\[
r:=R\circ\Pi^{-1}
\]
its joint law on the coarse grid. Let
\(\mu\in\P((\R^d)^{N+1})\) be a target law on the same grid, and assume
\begin{equation}
\mu\ll r,
\qquad
H(\mu\mid r)<\infty.
\label{eq:finite-entropy-assumption}
\end{equation}
We refer to the entropy minimisation problem
\begin{equation}
V(\mu)
:=
\inf\Bigl\{
H(P\mid R):
P\in\P(\Omega),\;
P\circ\Pi^{-1}=\mu
\Bigr\}
\label{eq:general-sbts-problem}
\end{equation}
as the \emph{reference-based SBTS entropy projection} problem associated with
the reference \(R\) and the constraint \(\mu\). The following theorem
identifies its value and unique minimiser (i.e. the \emph{entropy projection}) in closed form, under no
structural assumption on \(R\) beyond
\eqref{eq:finite-entropy-assumption}.

\begin{theorem}[Variational identification of the entropy projection]
\label{thm:general-sbts-main}
Under assumption \eqref{eq:finite-entropy-assumption}, set
\[
G:=\frac{d\mu}{d r}.
\]
Then the following statements hold.

\begin{enumerate}
\item[(i)] The value of problem \eqref{eq:general-sbts-problem} is
\[
V(\mu)=H(\mu\mid r).
\]

\item[(ii)] There exists a unique minimiser \(P^*\in\P(\Omega)\), given
by the \(h\)-transform
\[
\frac{dP^*}{dR}=G(\Pi),
\qquad
P^*\circ\Pi^{-1}=\mu.
\]

\item[(iii)] The density process of \(P^*\) with respect to \(R\) is
the positive \(R\)-martingale
\[
M_t:=E_R\!\left[G(\Pi)\mid\F_t\right],
\qquad 0\le t\le T,
\]
so that \(M_T=dP^*/dR\) and \(P^*=M_T\,R\).
\end{enumerate}
\end{theorem}

\begin{proof}
Set
\[
G:=\frac{d\mu}{d r},
\qquad
P^*:=G(\Pi)\,R.
\]
Since \(E_R[G(\Pi)]=1\), the measure \(P^*\) is well defined. For every
bounded measurable \(\varphi:(\R^d)^{N+1}\to\R\),
\[
E_{P^*}[\varphi(\Pi)]
=
E_R[G(\Pi)\varphi(\Pi)]
=
\int_{(\R^d)^{N+1}}\varphi(\mathbf x)\,\mu(d\mathbf x),
\]
hence \(P^*\circ\Pi^{-1}=\mu\), so \(P^*\) is feasible. Moreover,
\[
H(P^*\mid R)
=
E_R[G(\Pi)\log G(\Pi)]
=
H(\mu\mid r),
\]
so
\[
V(\mu)\le H(\mu\mid r).
\]

Now let \(P\) be any feasible probability measure. If
\(H(P\mid R)=\infty\), there is nothing to prove. Assume then that
\(H(P\mid R)<\infty\), and write
\[
L:=\frac{dP}{dR}.
\]
Since \(P\circ\Pi^{-1}=\mu\), one has
\[
E_P[\log G(\Pi)]
=
\int_{(\R^d)^{N+1}}\log G(\mathbf x)\,\mu(d\mathbf x)
=
H(\mu\mid r),
\]
and
\[
E_P[\log L]=H(P\mid R).
\]
Finally,
\[
1
=
E_R[G(\Pi)]
=
E_P\!\left[\frac{G(\Pi)}{L}\right]
=
E_P\!\left[\exp\!\bigl(\log G(\Pi)-\log L\bigr)\right].
\]
By Jensen's inequality,
\[
1
\ge
\exp\!\left(E_P\!\bigl[\log G(\Pi)-\log L\bigr]\right)
=
\exp\!\left(H(\mu\mid r)-H(P\mid R)\right),
\]
and therefore \(H(P\mid R)\ge H(\mu\mid r)\). This proves
\(V(\mu)=H(\mu\mid r)\). If equality holds, then equality holds in
Jensen's inequality. Since the exponential function is strictly convex,
there exists a constant \(c\in\R\) such that
\[
\log G(\Pi)-\log L=c
\qquad P\text{-a.s.}
\]
Taking expectation under \(P\), and using the equality of the two
entropies, yields \(c=0\). Hence \(L=G(\Pi)\) \(P\)-a.s. and
\(P=P^*\), proving uniqueness. The martingale representation in (iii)
is immediate from the definition \(M_t=E_R[G(\Pi)\mid\F_t]\).
\end{proof}

\subsection{Temporal conditional optimality and triangular references}
\label{subsec:conditional-characterization}

The entropy projection of Theorem~\ref{thm:general-sbts-main} is a
global minimisation problem on path space, but its optimality can be read
locally along the observation grid. From now on the canonical process is
denoted by \(Z\) rather than \(X\): nothing changes in the setting of
Subsection~\ref{subsec:abstract-entropy-projection}, but the notation
anticipates the augmented state \(Z=(X,Y)\) on which the triangular
reference is built later in this subsection.

Let
\begin{align*}
  \Pi_i(\omega)&:=(Z_{t_0}(\omega),\ldots,Z_{t_i}(\omega)) =: \mathbf Z_i(\omega),\\
  \Pi_{> i}(\omega)&:=(Z_{t_{i+1}}(\omega),\ldots,Z_{t_N}(\omega)) =: \mathbf Z_{>i}(\omega).
\end{align*}
We also write \(\omega_{\le i}\) for the restriction of the path to
\([0,t_i]\), and \(\omega_{>i}\) for the future path on \((t_i,T]\). 

We assume
that the reference is grid-compatible in the following sense: for every
\(i=0,\ldots,N-1\), the conditional law under \(R\) of the future path given
\(\mathcal F_{t_i}\) depends on the past only through \(\Pi_i\). Thus there
exist probability kernels
\[
R_{>i}^{\mathbf z_i} \in \mathcal P(\Omega_{>i}),
\qquad \mathbf z_i\in(\mathbb R^{d_Z})^{i+1},
\]
such that
\[
R(d\omega_{>i}\mid\mathcal F_{t_i})
=
R_{>i}^{\Pi_i(\omega)}(d\omega_{>i}).
\]
For every \(\mathbf z_i\), we denote the corresponding reference law of the future grid by
\[
r_{>i}^{\mathbf z_i}
:=
R_{>i}^{\mathbf z_i}\circ \Pi_{>i}^{-1}.
\]
Finally, disintegrate the target grid law as
\[
\mu(d\mathbf z_N)
=
\mu_{i}(d\mathbf z_i)\,
\mu_{>i}^{\mathbf z_i}(d\mathbf z_{>i}).
\]

For fixed \(\mathbf z_i\), define the future conditional entropy projection
problem
\begin{equation}
  \label{eq:future-conditional-problem}
V_{>i}(\mathbf z_i)
:=
\inf\Bigl\{
H(Q\mid R_{>i}^{\mathbf z_i})
:
Q\circ \Pi_{>i}^{-1}
=
\mu_{>i}^{\mathbf z_i}
\Bigr\}.
\end{equation}
By Theorem~\ref{thm:general-sbts-main}, applied to the conditional reference
\(R_{>i}^{\mathbf z_i}\), its value is
\[
V_{>i}(\mathbf z_i)
=
H\bigl(\mu_{>i}^{\mathbf z_i}\mid r_{>i}^{\mathbf z_i}\bigr),
\]
and, whenever this entropy is finite, the unique minimiser is
\[
\frac{dQ_{>i}^{\star,\mathbf z_i}}{dR_{>i}^{\mathbf z_i}}
=
\frac{d\mu_{>i}^{\mathbf z_i}}
     {dr_{>i}^{\mathbf z_i}}
\bigl(\Pi_{>i}\bigr).
\]

\begin{proposition}[Temporal conditional characterisation of the entropy projection]
\label{prop:temporal-conditional-characterization}
Assume
\[
\mu\ll r,
\qquad
H(\mu\mid r)<\infty ,
\]
and assume that \(R\) is grid-compatible in the sense above. Let \(P\) be an
admissible law, namely
\[
P\circ\Pi^{-1}=\mu .
\]
Fix \(i\in\{0,\ldots,N-1\}\), and disintegrate
\[
P(d\omega)
=
P_{\le i}(d\omega_{\le i})\,
P_{>i}^{\omega_{\le i}}(d\omega_{>i}),
\qquad
R(d\omega)
=
R_{\le i}(d\omega_{\le i})\,
R_{>i}^{\Pi_i(\omega)}(d\omega_{>i}).
\]
Then \(P\) minimises \(H(\cdot\mid R)\) under the full grid-law constraint
\(P\circ\Pi^{-1}=\mu\) if and only if the following two conditions hold.

\begin{enumerate}[i)]
\item The past marginal \(P_{\le i}\) minimises the truncated entropy
projection problem
\[
\inf\Bigl\{
H(M\mid R_{\le i})
:
M\circ \Pi_i^{-1}=\mu_i
\Bigr\}.
\]
\item For \(P_{\le i}\)-a.e. past \(\omega_{\le i} \), \(P_{>i}^{\omega_{\le i}}\) minimises
\[
\inf\Bigl\{
H(Q\mid R_{>i}^{\Pi_i(\omega_{\le i})})
:
Q\circ\Pi_{>i}^{-1}
=
\mu_{>i}^{\Pi_i(\omega_{\le i})}
\Bigr\}.
\]
\end{enumerate}
\end{proposition}

\begin{proof}
By the chain rule for relative entropy,
\begin{equation}
  \label{eq:chain-rule-entropy}
H(P\mid R)
=
H(P_{\le i}\mid R_{\le i})
+
\int
H\bigl(
P_{>i}^{\omega_{\le i}}
\mid
R_{>i}^{\Pi_i(\omega_{\le i})}
\bigr)\,
P_{\le i}(d\omega_{\le i}).
\end{equation}
Since \(P\circ\Pi^{-1}=\mu\), the past marginal satisfies
\[
P_{\le i}\circ \Pi_i^{-1}=\mu_i.
\]
Therefore, by Theorem~\ref{thm:general-sbts-main} applied on the truncated
time interval \([0,t_i]\),
\[
H(P_{\le i}\mid R_{\le i})
\ge
H(\mu_i\mid r_i),
\]
with equality if and only if \(P_{\le i}\) is the truncated entropy
projection.

Moreover, we note that
\[
P_{\le i}(d\omega_{\le i}) = P_{\le i}(d\omega_{\le i} \mid \Pi_i(\omega_{\le i}))\, \mu_i(d\Pi_i(\omega_{\le i})).
\]
Let
\[
P_{>i}^{\mathbf z_i}
:=
P(\omega_{>i}\in\cdot\mid \Pi_i=\mathbf z_i)
\]
be the conditional future law given the coarse grid past.
Hence, by the reference grid-compatibility and Jensen's inequality, we obtain
\[
\int
H\bigl(
P_{>i}^{\omega_{\le i}}
\mid
R_{>i}^{\Pi_i(\omega_{\le i})}
\bigr)\,
P_{\le i}(d\omega_{\le i}) 
\ge 
\int
H\bigl(
P_{>i}^{\Pi_i(\omega_{\le i})}
\mid
R_{>i}^{\Pi_i(\omega_{\le i})}
\bigr)\,
\mu_i(d\Pi_i(\omega_{\le i}))
\]
Equality in this Jensen step is precisely the condition that the
conditional future law \(P_{>i}^{\omega_{\le i}}\) depends on the past
only through the coarse grid variable \(\Pi_i\), and is therefore the
fibrewise entropy projection.

Moreover, for \(P_{\le i}\)-a.e. \(\omega_{\le i}\), the conditional future
law given the coarse grid past satisfies
\[
P_{>i}^{\Pi_i(\omega_{\le i})}\circ \Pi_{>i}^{-1}
=
\mu_{>i}^{\Pi_i(\omega)} .
\]
Hence the conditional form of Theorem~\ref{thm:general-sbts-main} gives
\[
H\bigl(
P_{>i}^{\Pi_i(\omega_{\le i})}
\mid
R_{>i}^{\Pi_i(\omega_{\le i})}
\bigr)
\ge
H\bigl(
\mu_{>i}^{\Pi_i(\omega_{\le i})}
\mid
r_{>i}^{\Pi_i(\omega)}
\bigr),
\]
with equality if and only if
\( P_{>i}^{\Pi_i(\omega_{\le i})}\) is the optimal solution of \eqref{eq:future-conditional-problem} with \(\mathbf z_i=\Pi_i(\omega_{\le i})\).

Hence, by \eqref{eq:chain-rule-entropy}, we obtain
\begin{align*}
H(P\mid R)
&\ge H(\mu_i\mid r_i)
+
\int
H\bigl(
\mu_{>i}^{\mathbf z_i}
\mid
r_{>i}^{\mathbf z_i}
\bigr)\,
\mu_i(d\mathbf z_i)
\end{align*}
By the chain rule for the grid laws, the last right-hand side is precisely
\[
H(\mu\mid r),
\]
which is the value of the global entropy projection by
Theorem~\ref{thm:general-sbts-main}. Thus \(P\) is globally optimal if and
only if both lower bounds above are equalities. This is exactly the pair of
conditions stated in the proposition.
\end{proof}

The proposition says that the global optimiser is not merely compatible with
conditional problems: it is characterised by them. At every time \(t_i\), the
past part is already the entropy projection for the truncated grid, while the
future conditional law is the entropy projection of the remaining target
conditional law, with respect to the reference conditional on the realised
past. In this precise sense, the global solution minimises all temporal
conditional problems simultaneously.

The one-step version is the form used in the construction below. Writing
\[
\mu_i^{\mathbf z_i}(dz_{i+1})
:=
\mu(dz_{i+1}\mid \mathbf z_i),
\qquad
r_i^{\mathbf z_i}(dz_{i+1})
:=
R(dz_{i+1}\mid \mathbf z_i),
\]
the conditional problem on the interval \((t_i,t_{i+1}]\) is
\[
V_i(\mathbf z_i)
:=
\inf\Bigl\{
H(Q\mid R_i^{\mathbf z_i})
:
Q\circ Z_{t_{i+1}}^{-1}
=
\mu_i^{\mathbf z_i}
\Bigr\},
\]
and its unique minimiser is
\[
\frac{dQ_i^{\star,\mathbf z_i}}{dR_i^{\mathbf z_i}}
=
\frac{d\mu_i^{\mathbf z_i}}{dr_i^{\mathbf z_i}}
\bigl(Z_{t_{i+1}}\bigr).
\]
The global bridge is therefore obtained by concatenating these one-step
conditional bridges along the realised grid path.

We now explain how this principle leads to the triangular reference. Suppose
that the augmented state is
\[
Z=(X,Y),
\]
and fix a realised past \(\mathbf z_i\). A one-step triangular reference is a
grid-compatible one-step reference such that, on the interval
\((t_i,t_{i+1}]\),
\[
R_i^{\mathbf z_i}(d\eta,d\xi)
=
R_{Y,i}^{\mathbf z_i}(d\eta)\,
R_{X,i}^{\mathbf z_i,\eta_{t_{i+1}}}(d\xi),
\]
where \(\eta\) denotes a \(Y\)-path and \(\xi\) an \(X\)-path on the current
interval. Thus the \(X\)-reference does not depend on the whole \(Y\)-path,
but only on the realised endpoint
\[
y_{i+1}=\eta_{t_{i+1}}.
\]
At the grid level this gives
\[
r_i^{\mathbf z_i}(dy_{i+1},dx_{i+1})
=
r_{Y,i}^{\mathbf z_i}(dy_{i+1})\,
r_{X,i}^{\mathbf z_i,y_{i+1}}(dx_{i+1}).
\]
Disintegrate the target conditional law in the same order:
\[
\mu_i^{\mathbf z_i}(dy_{i+1},dx_{i+1})
=
\mu_{Y,i}^{\mathbf z_i}(dy_{i+1})\,
\mu_{X,i}^{\mathbf z_i,y_{i+1}}(dx_{i+1}).
\]

\begin{proposition}[Triangular conditional characterisation]
\label{prop:triangular-conditional-characterization}
Fix \(\mathbf z_i\) and assume that the one-step reference
\(R_i^{\mathbf z_i}\) is triangular in the sense above. A one-step admissible
law \(Q\), satisfying
\[
Q\circ (Y_{t_{i+1}},X_{t_{i+1}})^{-1}
=
\mu_i^{\mathbf z_i},
\]
minimises
\[
H(Q\mid R_i^{\mathbf z_i})
\]
if and only if the following two conditions hold.

\begin{enumerate}[i)]
\item Its \(Y\)-marginal \(Q_Y\) minimises
\[
\inf\Bigl\{
H(M\mid R_{Y,i}^{\mathbf z_i})
:
M\circ Y_{t_{i+1}}^{-1}
=
\mu_{Y,i}^{\mathbf z_i}
\Bigr\}.
\]

\item Disintegrating
\[
Q(d\eta,d\xi)
=
Q_Y(d\eta)\,Q_X^\eta(d\xi),
\]
we have, for \(Q_Y\)-a.e. \(\eta\), with
\(y_{i+1}=\eta_{t_{i+1}}\),
\[
Q_X^\eta
=
Q_{X,i}^{\star,\mathbf z_i,y_{i+1}},
\]
where \(Q_{X,i}^{\star,\mathbf z_i,y_{i+1}}\) is the unique minimiser of
\[
\inf\Bigl\{
H(N\mid R_{X,i}^{\mathbf z_i,y_{i+1}})
:
N\circ X_{t_{i+1}}^{-1}
=
\mu_{X,i}^{\mathbf z_i,y_{i+1}}
\Bigr\}.
\]
\end{enumerate}
\end{proposition}

\begin{proof}
The proof is the same entropy-disintegration argument as in
Proposition~\ref{prop:temporal-conditional-characterization}, now applied in
the vertical variable. The chain rule gives
\[
H(Q\mid R_i^{\mathbf z_i})
=
H(Q_Y\mid R_{Y,i}^{\mathbf z_i})
+
\int
H\bigl(
Q_X^\eta
\mid
R_{X,i}^{\mathbf z_i,\eta_{t_{i+1}}}
\bigr)\,
Q_Y(d\eta).
\]
The first term is bounded below by
\[
H\bigl(
\mu_{Y,i}^{\mathbf z_i}
\mid
r_{Y,i}^{\mathbf z_i}
\bigr).
\]
For the second term, we note that 
\[
  Q_Y(d\eta) = Q_Y(d\eta\mid Y_{t_{i+1}}=y_{i+1})\,\mu_{Y,i}^{\mathbf z_i}(dy_{i+1}),
\]
by the admissibility condition \(Q\circ Y_{t_{i+1}}^{-1}=\mu_{Y,i}^{\mathbf z_i}\). 
Therefore, using Jensen's inequality, we obtain
\[
\int
H\bigl(
Q_X^\eta
\mid
R_{X,i}^{\mathbf z_i,\eta_{t_{i+1}}}
\bigr)\,
Q_Y(d\eta)
\ge
\int
H\bigl(
Q_X^{\eta_{t_{i+1}}}
\mid
R_{X,i}^{\mathbf z_i,\eta_{t_{i+1}}}
\bigr)\,
\mu_{Y,i}^{\mathbf z_i}(d\eta_{t_{i+1}}).
\]
where \(Q_X^{\eta_{t_{i+1}}}\) is the conditional law of \(Q_X^\eta\) given \(Y_{t_{i+1}}=\eta_{t_{i+1}}\).
Equality in this Jensen step means that, after conditioning on the
lower-level endpoint \(y_{i+1}\), the upper-level conditional law
\(Q_X^\eta\) depends on \(\eta\) only through \(y_{i+1}\), and is
therefore the corresponding fibrewise entropy projection.
Hence, since the conditional entropy
projection optimal value on the \(X\)-path space is \(H\bigl(
\mu_{X,i}^{\mathbf z_i,y_{i+1}}
\mid
r_{X,i}^{\mathbf z_i,y_{i+1}}
\bigr)\), 
\[
H(Q\mid R_i^{\mathbf z_i})
\ge
H\bigl(
\mu_{Y,i}^{\mathbf z_i}
\mid
r_{Y,i}^{\mathbf z_i}
\bigr)
+
\int
H\bigl(
\mu_{X,i}^{\mathbf z_i,y_{i+1}}
\mid
r_{X,i}^{\mathbf z_i,y_{i+1}}
\bigr)\,
\mu_{Y,i}^{\mathbf z_i}(dy_{i+1}).
\]
By the entropy chain rule for the one-step grid laws, the right-hand side is
\[
H\bigl(\mu_i^{\mathbf z_i}\mid r_i^{\mathbf z_i}\bigr),
\]
which is the value of the one-step entropy projection. Equality holds if and
only if equality holds in the \(Y\)-projection and, \(Q_Y\)-a.e., in the
conditioned \(X\)-projections. This gives the stated characterisation.
\end{proof}

The two propositions identify the recursive structure of the entropy
projection. The first one is temporal: after the grid past
\(\mathbf z_i\) has been realised, the remaining part of the global optimiser
is the entropy projection of the conditional future law
\(\mu_{>i}^{\mathbf z_i}\) with respect to the conditional reference
\(R_{>i}^{\mathbf z_i}\). The second one is triangular: on each interval, the
lower-level component \(Y\) is projected first, and, once its endpoint
\(y_{i+1}\) has been realised, the upper-level component \(X\) is projected
with respect to the corresponding conditional reference
\(R_{X,i}^{\mathbf z_i,y_{i+1}}\).

Thus the global bridge is characterised by a family of local conditional
bridges, organised first in time and then along the triangular hierarchy of
the augmented state. The rest of the paper exploits this reduction in the
specific references used for TR-SBTS: the lower-level bridge determines the
realised environment, and the upper-level bridge is solved conditionally on
the associated frozen Gaussian leaf. Section~\ref{sec:exact-intervalwise-dynamics}
derives the corresponding intervalwise dynamics.

\section{Triangular frozen references and leafwise dynamics}
\label{sec:exact-intervalwise-dynamics}

Section~\ref{sec:reference-based-sbts-projection} reduced the global entropy
projection to a coherent family of conditional one-step problems: temporally
along the coarse grid
(Proposition~\ref{prop:temporal-conditional-characterization}) and vertically
within each interval, where the lower level is projected first and the upper
level is projected conditionally on the realised lower-level endpoint
(Proposition~\ref{prop:triangular-conditional-characterization}). Once the
coarse past and the lower-level endpoint are fixed, what remains is a one-step
conditional entropy projection against an explicit reference: a \emph{leaf
problem}. This section solves these leaf problems in closed form for the
triangular frozen Gaussian reference used in the paper.
Subsection~\ref{subsec:conditional-to-leaf} introduces the reference and
formulates the leaf problem; Subsections~\ref{subsec:frozen-gaussian-leaves}
and~\ref{subsec:intervalwise-backward} develop the intrinsic geometry of
possibly degenerate Gaussian leaves and prove the leafwise $h$-transform
lemma, which yields the backward heat potential and the logarithmic-gradient
drift; Subsection~\ref{subsec:intervalwise-backward-local} identifies the
resulting local optimiser with the conditional optimiser required by
Section~\ref{sec:reference-based-sbts-projection}; and
Subsection~\ref{subsec:volatility-informed-freezing} explains in which sense
the frozen reference approximates an ideal volatility-informed reference.

\subsection{The frozen leaf problem}
\label{subsec:conditional-to-leaf}

\paragraph{The triangular frozen reference.}
Let \(m_0\in\P(\R^d)\) and let \(\mathcal Y\) denote the latent state space of
Subsection~\ref{subsec:volatility-informed-freezing}. For each
\(i=0,\dots,N-1\), let
\[
\sigma_i:(\R^d)^{i+1}\times\mathcal Y\to\R^{d\times d}
\]
be a Borel map and set, for \(\mathbf x_i=(x_0,\dots,x_i)\in(\R^d)^{i+1}\) and
\(y_{i+1}\in\mathcal Y\),
\[
A_i(\mathbf x_i,y_{i+1})
:=\sigma_i(\mathbf x_i,y_{i+1})\,\sigma_i(\mathbf x_i,y_{i+1})^\top
\in\mathbb S_+^d .
\]
The reference \(R\in\P(\Omega)\) is the law under which \(X_0\sim m_0\) and,
for each \(i\), conditionally on the coarse past
\(\mathbf X_i:=(X_0,X_{t_1},\dots,X_{t_i})\) and on the realised lower-level
endpoint \(y_{i+1}\), the process on \((t_i,t_{i+1}]\) is the frozen Gaussian
martingale
\[
dX_t=\sigma_i(\mathbf X_i,y_{i+1})\,dW_t^R,
\qquad t\in(t_i,t_{i+1}],
\]
for some Brownian motion \(W^R\); equivalently, \(X\) is an \(R\)-martingale
with predictable bracket \(A(dt)=A_i(\mathbf X_i,y_{i+1})\,dt\) on
\((t_i,t_{i+1}]\), and \(R\) is obtained by concatenating the Gaussian
transition kernels of these frozen intervalwise diffusions, conditioned on the
latent realisations \((y_{i+1})_{i=0}^{N-1}\). The matrix \(A_i\) may be
degenerate, in which case the process evolves on the affine \emph{active leaf}
\[
X_{t_i}+\operatorname{Ran}A_i .
\]
Triangularity enters only here: the lower level is generated first, and its
realised endpoint \(y_{i+1}\) is then a known parameter of the upper-level
reference, which thereby becomes an explicit conditional Gaussian.

Throughout this section the target \(\mu\) has first marginal \(m_0\) and
satisfies the finite-entropy condition \eqref{eq:finite-entropy-assumption}
conditionally on the latent realisation, and every \(i\)-indexed object
carries the conditioning \((\mathbf x_i,y_{i+1})\), often suppressed from the
notation.

\paragraph{The leaf problem.}
Write \(\Delta X_{i+1}:=X_{t_{i+1}}-X_{t_i}\) for the next increment and
define the conditional one-step laws
\begin{align*}
\eta_i^\mu(\cdot)
&:=\mu\bigl(\Delta X_{i+1}\in\cdot\mid\mathbf X_i=\mathbf x_i,\,Y=y_{i+1}\bigr),
\\
k_i^R(\cdot)
&:=R\bigl(\Delta X_{i+1}\in\cdot\mid\mathbf X_i=\mathbf x_i,\,Y=y_{i+1}\bigr),
\end{align*}
the target and reference conditional next-increment laws; under the frozen
reference, \(k_i^R\) is the centred Gaussian on \(\operatorname{Ran}A_i\) with
covariance \((t_{i+1}-t_i)A_i\). The \emph{leaf problem} on \((t_i,t_{i+1}]\)
is the conditional one-step entropy projection of
Proposition~\ref{prop:triangular-conditional-characterization}\,(ii): minimise
the relative entropy with respect to the conditional reference
\(R_{X,i}^{\mathbf z_i,y_{i+1}}\) over one-step laws whose next-increment law
is \(\eta_i^\mu\). By Theorem~\ref{thm:general-sbts-main} applied to the
conditional reference, its unique solution is the \(h\)-transform of
\(R_{X,i}^{\mathbf z_i,y_{i+1}}\) by the \emph{terminal tilt}
\[
f_i:=\frac{d\eta_i^\mu}{dk_i^R}.
\]
The dynamics of this \(h\)-transform are governed by the backward
\emph{heat potential}
\[
H_i(t,x)
=
E_R\!\bigl[f_i(\Delta X_{i+1})\mid X_t=x,\ \mathbf x_i,\ y_{i+1}\bigr]
\]
through the logarithmic-gradient drift
\[
b_i^*(t,x)
=
A_i\,\nabla\log H_i(t,x)
=
\sum_{j=1}^{r_i}\lambda_{i,j}\,D_{q_{i,j}}\log H_i(t,x)\,q_{i,j},
\]
where \(r_i=\rank A_i\) and \((\lambda_{i,j},q_{i,j})_{j\le r_i}\) are the
positive eigenpairs of \(A_i\).

\subsection{Intrinsic Gaussian geometry on a degenerate leaf}
\label{subsec:frozen-gaussian-leaves}

We fix a single covariance matrix and record the intrinsic objects attached to
it; in the application, \(A=A_i(\mathbf x_i,y_{i+1})\) for a fixed value of
the conditioning.

Let \(A\in\mathbb S_+^d\), let
\[
r:=\rank(A),
\qquad
\lambda_1,\dots,\lambda_r>0,
\qquad
q_1,\dots,q_r\in\R^d
\]
be the strictly positive eigenvalues of \(A\) and associated orthonormal
eigenvectors, and set
\[
E_A:=\mathrm{Span}\{q_1,\dots,q_r\}.
\]
For every \(x\in\R^d\), we define the \emph{affine diffusive leaf through \(x\)} by
\[
\mathcal L_A^x:=x+E_A.
\]
This is the affine subspace on which a diffusion with covariance matrix \(A\),
started from \(x\), evolves.

We denote by \(\ell_A^x\) the canonical Lebesgue measure on \(\mathcal L_A^x\),
defined as the push-forward of the Lebesgue measure on \(\R^r\) under the map
\[
\Xi_x:\R^r\to \mathcal L_A^x,
\qquad
\Xi_x(\xi):=x+\sum_{i=1}^r \xi_i q_i.
\]

Since \(A\) is invertible on \(E_A\), we denote by
\[
A^{-1}_{E_A}:E_A\to E_A
\]
the inverse of the restriction of \(A\) to \(E_A\). Explicitly,
\[
A^{-1}_{E_A}v
=
\sum_{i=1}^r \lambda_i^{-1}\langle v,q_i\rangle q_i,
\qquad v\in E_A.
\]

For \(S\le t<T\) and \(y,z\in\mathcal L_A^x\), we define the intrinsic heat kernel
associated with \(A\) by
\begin{align*}
p_A(t,y;T,z)
&:=
\frac{1}{(2\pi(T-t))^{r/2}\prod_{i=1}^r \lambda_i^{1/2}}
\\[2pt]
&\quad\times
\exp\!\left(
-\frac{1}{2(T-t)}
\big\langle A^{-1}_{E_A}(z-y),z-y\big\rangle
\right).
\end{align*}

Given a nonnegative Borel function \(g:\mathcal L_A^x\to[0,\infty]\), we define its
backward heat propagation by
\[
(P_{t,T}^A g)(y)
:=
\int_{\mathcal L_A^x} p_A(t,y;T,z)\,g(z)\,d\ell_A^x(z),
\qquad
(t,y)\in (S,T)\times\mathcal L_A^x,
\]
whenever the integral is finite.

We say that a nonnegative Borel function \(g:\mathcal L_A^x\to[0,\infty]\) is
\((A,S,T,x)\)-admissible if
\begin{equation}\label{cond:admissibility}
(P_{t,T}^A g)(y)<\infty
\qquad
\forall (t,y)\in (S,T)\times\mathcal L_A^x.
\end{equation}

We also denote by \(\mathcal M_+^{\mathrm{loc}}(\mathcal L_A^x)\) the cone of positive
Radon measures on \(\mathcal L_A^x\), that is, positive locally finite Borel measures
on \(\mathcal L_A^x\). Whenever convenient, such measures are identified with Borel
measures on \(\R^d\) supported on \(\mathcal L_A^x\), by extension by zero outside
\(\mathcal L_A^x\).

\subsection{The leafwise \texorpdfstring{$h$}{h}-transform}
\label{subsec:intervalwise-backward}

Theorem~\ref{thm:general-sbts-main} identifies the optimal law \(P^*\)
abstractly as an \(h\)-transform, with no regularity required of \(R\). When
the reference moreover satisfies the uniqueness condition \((U)\) of
L\'eonard \cite{Leonard}, Theorems~2.1 and~2.3 of \cite{Leonard} upgrade this
to a controlled semimartingale representation: \(P^*\) is a controlled
diffusion with respect to the predictable bracket of \(R\), driven by an
adapted drift correction \(\beta^*\). The triangular frozen references of
Subsection~\ref{subsec:conditional-to-leaf} satisfy condition \((U)\)
conditionally on the triangular environment: given the latent realisations,
the reference density factorises into a finite product of explicit conditional
Gaussians, and the associated stochastic integrals are finite. The next lemma
identifies this drift correction in closed form on a single frozen interval.

\begin{lemma}[Controlled diffusion induced by a terminal martingale density on one frozen interval]
\label{lem:frozen-interval-controlled-diffusion}
Let \(0\le S<T\), let \(x\in\R^d\), and let
\[
(\Omega,\G,(\G_t)_{t\in[S,T]},\Q)
\]
support a \(d\)-dimensional Brownian motion \(W\).
Let \(\sigma\in\R^{d\times d}\), set
\[
A:=\sigma\sigma^\top\in\mathbb S_+^d,
\qquad
X_t:=x+\sigma(W_t-W_S),
\qquad t\in[S,T].
\]
Let \(r:=\rank(A) \ge1\), and let \(q_1,\dots,q_r\in\R^d\) be orthonormal eigenvectors of \(A\)
associated with its strictly positive eigenvalues \(\lambda_1,\dots,\lambda_r>0\).

Let
\[
f:\mathcal L_A^x\to[0,\infty]
\]
be a Borel measurable function, identified throughout with its extension
by zero outside \(\mathcal L_A^x\). Assume that
\[
\E_\Q[f(X_T)]=1,
\]
and that \(f\) is \((A,S,T,x)\)-admissible.

Define the martingale probability \(\Q^f\) on \((\Omega,\G_T)\) by
\[
\frac{d\Q^f}{d\Q}:=f(X_T),
\]
and define
\[
H_t:=\E_\Q[f(X_T)\mid \G_t],
\qquad t\in[S,T].
\]

Then there exists a unique Borel function
\[
h:(S,T)\times\mathcal L_A^x\to(0,\infty]
\]
with \(H_t=h(t,X_t)\) \(\Q\)-a.s.\ for every \(t\in(S,T)\), namely the backward
heat propagation
\[
h(t,y)
=
(P_{t,T}^A f)(y)
=
\int_{\mathcal L_A^x} p_A(t,y;T,z)\,f(z)\,d\ell_A^x(z),
\qquad
(t,y)\in(S,T)\times\mathcal L_A^x.
\]
Moreover, the following hold.

\begin{enumerate}
\item[(i)]
\(h\) is \(C^1\) in time and \(C^\infty\) along the diffusive directions
\(q_1,\dots,q_r\) on \((S,T)\times\mathcal L_A^x\), strictly positive there, and
solves the backward heat equation on the active leaf,
\[
\partial_t h(t,y)
+
\frac12\sum_{j=1}^r \lambda_j\,D^2_{q_jq_j}h(t,y)=0,
\qquad
(t,y)\in(S,T)\times\mathcal L_A^x;
\]
it attains the terminal datum in the sense that
\(h(t,\cdot)\,\ell_A^x\to f\,\ell_A^x\) in
\(\mathcal M_+^{\mathrm{loc}}(\mathcal L_A^x)\) as \(t\uparrow T\).

\item[(ii)]
The density process \(M_t:=h(t,X_t)=H_t\) is a positive \(\Q\)-martingale and
satisfies the intrinsic It\^o formula
\[
dM_t
=
\sum_{j=1}^r \sqrt{\lambda_j}\,D_{q_j}h(t,X_t)\,dB_t^j
=
\nabla h(t,X_t)\,\sigma\,dW_t,
\]
where \(B=(B^1,\dots,B^r)\) is the \(r\)-dimensional Brownian motion driving the
active directions (constructed in the proof).

\item[(iii)]
Under \(\Q^f\) there is a \(d\)-dimensional Brownian motion \(W^f\) such that,
on \((S,T)\),
\begin{align*}
dX_t
&=
A\,\nabla\log h(t,X_t)\,dt+\sigma\,dW_t^f
\\
&=
\sum_{j=1}^r \lambda_j\,D_{q_j}\log h(t,X_t)\,q_j\,dt+\sigma\,dW_t^f .
\end{align*}
The drift is intrinsic to the active covariance directions and is well defined
even when \(A\) is rank-deficient.
\end{enumerate}
\end{lemma}
\begin{proof}
Let \(Q_r:=[q_1,\dots,q_r]\), let \(Q_0\) span \(\ker A\), and set
\(\eta_0:=Q_0^\top x\), so that
\(\mathcal L_A^x=\{Q_r\xi+Q_0\eta_0:\xi\in\R^r\}\). Since \(Q_0^\top A Q_0=0\)
gives \(Q_0^\top\sigma=0\), we have \(Q_0^\top X_t\equiv\eta_0\), hence
\(X_t\in\mathcal L_A^x\) for all \(t\), \(\Q\)-a.s., and \(f(X_T)\) is well
defined. As \(X_T=X_t+\sigma(W_T-W_t)\) with \(W_T-W_t\) independent of
\(\G_t\), for \(t\in(S,T)\)
\[
H_t=\E_\Q[f(X_T)\mid\G_t]
=\E_\Q\!\bigl[f(y+\sigma(W_T-W_t))\bigr]\big|_{y=X_t}
=h(t,X_t),
\]
with $h(t,y)=(P_{t,T}^A f)(y),$ which is finite by \((A,S,T,x)\)-admissibility.

\emph{Straightening.} On the active coordinates \(\xi\in\R^r\) set
\(U(t,\xi):=h(t,Q_r\xi+Q_0\eta_0)\). Only the \(q\)-directions are diffusive, so
\(U\) is the classical nondegenerate heat semigroup in \(\R^r\), with generator
\(\frac12\sum_j\lambda_j\partial^2_{\xi_j}\), applied to the terminal datum
\(\xi\mapsto f(Q_r\xi+Q_0\eta_0)\). Hence \(U\in C^{1,\infty}((S,T)\times\R^r)\),
\(U>0\), \(U\) solves \(\partial_tU+\frac12\sum_j\lambda_j\partial^2_{\xi_j}U=0\),
and \(U(t,\cdot)\,d\xi\to f(Q_r\cdot+Q_0\eta_0)\,d\xi\) in
\(\mathcal M^{\mathrm{loc}}_+(\R^r)\) as \(t\uparrow T\), with \(h\) the unique such
solution. Transporting these statements back through the affine isometry
\(\xi\mapsto Q_r\xi+Q_0\eta_0\), and pushing forward Lebesgue measure to
\(\ell_A^x\), yields (i); strict positivity on all of
\((S,T)\times\mathcal L_A^x\) follows from \(p_A>0\) together with
\(\int_{\mathcal L_A^x}f\,p_A(S,x;T,\cdot)\,d\ell_A^x=\E_\Q[f(X_T)]=1\), so that
\(f\) is not \(\ell_A^x\)-negligible.

\emph{It\^o and Girsanov.} Put \(Z_t:=Q_r^\top X_t\), so \(X_t=Q_rZ_t+Q_0\eta_0\),
and let \(R:=\Lambda_r^{-1/2}Q_r^\top\sigma\), \(\Lambda_r:=\diag(\lambda_1,\dots,\lambda_r)\).
Then \(RR^\top=\Lambda_r^{-1/2}Q_r^\top A Q_r\Lambda_r^{-1/2}=I_r\), so
\(B_t:=R(W_t-W_S)\) is an \(r\)-dimensional \(\Q\)-Brownian motion and
\(dZ_t=Q_r^\top\sigma\,dW_t=\Lambda_r^{1/2}\,dB_t\). Since
\(M_t=H_t=U(t,Z_t)\) with \(U\in C^{1,2}\), the classical It\^o formula and the
heat equation give
\[
dM_t=\sum_{j=1}^r\sqrt{\lambda_j}\,\partial_{\xi_j}U(t,Z_t)\,dB_t^j
=\nabla h(t,X_t)\,\sigma\,dW_t,
\]
using \(\partial_{\xi_j}U=D_{q_j}h\) and \(q_j^\top\sigma=\sqrt{\lambda_j}\,R_j\)
(the \(j\)-th row of \(R\)); thus \(M\) is a positive \(\Q\)-martingale, proving
(ii). Applying Girsanov to \(d\Q^f/d\Q=M_T\) yields, in the active coordinates,
\[
dZ_t=\sum_{j=1}^r\lambda_j\,\partial_{\xi_j}\log U(t,Z_t)\,e_j\,dt
+\Lambda_r^{1/2}\,dB_t^f
\]
for a \(\Q^f\)-Brownian motion \(B^f\). Complete the kernel directions by an
independent \((d-r)\)-dimensional \(\Q^f\)-Brownian motion \(B^{f,\perp}\) and let
\(\bar Q_0\) be an orthonormal basis of \(\ker\sigma\); since
\(R^\top R+\bar Q_0\bar Q_0^\top=I_d\), the process
\(W_t^f:=R^\top B_t^f+\bar Q_0 B_t^{f,\perp}\) is a \(d\)-dimensional
\(\Q^f\)-Brownian motion with \(\sigma\,dW_t^f=Q_r\Lambda_r^{1/2}\,dB_t^f\).
Multiplying the displayed equation by \(Q_r\) gives
\[
dX_t=\sum_{j=1}^r\lambda_j\,D_{q_j}\log h(t,X_t)\,q_j\,dt+\sigma\,dW_t^f
=A\,\nabla\log h(t,X_t)\,dt+\sigma\,dW_t^f,
\]
the intrinsic controlled diffusion of (iii).
\end{proof}

\subsection{Identification of the conditional optimiser}
\label{subsec:intervalwise-backward-local}

We now apply Lemma~\ref{lem:frozen-interval-controlled-diffusion} to the
\(i\)-th leaf problem and identify the resulting controlled diffusion with the
conditional optimiser required by
Section~\ref{sec:reference-based-sbts-projection}.

\begin{proposition}[Local intervalwise representation of the minimiser]
\label{prop:local-intervalwise-minimiser}
Adopt the setup of Subsection~\ref{subsec:conditional-to-leaf}, with \(\mu\)
satisfying \eqref{eq:finite-entropy-assumption} conditionally on the latent
realisation. Fix \(i\in\{0,\dots,N-1\}\), a coarse past
\(\mathbf X_i=\mathbf x_i\), and a latent realisation \(y_{i+1}\in\mathcal Y\)
determining the frozen reference on \((t_i,t_{i+1}]\). Every \(i\)-indexed
object below depends on \((\mathbf x_i,y_{i+1})\), suppressed from the
notation: we write
\[
\sigma_i,
\qquad
A_i=\sigma_i\sigma_i^\top,
\qquad
r_i=\rank A_i,
\qquad
(\lambda_{i,j},q_{i,j})_{j=1}^{r_i}
\]
for the frozen coefficients and the positive eigenpairs of \(A_i\),
\[
\mathcal L_i:=\operatorname{Ran}A_i,
\qquad
\ell_i:=\text{canonical Lebesgue measure on }\mathcal L_i,
\]
for the active subspace of admissible increments (the active affine leaf in
position space being \(x_i+\mathcal L_i\)), and
\[
P^{*,i}:=P^*(\,\cdot\mid \mathbf X_i=\mathbf x_i,\,Y=y_{i+1})
\]
for the conditional law of the minimiser. With
\(\Delta X_{i+1}:=X_{t_{i+1}}-X_{t_i}\), let
\begin{align*}
\eta_i^*
&:= P^*\!\left(\Delta X_{i+1}\in\cdot\mid \mathbf X_i=\mathbf x_i,Y=y_{i+1}\right),
\\
k_i^R
&:= R\!\left(\Delta X_{i+1}\in \cdot\mid \mathbf X_i=\mathbf x_i,Y=y_{i+1}\right)
\end{align*}
be the conditional next-increment laws under the minimiser and under the
reference. Choosing regular conditional probabilities consistently under
\(P^\ast=G(\Pi)\,R\), we have \(\eta_i^*\ll k_i^R\) for
\(P^\ast\circ (\mathbf X_i,Y)^{-1}\)-a.e.\ \((\mathbf x_i,y_{i+1})\), and we
set
\[
f_i(\delta)
:= \frac{d\eta_i^*}{dk_i^R}(\delta),
\qquad \delta\in\R^d,
\]
viewed as a Borel function on \(\mathcal L_i\) and extended by zero outside
\(\mathcal L_i\). 

Then, for \(P^*\circ (\mathbf X_i,Y)^{-1}\)-a.e.\
\((\mathbf x_i,y_{i+1})\), the following hold.

\begin{enumerate}
\item[(a)]
The function \(f_i\) is
\((A_i,t_i,t_{i+1},x_i)\)-admissible in the sense of
\eqref{cond:admissibility} (formulated on the active increment
subspace \(\mathcal L_i\)).

\item[(b)]
Under \(P^{*,i}\), there exists a \(d\)-dimensional Brownian motion
\(W^{*,i}\) such that, on \([t_i,t_{i+1})\),
\[
dX_t=b_i^*(t,X_t)\,dt+\sigma_i\,dW_t^{*,i},
\]
where
\[
b_i^*(t,x)
=
\sum_{j=1}^{r_i}
\lambda_{i,j}\,
D_{q_{i,j}}\log H_i(t,x)\,
q_{i,j},
\qquad
dt\otimes P^{*,i}\text{-a.e.}
\]
and the potential \(H_i\) being the unique function given by
Lemma~\ref{lem:frozen-interval-controlled-diffusion}, namely the backward heat
potential associated with the frozen interval \([t_i,t_{i+1}]\), the frozen
diffusion matrix \(\sigma_i\), and the terminal tilt \(f_i\) on
\(\mathcal L_i\). In particular,
\[
\partial_t H_i(t,x)
+
\frac12\sum_{j=1}^{r_i}\lambda_{i,j}\,
D^2_{q_{i,j}q_{i,j}}H_i(t,x)
=0,
\quad
(t,x)\in(t_i,t_{i+1})\times (x_i+\mathcal L_i),
\]
with the terminal datum attained as
\(H_i(t,\cdot)\,\ell_i\to f_i\,\ell_i\) as \(t\uparrow t_{i+1}\), and
\[
H_i(t,x)>0
\qquad\forall (t,x)\in(t_i,t_{i+1})\times (x_i+\mathcal L_i).
\]
\end{enumerate}
\end{proposition}

\begin{proof}
Write \(R^i:=R(\,\cdot\mid\mathbf X_i=\mathbf x_i,Y=y_{i+1})\) for the
conditional reference and \(p_i(t,x;t_{i+1},\delta)\) for the reference density
of the increment \(\Delta X_{i+1}=\delta\) given \(X_t=x\); under the frozen
reference this is the intrinsic Gaussian heat kernel on \(\mathcal L_i\)
attached to \(A_i\), and
\(k_i^R(d\delta)=p_i(t_i,x_i;t_{i+1},\delta)\,\ell_i(d\delta)\). Fix
\((\mathbf x_i,y_{i+1})\) in a full-measure set; under \(R^i\) the process solves
\(dX_t=\sigma_i\,dW_t^R\), \(X_{t_i}=x_i\), so \(\Delta X_{i+1}\in\mathcal L_i\)
and \(X\) evolves on \(x_i+\mathcal L_i\).

\emph{Terminal tilt via the conditional entropy projection.} By the
one-step triangular characterisation of
Proposition~\ref{prop:triangular-conditional-characterization}, and by
Corollary~\ref{cor:terminal-conditional-law}, the conditional optimiser
on the current leaf is the \(h\)-transform of \(R_i\) by the terminal
density
\[
f_i=\frac{d\eta_i^\mu}{dk_i^R}.
\]
Equivalently, \(f_i=d\eta_i^*/dk_i^R\) is the terminal density tilting
\(R^i\) into \(P^{*,i}\), i.e.\
\(E_{P^{*,i}}[F]=E_{R^i}[F\,f_i(\Delta X_{i+1})]\) for every bounded
\(\F_{t_{i+1}}\)-measurable \(F\); Corollary~\ref{cor:terminal-conditional-law}
identifies \(\eta_i^*\) with the target transition \(\eta_i^\mu\). The
local terminal tilt is thus justified not by Markovianity alone, but by
the conditional entropy projection structure together with the
identification of the projected grid law with \(\mu\).

\emph{Admissibility: the \(\alpha\)--\(\beta\) gap.} For \(t\in(t_i,t_{i+1})\)
and \(x\in x_i+\mathcal L_i\), since \(f_i\,k_i^R=\eta_i^*\),
\[
\int_{\mathcal L_i} p_i(t,x;t_{i+1},\delta)\,f_i(\delta)\,d\ell_i(\delta)
=
\int_{\mathcal L_i}
\frac{p_i(t,x;t_{i+1},\delta)}{p_i(t_i,x_i;t_{i+1},\delta)}\,\eta_i^*(d\delta),
\]
and, in intrinsic coordinates on \(\mathcal L_i\) with Mahalanobis norm
\(\|v\|_i^2:=\langle A_i^+v,v\rangle\),
\[
\frac{p_i(t,x;t_{i+1},\delta)}{p_i(t_i,x_i;t_{i+1},\delta)}
=
\Big(\tfrac{t_{i+1}-t_i}{t_{i+1}-t}\Big)^{r_i/2}
\exp\!\Big(-\tfrac{\|x_i+\delta-x\|_i^2}{2(t_{i+1}-t)}
+\tfrac{\|\delta\|_i^2}{2(t_{i+1}-t_i)}\Big).
\]
The quadratic part in \(\delta\) is
\(-\tfrac12\big(\tfrac{1}{t_{i+1}-t}-\tfrac{1}{t_{i+1}-t_i}\big)\|\delta\|_i^2\),
strictly negative because the variance from \(t\) to \(t_{i+1}\) is smaller than
that from \(t_i\) to \(t_{i+1}\). Completing the square bounds the ratio by a
Gaussian-type function of \(\delta\); since \(\eta_i^*\) is a probability
measure, the integral is finite and \(f_i\) is
\((A_i,t_i,t_{i+1},x_i)\)-admissible, proving (a). The same integral at the
anchor \((t_i,x_i)\) equals \(E_{R^i}[f_i(\Delta X_{i+1})]=1\), fixing the
boundary normalisation \(H_i(t_i,x_i)=1\).

\emph{Conclusion.} We are in the setting of
Lemma~\ref{lem:frozen-interval-controlled-diffusion} with
\(S=t_i,\ x=x_i,\ \sigma=\sigma_i,\ f=f_i\), which provides the unique backward
heat potential \(H_i\) on \((t_i,t_{i+1})\times\mathcal L_i\) with terminal datum
\(f_i\,\ell_i\) and the controlled diffusion
\[
dX_t=\sum_{j=1}^{r_i}\lambda_{i,j}\,D_{q_{i,j}}\log H_i(t,X_t)\,q_{i,j}\,dt
+\sigma_i\,dW_t^{*,i},
\qquad t\in(t_i,t_{i+1}),
\]
under \(P^{*,i}\), proving (b).
\end{proof}

\begin{remark}[Boundary semantics at \(t=t_i\)]
\label{rem:boundary-semantics-t-i}
Under \(P^{*,i}\) the controlled diffusion is initialised at the
single point \(X_{t_i}=x_i\), so the natural domain of \(H_i\) and
\(b_i^*\) is the \emph{open} cylinder
\((t_i,t_{i+1})\times\mathcal L_i\), and only the diagonal evaluation at
\((t_i,x_i)\) is meaningful. We therefore set
\[
H_i(t_i,x_i):=\lim_{t\downarrow t_i}H_i(t,x_i)=1,
\qquad
b_i^*(t_i,x_i):=\lim_{t\downarrow t_i}b_i^*(t,x_i),
\]
respectively the bridge normalisation and the starting drift of the
controlled diffusion. The same convention applies to the
regularised and empirical potentials of
Section~\ref{sec:analytic-statistical-terminal-datum}.
\end{remark}

Since \(P^*\) projects onto \(\mu\) on the augmented state space, the terminal
datum \(f_i\) is determined entirely by \(\mu\), as we now record.

\begin{corollary}[Identification of the exact terminal datum under \(\mu\)]
\label{cor:terminal-conditional-law}
Adopt the setup and shorthand of
Proposition~\ref{prop:local-intervalwise-minimiser}, and set
\[
\eta_i^\mu
:= \mu\!\left(\Delta X_{i+1}\in\cdot\mid \mathbf X_i=\mathbf x_i,Y=y_{i+1}\right).
\]
Let \(\mu_i\) denote the joint law of \((\mathbf X_i,Y)\) on the
augmented state space. Then, for \(\mu_i\)-a.e.\
\((\mathbf x_i,y_{i+1})\),
\[
\eta_i^* = \eta_i^\mu,
\qquad
f_i(\delta) = \frac{d\eta_i^\mu}{dk_i^R}(\delta).
\]
\end{corollary}

\begin{proof}
On the augmented space \(P^*\circ(X_{t_0},\dots,X_{t_N},Y)^{-1}=\mu\), so the
conditional laws of \(\Delta X_{i+1}\) given \((\mathbf X_i,Y)\) agree under
\(P^*\) and \(\mu\). Hence \(\eta_i^*=\eta_i^\mu\) for \(\mu_i\)-a.e.\
\((\mathbf x_i,y_{i+1})\), and substituting into the definition of \(f_i\)
yields the claim.
\end{proof}

\paragraph{Reconstruction of the conditional bridge.}
Proposition~\ref{prop:local-intervalwise-minimiser} and
Corollary~\ref{cor:terminal-conditional-law} identify, for a.e.\ conditioning
\((\mathbf x_i,y_{i+1})\), the conditional law \(P^{*,i}\) as the
\(h\)-transform of the frozen kernel \(k_i^R\) by the terminal tilt
\(f_i=d\eta_i^\mu/dk_i^R\). This is exactly the conditional optimiser
\(Q_{X,i}^{\star,\mathbf z_i,y_{i+1}}\) required by
Proposition~\ref{prop:triangular-conditional-characterization}\,(ii), and the
measurable kernel \((\mathbf x_i,y_{i+1})\mapsto P^{*,i}\) is the family of
one-step conditional optimisers required by
Proposition~\ref{prop:temporal-conditional-characterization}; concatenated
along the realised triangular grid with the lower-level optimisers, these
kernels reconstruct the global minimiser \(P^*\).

\paragraph{What remains to be estimated.}
Once the conditioning \((\mathbf x_i,y_{i+1})\) is fixed, the frozen
covariance \(A_i\), the reference kernel \(k_i^R\), and the leafwise heat
semigroup \(P^{A_i}_{t,t_{i+1}}\) are deterministic, in closed form from the
construction of the triangular reference. The terminal tilt
\(f_i=d\eta_i^\mu/dk_i^R\) and the backward potential
\(H_i=P^{A_i}_{t,t_{i+1}}f_i\) are therefore explicit functions of the single
unknown \(\eta_i^\mu\): the only statistical object to be estimated from data
is the conditional next-increment law \(\eta_i^\mu\). This is the estimation
problem taken up in Section~\ref{sec:analytic-statistical-terminal-datum}.

\subsection{Frozen approximation of an ideal volatility-informed continuous reference}
\label{subsec:volatility-informed-freezing}

High-frequency observations need not be used by refining the SBTS
observation grid. Refining the output grid multiplies the number of
generation steps, of sliding windows, of log-weight updates, and the size of
the empirical databases, without necessarily adding information about the
drift on the coarse scale. A more natural use of intra-interval observations
is the estimation of the evolution of the volatility, or of the covariance,
inside each coarse interval. In TR-SBTS this information enriches the
reference: high-frequency data may enrich the covariance evolution within each coarse interval, while the generative grid remains
coarse. The refinement is therefore placed in the reference, rather than
necessarily in the generated observation grid.

In this perspective, the frozen reference of
Subsection~\ref{subsec:conditional-to-leaf} is a one-piece approximation of
an \emph{ideal} volatility-informed reference whose covariance varies
continuously inside each coarse interval. This subsection quantifies the
approximation and shows that the resulting bridge is stable under it. Since
the endpoint Gaussian scale is controlled by the cumulative covariance, the
relevant error is the error in the covariance cumulant, not the pointwise
volatility error.

Fix a coarse interval \([S,T]\), \(\Delta:=T-S\), a coarse past \(h\) at time
\(S\), and a latent volatility environment \(Y\). Conditionally on \((h,Y)\),
the ideal reference is the additive Gaussian martingale
\(dZ_t=\sigma(t,h,Y)\,dW_t\) with covariance
\(a_t(h,Y)=\sigma\sigma^\top\in\mathbb S_+^d\), possibly random and degenerate
but \emph{state-independent} in \(Z_t\). Define the covariance cumulant
\[
\Gamma_t(h,Y):=\int_S^t a_r(h,Y)\,dr,
\qquad
C_{S,T}(h,Y):=\Gamma_T(h,Y),
\]
which determines the terminal Gaussian scale and the active affine leaf
\(x+\operatorname{Ran}C_{S,T}\). The implementable construction replaces
\(a_t\) by the one-piece frozen covariance \(\bar a_{S,T}=C_{S,T}/\Delta\),
whose cumulant is the linear interpolation
\(\bar\Gamma_t=\frac{t-S}{\Delta}C_{S,T}\) between \(0\) and \(C_{S,T}\); more
generally, a refinement \(\mathcal P=\{S=s_0<\cdots<s_M=T\}\) yields the
\emph{piecewise-frozen} cumulant \(\Gamma^{\mathcal P}\) by linear
interpolation of \(\Gamma\) between consecutive gridpoints. The algorithm of
this work uses \(M=1\).

The two references share the same terminal covariance \(C_{S,T}\),
hence the same endpoint Gaussian scale and the same active leaf;
they differ only inside the interval. For a fixed endpoint pair
\((x,z)\) with \(z-x\in\operatorname{Ran}C_{S,T}\), the means of the
ideal and piecewise-frozen bridges are
\(m_t(\cdot;x,z)=x+\Gamma_t C_{S,T}^+(z-x)\) and
\(m_t^{\mathcal P}(\cdot;x,z)=x+\Gamma_t^{\mathcal P}C_{S,T}^+(z-x)\),
where \(C_{S,T}^+\) is the Moore--Penrose inverse on the active
leaf. Setting the intrinsic endpoint norm and the normalised
cumulant interpolation error
\[
r_{S,T}(h,Y;x,z) := \bigl|C_{S,T}^{+/2}(z-x)\bigr|,
\qquad
\eta_{\mathcal P}(h,Y) := \sup_{S\le t\le T}
\bigl\|(\Gamma_t-\Gamma_t^{\mathcal P})C_{S,T}^{+/2}\bigr\|,
\]
one has the uniform mean-bound
\(\sup_t|m_t-m_t^{\mathcal P}|\le\eta_{\mathcal P}\cdot r_{S,T}\),
and the covariance kernels of the two Gaussian bridges are
controlled by the same \(\eta_{\mathcal P}\) through their explicit
expressions in terms of \(\Gamma\) and \(C_{S,T}^+\).

\begin{proposition}[Stability of piecewise-frozen endpoint bridges]
\label{prop:piecewise-frozen-bridge-stability}
Under conditional Gaussianity of \(R^a\), if
\(\eta_{\mathcal P_n}(h,Y)\to 0\) in probability and the intrinsic
endpoint norms \(r_{S,T}(h,Y;x,z)\) are tight under the endpoint
laws \(\pi^{h,Y}\), then the piecewise-frozen entropy projection
\(P^{h,Y,\mathcal P_n,\star}\) converges in probability, in the weak
topology on \(C([S,T];\R^d)\), to its ideal counterpart
\(P^{h,Y,\star}\).
\end{proposition}

\begin{proof}
For fixed \((h,Y,x,z)\), both bridges are Gaussian; means and
covariance kernels are continuous functions of \(\Gamma\),
\(\Gamma^{\mathcal P_n}\) and \(C_{S,T}^+\), and the difference is
uniformly controlled by \(\eta_{\mathcal P_n}\cdot r_{S,T}\).
Tightness of \(r_{S,T}\) localises the argument, and
\(\eta_{\mathcal P_n}\to 0\) closes the convergence after
integration against \(\pi^{h,Y}\).
\end{proof}

Natural sufficient conditions are continuity of \((t,h,Y)\mapsto a_t(h,Y)\)
on compact strata (so that \(\eta_{\mathcal P_n}\to 0\) along refining
subgrids) and finite entropy of \(\pi^{h,Y}\) relative to the Gaussian
endpoint law induced by \(C_{S,T}\) (so that the intrinsic endpoint norms are
tight). The approximation error thus concerns the cumulative covariance, not
instantaneous volatilities: already at \(M=1\) the frozen reference preserves
the total covariance \((T-S)\bar a_{S,T}=C_{S,T}\), hence the endpoint
Gaussian scale and the terminal leaf; finer subgrids refine only the
path-level interpolation of the cumulant.

\section{Statistical approximation and main convergence theorem}
\label{sec:analytic-statistical-terminal-datum}

For each frozen interval \([t_i,t_{i+1}]\), the exact local drift is
determined by the backward potential \(H_i\) of
Proposition~\ref{prop:local-intervalwise-minimiser}. This section
isolates the implementable approximation problem in the order in
which it appears at run time: target of estimation, analytic
regularisation, empirical Gaussian estimator, statistical
assumptions, and final convergence theorem. Detailed proofs are
given in Appendix~\ref{sec:appendix-statistical-approximation}.

\subsection{The target of estimation}
\label{subsec:target-of-estimation}

We adopt throughout this section the shorthand of
Proposition~\ref{prop:local-intervalwise-minimiser}: \(A_i\),
\(\sigma_i\), \(\mathcal L_i\), \(k_i^R\), \(\eta_i^\mu\), \(f_i\),
\(H_i\), \(b_i^*\), etc., all carry the implicit conditioning
\((\mathbf x_i,y_{i+1})\). By the closing remarks of
Corollary~\ref{cor:terminal-conditional-law}, the frozen covariance
\(A_i\), the reference kernel \(k_i^R\) and the leafwise heat
semigroup are deterministic functions of the conditioning; the only
object that has to be estimated from data is the transition kernel
of \(\mu\) on the coarse grid,
\[
\eta_i^\mu
:= \mu\bigl(\Delta X_{i+1}\in\cdot\mid\mathbf X_i=\mathbf x_i,Y=y_{i+1}\bigr).
\]
Our strategy is three-step:
\begin{enumerate}
\item[(i)]
\emph{Analytic regularisation}: replace the rank-deficient frozen
covariance \(A_i\) with its spectral floor \(A_i^\varepsilon\) and
the leafwise heat semigroup with the corresponding full-rank
Gaussian semigroup, producing a regularised potential
\(H_i^\varepsilon\) and drift \(b_i^\varepsilon\). Under the
conditioning, \(A_i^\varepsilon\) remains a deterministic function
of \((\mathbf x_i,y_{i+1})\); only the leafwise geometry of \(A_i\)
is softened, not the conditioning itself.
\item[(ii)]
\emph{Empirical substitution}: replace the conditional law
\(\eta_i^\mu\) inside \(H_i^\varepsilon\) by its
Nadaraya--Watson kernel-weighted empirical surrogate built from
\(M\) i.i.d.\ samples, producing an empirical estimator
\(\widehat H_i^{\varepsilon,M}\) and the corresponding drift
\(\widehat b_i^{\varepsilon,M}\).
\item[(iii)]
\emph{Diagonal extraction}: drive
\((\varepsilon_n,M_n,h_n)\to(0,\infty,0)\) along a sequence
admissible in the statistical sense, so that
\(\widehat H_i^{\varepsilon_n,M_n}\to H_i\) on compacta of the open
intervalwise cylinder
\((t_i,t_{i+1})\times\mathcal L_i\), and the corresponding drifts converge to their exact counterparts.
\end{enumerate}
Steps (i)--(ii) are spelled out in
Sections~\ref{subsec:spectral-regularisation-leafwise-kernel}
and~\ref{subsec:statistical-estimator}; step (iii) is
Theorem~\ref{thm:approximate-intervalwise-dynamics} of
Section~\ref{subsec:main-statistical-convergence-theorem}.

\subsection{Coherent regularised objects}
\label{subsec:spectral-regularisation-leafwise-kernel}

Write \(r_i:=\rank A_i\), \(\Delta_i:=t_{i+1}-t_i\), and
\(\lambda_i^{\min,+}\) for the smallest strictly positive eigenvalue
of \(A_i\). For \(\varepsilon>0\), the spectral floor of \(A_i\) is
\begin{equation}
\label{eq:def-A-eps}
A_i^\varepsilon
:=
Q\Lambda_\varepsilon Q^{\top},
\qquad
\Lambda_\varepsilon
:=
\operatorname{diag}\!\bigl(
\max(\lambda_1,\varepsilon),\dots,\max(\lambda_d,\varepsilon)
\bigr),
\end{equation}
where \(A_i=Q\operatorname{diag}(\lambda_1,\dots,\lambda_d)Q^\top\)
is any spectral decomposition; \(A_i^\varepsilon\) is full-rank,
uniformly elliptic, with \(A_i^\varepsilon\succeq\varepsilon I\) and
\(\det A_i^\varepsilon\ge\varepsilon^d\), and is independent of the
choice of \(Q\). The associated full-rank backward Gaussian density
\emph{on the increment} is
\begin{equation}
\label{eq:def-p-eps}
p_i^\varepsilon(t,x;t_{i+1},\delta)
:=
\frac{
\exp\!\Bigl(
-\tfrac{1}{2(t_{i+1}-t)}
\bigl\langle (A_i^\varepsilon)^{-1}(x_i+\delta-x),x_i+\delta-x\bigr\rangle
\Bigr)
}{
(2\pi(t_{i+1}-t))^{d/2}\,(\det A_i^\varepsilon)^{1/2}
},
\end{equation}
the density of \(\Delta X_{i+1}=\delta\) given \(X_t=x\) under the
\(\varepsilon\)-floored reference (so that the position at
\(t_{i+1}\) lies at \(x_i+\delta\)).

The \emph{regularised backward potential} is defined as the coherent
Doob kernel ratio: for \(t_i<t<t_{i+1}\) and \(x\in\R^d\),
\begin{equation}
\label{eq:def-H-eps}
H_i^\varepsilon(t,x)
:=
\int_{\R^d}
\Phi_i^\varepsilon(t,x,\delta)\,
\eta_i^\mu(d\delta),
\qquad
\Phi_i^\varepsilon(t,x,\delta)
:=
\frac{p_i^\varepsilon(t,x;t_{i+1},\delta)}
     {p_i^\varepsilon(t_i,x_i;t_{i+1},\delta)};
\end{equation}
equivalently,
\begin{equation}
\label{eq:def-H-eps-cond}
H_i^\varepsilon(t,x)
=
\mathbb E_\mu\!\left[
\Phi_i^\varepsilon(t,x,\Delta X_{i+1})
\,\middle|\,
\mathbf X_i=\mathbf x_i,\,Y=y_{i+1}
\right].
\end{equation}
Both transition densities in \(\Phi_i^\varepsilon\) share the same
floor \(A_i^\varepsilon\); the \(\varepsilon\)-dependent
normalisation \((2\pi(t_{i+1}-t))^{-d/2}(\det A_i^\varepsilon)^{-1/2}\)
cancels in the ratio, and \(\Phi_i^\varepsilon\) collapses to a
Gaussian-type function of \(\delta\). The corresponding \emph{regularised
drift} is the intrinsic logarithmic-gradient quantity
\begin{equation}
\label{eq:def-b-eps}
b_i^\varepsilon(t,x)
:=
A_i^\varepsilon\,\nabla_x\log H_i^\varepsilon(t,x).
\end{equation}

\paragraph{Analytic facts.}
The analytic properties of \(H_i^\varepsilon\) and
\(b_i^\varepsilon\) are proved in
Proposition~\ref{prop:spectral-regularisation-leafwise-kernel} of
Appendix~\ref{sec:appendix-statistical-approximation}. The key
ingredient is the \(\alpha<\beta\) gap of the kernel ratio
(\(\alpha:=t_{i+1}-t\), \(\beta:=t_{i+1}-t_i\)), which realises
\(\Phi_i^\varepsilon\) as a bounded Gaussian-type function of
\(\delta\) and supplies uniform domination on bounded \(x\)-sets and
on \(\{t\le t_{i+1}-\tau\}\) for every \(\tau>0\). The matched floor
leaves the active eigenvalues unchanged, so the coherent ratio on
the active affine leaf reduces to the unfloored leafwise ratio,
yielding
\begin{equation}
\label{eq:H-eps-equals-H}
H_i^\varepsilon(t,x)=H_i(t,x),
\qquad
(t,x)\in(t_i,t_{i+1})\times(x_i+\mathcal L_i),
\quad
\varepsilon\in(0,\lambda_i^{\min,+}),
\end{equation}
and the global collapse
\begin{equation}
\label{eq:H-eps-leaf-limit}
H_i^\varepsilon \longrightarrow H_i\,\mathbf 1_{x_i+\mathcal L_i}
\quad\text{as }\varepsilon\downarrow 0
\quad\text{pointwise on }\R^d,
\end{equation}
locally uniformly on compacta of
\((t_i,t_{i+1})\times(x_i+\mathcal L_i)\). Off the leaf the
convergence is driven by the same \(\alpha<\beta\) gap: numerator
and denominator of \(\Phi_i^\varepsilon\) share the same
\(A_i^\varepsilon\) and both contribute a transverse penalty in
\(1/\varepsilon\), but the numerator's penalty, evaluated at the
shorter time \(\alpha=t_{i+1}-t\), is strictly stronger than the
denominator's at \(\beta=t_{i+1}-t_i\); the net residual
\(\sim\exp(-\mathrm{dist}(x-x_i,\mathcal L_i)^2/(2\varepsilon\alpha))\)
drives \(H_i^\varepsilon\to 0\) exponentially in \(1/\varepsilon\).
The intrinsic-gradient drift converges accordingly on every compact
\(K\Subset(t_i,t_{i+1})\times(x_i+\mathcal L_i)\):
\(b_i^\varepsilon \to A_i\nabla_x\log H_i = b_i^*\), in the
Mahalanobis geometry of \(A_i\). The natural quantities throughout
are \(A_i^{\varepsilon,1/2}\nabla H_i^\varepsilon\),
\(A_i^\varepsilon\nabla\log H_i^\varepsilon\) and the intrinsic
Dirichlet energy
\(|A_i^{\varepsilon,1/2}\nabla\log H_i^\varepsilon|^2\), not the
naked Euclidean gradient.

\paragraph{Boundary semantics at \(t=t_i\).}
At the left endpoint the boundary semantics of
Remark~\ref{rem:boundary-semantics-t-i} apply unchanged. We set
\begin{equation}
\label{eq:H-eps-boundary-value}
H_i^\varepsilon(t_i,x_i):=\lim_{t\downarrow t_i}H_i^\varepsilon(t,x_i)=1,
\end{equation}
\begin{equation}
\label{eq:b-eps-boundary-value}
b_i^\varepsilon(t_i,x_i):=\lim_{t\downarrow t_i}b_i^\varepsilon(t,x_i).
\end{equation}
The first equality follows from~\eqref{eq:H-eps-equals-H} together
with the Doob normalisation \(H_i(t,x_i)\to 1\); the second limit
exists by dominated convergence on the Gaussian-ratio integrand,
using the finite intrinsic \(A_i^+\)-Mahalanobis first moment on the
active leaf, implied by finite conditional entropy via
Donsker--Varadhan (with global finite entropy via the
finite-entropy hypothesis~\eqref{eq:finite-entropy-assumption} and
disintegration); for fixed \(\varepsilon\in(0,\lambda_i^{\min,+})\)
the leaf inequality
\(|\delta|_{(A_i^\varepsilon)^{-1}}\le\|(A_i^\varepsilon)^{-1/2}(A_i^+)^{1/2}\|_{\mathrm{op}}|\delta|_{A_i^+}\)
transfers this to the regularised
\(A_i^\varepsilon\)-Mahalanobis quantity that actually appears in
the integrand.

\subsection{Statistical estimator}
\label{subsec:statistical-estimator}

We use throughout the joint samples
\(\bigl(\mathbf X_i^{(m)},Y^{(m)},\Delta X_{i+1}^{(m)}\bigr)_{m=1}^M\)
i.i.d.\ from \(\mu\); the latent realisations
\(Y^{(m)}\in\mathcal Y\) are observed since they are produced by the
joint TR-SBTS construction
(Section~\ref{subsec:joint-tr-sbts-generation}).

\paragraph{Macro-conditioning variable.}
For every level \(\ell\) and interval \(i\), package all the
finite-dimensional summaries on which the regression conditions into
a single macro-variable
\begin{equation}
\label{eq:macro-conditioning}
\mathcal U_i^\ell
:=\Bigl(
\underbrace{x_i}_{\text{state}},\;
\underbrace{(\Delta x_i,\dots,\Delta x_1)}_{\text{past incr.}},\;
\underbrace{y_{i+1}}_{\text{latent}},\;
\underbrace{(C_k^{\varepsilon,\ell})_{k\le i+1}}_{\text{frozen ref.s}}
\Bigr),
\end{equation}
with runtime query value \(u_0\in\mathcal U_i^\ell\) and historical
sample value \(v=U_i^{\ell,(m)}\in\mathcal U_i^\ell\). All
notations \((\mathbf x_i,y_{i+1})\), \((\boldsymbol\xi_i,y)\), etc.\
collapse onto \(u_0\) and \(v\) from now on.

\paragraph{Reference-aware pseudo-distance.}
Define a single query-dependent pseudo-distance on
\(\mathcal U_i^\ell\) by
\begin{equation}
\label{eq:macro-pseudo-distance}
d_\varepsilon^\ell(u_0,v)^2
:=
\sum_{a\in\mathcal A_i^\ell} d_{a,\varepsilon}(u_0,v)^2,
\end{equation}
summing over the components of \(\mathcal U_i^\ell\): for a state or
increment component \(z_a\),
\(d_{a,\varepsilon}(u_0,v)^2=|z_a(u_0)-z_a(v)|_{C_{a,\varepsilon}(u_0)^{-1}}^2\),
the Mahalanobis distance in the runtime-floored covariance; for each
frozen-reference component the log-spectral distance
\begin{equation}
\label{eq:macro-reference-distance}
d_{a,\varepsilon}^{\mathrm{ref}}(u_0,v)
:=
\bigl|\log\!\bigl(C_{a,\varepsilon}(u_0)^{-1/2}\,
C_{a,\varepsilon}(v)\,C_{a,\varepsilon}(u_0)^{-1/2}\bigr)\bigr|_{\mathrm{op}}.
\end{equation}
Locality is measured in the geometry in which the process actually
moves, i.e.\ in the metrics induced by the floored references; in
particular \(d_\varepsilon^\ell\) automatically contains the
reference-comparison growth bounded by
\(\kappa_\varepsilon^\ell(u_0,v)\le\exp(d_\varepsilon^\ell(u_0,v)/2)\)
(Lemma~\ref{lem:reference-aware-comparison}).

\paragraph{Single-kernel regression.}
The Nadaraya--Watson kernel is a univariate function of the macro
pseudo-distance:
\begin{equation}
\label{eq:reference-aware-kernel}
W_h(u_0,v):=K_h\!\bigl(d_\varepsilon^\ell(u_0,v)\bigr),
\quad
\begin{cases}
K(d_\varepsilon^\ell/h),\ \operatorname{supp}K\subset[0,1],&\text{(compact)}\\[2pt]
\exp\!\bigl(-d_\varepsilon^\ell(u_0,v)^2/(2h^2)\bigr),&\text{(Gaussian)}
\end{cases}
\end{equation}
with NW weights and empirical regression
\begin{align}
\label{eq:reference-aware-NW-weights}
\omega_{i,m}^{M}(u_0)
&:=
\frac{W_{h_M}(u_0,U_i^{\ell,(m)})}
     {\sum_{n=1}^M W_{h_M}(u_0,U_i^{\ell,(n)})},
\\
\widehat{\mathbb E}_M[\psi(\Delta)\mid u_0]
&:=
\sum_{m=1}^M\omega_{i,m}^M(u_0)\,\psi(\Delta_i^{(m)}).
\notag
\end{align}
The \emph{empirical regularised potential} replaces \(\eta_i^\mu\)
in~\eqref{eq:def-H-eps} by the NW surrogate
\(\widehat\eta_i^{M}:=\sum_m\omega_{i,m}^M(u_0)\,\delta_{\Delta X_{i+1}^{(m)}}\):
\begin{align}
\label{eq:empirical-H-eps}
\widehat H_i^{\varepsilon,M}(t,x;u_0)
&:=
\sum_{m=1}^M\omega_{i,m}^M(u_0)\,
\Phi_i^\varepsilon\!\bigl(t,x,\Delta X_{i+1}^{(m)}\bigr),
\\
\widehat b_i^{\varepsilon,M}(t,x;u_0)
&:=A_i^\varepsilon\nabla_x\log\widehat H_i^{\varepsilon,M}(t,x;u_0),
\notag
\end{align}
with \(\Phi_i^\varepsilon\) the kernel-ratio response
of~\eqref{eq:def-H-eps}; at \(t=t_i\) the boundary semantics of
Remark~\ref{rem:boundary-semantics-t-i} apply unchanged.

\begin{remark}[Geometry of the kernel: \(d_\varepsilon^\ell\) is fundamental]
\label{rem:reference-aware-kernel-rationale}
The single-distance form~\eqref{eq:reference-aware-kernel} is the
fundamental object: the theory is developed in
\(d_\varepsilon^\ell\), and all separate Mahalanobis, log-spectral
and Euclidean kernels collapse to it. The classical Euclidean
estimator with bandwidth in \(q_i+q_Y\) is recovered as a
\emph{sufficient special case}: if there is a local chart
\(T_{u_0}\) of \(\mathcal U_i^\ell\) that is bi-Lipschitz with
respect to \(d_\varepsilon^\ell\),
\(c|T_{u_0}(v)-T_{u_0}(u_0)|\le d_\varepsilon^\ell(u_0,v)\le C|T_{u_0}(v)-T_{u_0}(u_0)|\)
near \(u_0\), then the effective dimension is
\(q_{\mathrm{eff}}=\dim T_{u_0}(\mathcal U_i^\ell)\) and a positive
Euclidean density at \(u_0\) implies the small-ball
condition~\eqref{eq:s7-small-ball-lower-mass}. In general the
chart is not used; the regression localises directly in
\(d_\varepsilon^\ell\).
\end{remark}

\subsection{Statistical assumptions}
\label{subsec:statistical-assumptions-app}

All assumptions are stated in the macro-conditioning variable
\(\mathcal U_i^\ell\) of~\eqref{eq:macro-conditioning} and the
reference-aware pseudo-distance \(d_\varepsilon^\ell\)
of~\eqref{eq:macro-pseudo-distance}; we write
\(B_\varepsilon(u_0,r):=\{v\in\mathcal U_i^\ell:d_\varepsilon^\ell(u_0,v)<r\}\)
for the associated pseudo-metric balls. The query point is
\(u_0\in\mathcal U_i^\ell\). For the interior convergence we
assume:
\begin{enumerate}[(S1)]
\item \textnormal{(\textsc{Local lower mass})} there exist
\(r_0>0\), \(c_{u_0}>0\) and an effective dimension
\(q_{\mathrm{eff}}>0\) such that
\begin{equation}
\label{eq:s7-small-ball-lower-mass}
\mu_U\bigl(B_\varepsilon(u_0,r)\bigr)
\ge c_{u_0}\,r^{q_{\mathrm{eff}}}
\qquad\forall\,0<r<r_0;
\end{equation}
\item \textnormal{(\textsc{Conditional continuity})} the map
\(v\mapsto\eta_i^{\mu,\ell}(\cdot\mid v):=\mu(\Delta X_{i+1}\in\cdot\mid U_i^\ell=v)\)
is weakly continuous on a \(d_\varepsilon^\ell\)-neighbourhood of
\(u_0\), and \(v\mapsto A_i^\ell(v)\) is operator-norm continuous
at \(u_0\);
\item \textnormal{(\textsc{i.i.d.\ sample})} the historical sample
\((U_i^{\ell,(m)},\Delta X_{i+1}^{(m)})_{m=1}^M\) is i.i.d.\ with
law induced by \(\mu\);
\item \textnormal{(\textsc{Bandwidth})} the bandwidth satisfies
\(h_M\downarrow 0\) and \(Mh_M^{q_{\mathrm{eff}}}\to\infty\);
\item \textnormal{(\textsc{Strong bandwidth})} for almost-sure
convergence, the stronger condition
\(Mh_M^{q_{\mathrm{eff}}}/\log M\to\infty\) holds.
\end{enumerate}
For the boundary-diagonal statement of
Theorem~\ref{thm:approximate-intervalwise-dynamics}(ii) we
additionally assume the \emph{local conditional entropy
admissibility in the same geometry}:
\begin{enumerate}[(S1)]\setcounter{enumi}{5}
\item \textnormal{(\textsc{Local conditional entropy admissibility})}
there exists \(\rho>0\) such that
\[
\sup_{v\in B_\varepsilon(u_0,\rho)}
H\!\left(\eta_i^{\mu,\ell}(\cdot\mid v)\,\big|\,
k_i^{R,\ell}(\cdot\mid v)\right)<\infty.
\]
\end{enumerate}
The small-ball lower mass~\eqref{eq:s7-small-ball-lower-mass} is the
fundamental positivity condition in \(d_\varepsilon^\ell\); the
kernel-denominator non-collapse and localiser concentration follow
as a lemma (Appendix~\ref{app:boundary-drift}).
\emph{Euclidean special case.} If a chart \(T_{u_0}\) of
\(\mathcal U_i^\ell\) is bi-Lipschitz with respect to
\(d_\varepsilon^\ell\) near \(u_0\), then \(q_{\mathrm{eff}}=\dim T_{u_0}(\mathcal U_i^\ell)\)
and (S1) follows from strict positivity and continuity of the
push-forward density at \(u_0\); in the elementary chart with
present state, past increments and latent stacked into
\(\R^{q_i+q_Y}\) one recovers the classical
\(Mh_M^{q_i+q_Y}\to\infty\) bandwidth condition. The macro form is
the fundamental one; the Euclidean chart is invoked only when
verifying (S1) in concrete examples.

\subsection{Main statistical convergence theorem}
\label{subsec:main-statistical-convergence-theorem}

The theorem below combines the fixed-\(\varepsilon\) statistical
step \(\widehat H_i^{\varepsilon,M}\to H_i^\varepsilon\) with the
deterministic analytic step
\(H_i^\varepsilon\to H_i\,\mathbf 1_{\mathcal L_i}\), and closes the
chain by a diagonal extraction in \((\varepsilon_n,M_n,h_n)\).

\begin{theorem}[Main statistical convergence]
\label{thm:approximate-intervalwise-dynamics}
Adopt the setup and shorthand of
Proposition~\ref{prop:local-intervalwise-minimiser}. Fix a compact
set \(K\Subset(t_i,t_{i+1})\times\mathcal L_i\) contained in the
backward stratum \(\{t\le t_{i+1}-\tau\}\) for some \(\tau>0\). Then
there exist sequences \(\varepsilon_n\downarrow 0\),
\(M_n\to\infty\), and admissible bandwidths \(h_{M_n}\) such that,
writing \(\widehat H_i^{(n)}\) for the empirical regularised
potential with parameters \((\varepsilon_n,M_n,h_{M_n})\):
\begin{enumerate}
\item[(i)] \textnormal{\textsc{Interior drift convergence.}}
Under (S1)--(S5),
\(\displaystyle
\sup_{(t,x)\in K}\bigl|\widehat b_i^{(n)}(t,x)-b_i^*(t,x)\bigr|\to 0
\) in probability, with
\(\widehat b_i^{(n)}:=A_i^{\varepsilon_n}\nabla_x\log\widehat H_i^{(n)}\);
a.s.\ under the stronger (S5) bandwidth.

\item[(ii)] \textnormal{\textsc{Boundary-diagonal drift convergence.}}
Under (S1)--(S6), the empirical boundary drift defined as the
finite-sample diagonal limit
\(\widehat b_i^{(n)}(t_i,x_i):=\lim_{t\downarrow t_i}\widehat b_i^{(n)}(t,x_i)\)
satisfies \(\widehat b_i^{(n)}(t_i,x_i)\to b_i^*(t_i,x_i;u_0)\) in
probability; a.s.\ under (S5).

\end{enumerate}
The boundary statement~(ii) is the only one requiring the
local entropy admissibility~(S6); statement~(i) holds under (S1)--(S5) alone.
\end{theorem}

\begin{remark}[Diagonal extraction]
\label{rem:finite-eps-vs-limit}
At fixed \(\varepsilon>0\) the regularised Gaussian kernels are
full-rank and uniformly elliptic
(\(A_i^\varepsilon\succeq\varepsilon I\),
\(\det A_i^\varepsilon\ge\varepsilon^d\)); determinant and
inverse-covariance constants depend polynomially on \(\varepsilon\)
but are harmless: at every \(\varepsilon>0\) the statistical error
can be made small by choosing \(M\) large and \(h_M\) small, and the
final convergence follows by selecting the diagonal sequence
\((\varepsilon_n,M_n,h_n)\) sufficiently slowly. No rank-stability
assumption on \(A_i\) is required: rank and active leaf enter only
in the deterministic geometric limit \(\varepsilon\downarrow 0\) via
the global collapse
\(H_i^\varepsilon\to H_i\,\mathbf 1_{\mathcal L_i}\). The compactly
supported auxiliary estimator
\(\widehat H_i^{\varepsilon,M,K}\) of
Appendix~\ref{app:compact-kernel-estimators} achieves the same
compact-uniform convergence almost surely under the stronger
bandwidth condition \(Mh_M^{q_{\mathrm{eff}}}/\log M\to\infty\).
\end{remark}

We can now turn to the finite-dimensional conditioning,
precomputation, and runtime choices of the implementation.

\section{Computational architecture}
\label{sec:computational-architecture}

This section implements the statistical estimator of
Section~\ref{sec:analytic-statistical-terminal-datum}. The only object
estimated at run time is the conditional next-increment law
\(\eta_i^\mu(\cdot\mid U_i=u)\) given a finite-dimensional augmented
conditioning variable \(U_i\), through the single reference-aware
Nadaraya--Watson surrogate
\[
\widehat\eta_i^M(\cdot\mid u)=\sum_m\omega_{i,m}(u)\,\delta_{\Delta X_{i+1}^{(m)}}
\]
of~\eqref{eq:reference-aware-NW-weights}. Throughout this section we specify the actual conditioning
variable \(U_i\), the coordinate selectors used inside it, and the joint
logweights \(\ell_{i,m}\) that produce the weights \(\omega_{i,m}\). The bridge
drift is then read off the backward heat potential associated with the empirical
terminal law, exactly as in Section~\ref{subsec:statistical-estimator}.

\subsection{Finite-dimensional conditioning variable}
\label{subsec:conditioning-variable}

The exact intervalwise conditioning is the entire augmented past, whose
dimension grows with time. We replace it by the finite-dimensional instance of
the macro-conditioning variable~\eqref{eq:macro-conditioning},
\[
U_i=\bigl(U_i^{\mathrm{anch}},\,U_i^{\mathrm{drift}},\,U_i^{\mathrm{inc}},\,U_i^{\mathrm{lat/ref}}\bigr),
\]
where \(U_i^{\mathrm{anch}}\) anchors the current augmented state, \(U_i^{\mathrm{drift}}\) is a
fixed-dimensional weighted-least-squares summary of the recent \(X\)-past
computed in the runtime reference metric, \(U_i^{\mathrm{inc}}\) holds the
past-increment comparisons measured in the intervalwise Mahalanobis geometries,
and \(U_i^{\mathrm{lat/ref}}\) carries the latent/reference information. The
conditioning is on the augmented past and reference: for the state component, \(X\) enters through past and present
information; for the latent/reference component, \(Y\) enters through past,
present, and the immediately-next latent observation.

The drift block is a causal first-order summary of the past: over a memory window of length \(L\),
\begin{equation}
\label{eq:wls-drift-regressor}
\widehat\theta_i
:=
\operatorname*{argmin}_{\theta}
\sum_{k=i-L+1}^{i}
\bigl\|\Delta X_k-m_\theta(\cdot)\bigr\|^2_{(C_k^\varepsilon)^{-1}},
\end{equation}
with \(m_\theta\) a low-order parametric drift model and
\((C_k^\varepsilon)^{-1}\) the runtime reference metric. The output
\(\widehat\theta_i\) has fixed dimension independent of \(L\), so lengthening
the memory window improves the statistical stability of this block without
increasing the dimension seen by the kernel.

\subsection{Coordinate selectors and reference-aware metrics}
\label{subsec:coordinate-selectors}

Each block \(B\in\mathcal B=\{\mathrm{anch},\mathrm{drift},\mathrm{inc},\mathrm{lat/ref}\}\)
may be followed by a statistically validated coordinate selector
\[
S_B:\R^{d_B}\to\R^{k_B},
\qquad B\in\mathcal B,
\]
so that the kernel sees \(S_B U_i^B\), not necessarily the full raw block.
Selectors may be chosen globally, blockwise, componentwise, or on the full
vector. For the state anchor, the drift summary, and the latent descriptors, a
natural choice is PCA/PCR on the corresponding coordinates, truncated by a
validated explained-variance threshold. For the reference-aware increment block
the natural coordinates are the spectral coordinates of the runtime frozen
covariance: if
\[
C_k^\varepsilon=Q_k\Lambda_k^\varepsilon Q_k^\top,
\]
then the Mahalanobis contribution is
\[
|z|_{(C_k^\varepsilon)^{-1}}^2=\sum_j\frac{\langle z,q_{k,j}\rangle^2}{\lambda_{k,j}^\varepsilon}.
\]
The eigenvalues thus play a double role: they are the metric weights of the
comparison and a criterion for coordinate selection. Small eigenvalues identify
characteristic low-noise directions, informative but possibly requiring
flooring or removal for a stable inverse metric; very large eigenvalues
correspond to highly diffusive directions, downweighted automatically by the
Mahalanobis metric and discarded when validation shows that they mostly add
noise. The floored covariances \(C_k^\varepsilon\) are those of
Section~\ref{subsec:statistical-estimator}, so the metric is bounded above and
below and the consistency analysis of
Section~\ref{subsec:main-statistical-convergence-theorem} applies unchanged
(Remark~\ref{rem:reference-aware-kernel-rationale}).

The latent/reference distance is controlled indirectly: the frozen
covariances and covariance descriptors are deterministic functions of the
latent coordinates; hence, on the validated operating region, controlling the
latent coordinates is a stable finite-dimensional proxy for controlling the
induced reference covariances, with the one-sided local bound
\[
d_{\mathrm{ref}}\bigl(\Sigma(y),\Sigma(y')\bigr)\le L\,|y-y'|
\]
on the observed compact region.

\subsection{Joint logweights and empirical terminal laws}
\label{subsec:joint-logweights}

The blockwise scores are not independent estimators. They are raw
log-scores used to select the historical terminal atoms of the empirical
conditional laws. In the joint state--latent implementation, two raw
families are computed at each step: the state score \(\ell^X_{i,j}\), built from
the state anchor, WLS summary and reference-aware increment comparisons,
and the latent score \(\ell^Y_{i,j}\), built from the latent/reference descriptor.

For the historical candidate \(j\), the implementation uses
\[
\bar \ell^X_{i,j}
:=
(1-\rho_X)\ell^X_{i,j}
+
\rho_X \ell^Y_{i,j},
\]
and
\[
\bar \ell^Y_{i,j}
:=
(1-\rho_Y)\ell^Y_{i,j}
+
\rho_Y\bigl((1-\alpha)\ell^X_{i,j}
+
\alpha \ell^{X,i}_{i,j}\bigr),
\]
with validated weights \(\rho_X,\rho_Y,\alpha\in[0,1]\). Thus \(\rho_X\)
injects latent/reference information into the state terminal law, while
\(\rho_Y\) injects state information into the latent terminal law. The
parameter \(\alpha\) separates the state contribution to the latent score
between the running state history and the present anchor \(\ell^{X,i}_{i,j}\).

The two coupled scores are normalised separately,
\[
w^X_{i,j}
=
\frac{\exp(\bar\ell^X_{i,j}-L^X_i)}
{\sum_m \exp(\bar\ell^X_{i,m}-L^X_i)},
\qquad
L^X_i:=\max_m\bar\ell^X_{i,m},
\]
and
\[
w^Y_{i,j}
=
\frac{\exp(\bar\ell^Y_{i,j}-L^Y_i)}
{\sum_m \exp(\bar\ell^Y_{i,m}-L^Y_i)},
\qquad
L^Y_i:=\max_m\bar\ell^Y_{i,m}.
\]
They define the two empirical terminal laws
\[
\widehat\eta^X_i
=
\sum_j w^X_{i,j}\,\delta_{\Delta X^{(j)}_{i+1}},
\qquad
\widehat\eta^Y_i
=
\sum_j w^Y_{i,j}\,\delta_{\Delta Y^{(j)}_{i+1}}.
\]
the bridge drifts are read off the backward heat potentials associated with them, for example the drift of the state bridge is computed as:
\begin{align*}
\widehat b_i^{\varepsilon,M, X}(t, x)&=A_i^\varepsilon\nabla_x\log\widehat H_i^{\varepsilon,M}(t, x)\\
& = \frac{\sum_m \Delta X_{i+1}^{(m)}\omega_{i,m}^M(u)\,
\Phi_i^\varepsilon\!(t,x,\Delta X_{i+1}^{(m)})}{(t_{i+1}-t)\sum_m\omega_{i,m}^M(u)\,
\Phi_i^\varepsilon\!(t,x,\Delta X_{i+1}^{(m)})} - (x-x_i)/(t_{i+1}-t),
\end{align*}
with \(\Phi_i^\varepsilon\) the kernel-ratio response
of~\eqref{eq:def-H-eps}, as in~\eqref{eq:empirical-H-eps}.

\subsection{Covariance descriptors and triangular generation}
\label{subsec:stable-coefficient-database}
\label{subsec:joint-tr-sbts-generation}

The realised lower-level variable that fixes the reference is represented in
finite coordinates by a covariance descriptor \(\Sigma_i\in\mathbb S_+^d\),
estimated from data or supplied exogenously (e.g.\ a daily realised covariance
from intraday increments), and packed as
\(\gamma_i:=\operatorname{vech}(\Sigma_i)\in\R^{d(d+1)/2}\) for regression in the
flat space of symmetric matrices. A generated descriptor may fail to be PSD and
is projected onto \(\mathbb S_+^d\) spectrally: with
\(\Sigma_i^{\mathrm{gen}}=Q_i\operatorname{diag}(\lambda_{i,1},\dots,\lambda_{i,d})Q_i^\top\),
\begin{equation}
\label{eq:psd-projection}
\Sigma_i:=Q_i\operatorname{diag}\!\bigl(\lambda_{i,1}^+,\dots,\lambda_{i,d}^+\bigr)Q_i^\top,
\qquad \lambda^+:=\max(\lambda,0).
\end{equation}
Where a non-degenerate inverse metric is needed---in the transition densities of
Section~\ref{subsec:spectral-regularisation-leafwise-kernel} and the
reference-aware Mahalanobis distances of
Section~\ref{subsec:statistical-estimator}---the descriptor is spectrally
floored with the regularisation parameter \(\varepsilon\),
\begin{equation}
\label{eq:psd-floor}
\Sigma_i^\varepsilon:=Q_i\operatorname{diag}\!\bigl(\max(\lambda_{i,1}^+,\varepsilon),\dots,\max(\lambda_{i,d}^+,\varepsilon)\bigr)Q_i^\top,
\end{equation}
with \(\Sigma_i^\varepsilon\succeq\varepsilon I\); for simulation the canonical
symmetric square root \(F_i:=\Sigma_i^{1/2}\), \(F_iF_i^\top=\Sigma_i\), is used,
continuous on the closed PSD cone and hence stable on rank-deficient matrices.

In joint generation the triangular order of
Section~\ref{subsec:conditional-to-leaf} is followed literally, as illustrated in Algorithm~\ref{alg:joint-tr-sbts}. At each coarse
step the covariance/latent component is generated first, from the joint past
\((X_{0:m},\Sigma_{0:m})\), producing the next descriptor
\(\Sigma_{m+1}\in\mathbb S_+^d\) through its own bridge step on the descriptor
sequence; this descriptor is the realised \(y_{m+1}\) that fixes the frozen
reference on \((t_m,t_{m+1}]\). The state bridge step then produces \(X_{m+1}\)
from the same joint past under that reference, using the joint logweights of
Section~\ref{subsec:joint-logweights}.

\begin{algorithm}[H]
\caption{Joint TR-SBTS triangular generation}
\label{alg:joint-tr-sbts}
\begin{spacing}{1.08}
\begin{algorithmic}[1]
\setlength{\itemsep}{2pt}

\Require fitted state and latent components, initial pair \((X_0,Y_0)\), horizon \(H\)
\State initialise raw log-scores
\(
\ell^X_{\cdot},\, \ell^{X,0}_{\cdot},\, \ell^Y_{\cdot}
\)
\For{\(i=0,\ldots,H-2\)}

    \Statex \emph{Latent step.}
    \State 
    \(
    \bar\ell^Y_{\cdot}
    =
    (1-\rho_Y)\ell^Y_{\cdot}
    +
    \rho_Y\bigl((1-\alpha)\ell^X_{\cdot}
    +
    \alpha\ell^{X,i}_{\cdot}\bigr)
    \)
    \State
    \(
    w^Y_{\cdot}
    \leftarrow
    \operatorname{softmax}(\bar\ell^Y_{\cdot}),
    \quad
    \widehat\eta^Y_i
    \leftarrow
    \sum_j w^Y_j\,\delta_{\Delta Y^{(j)}_{i+1}}
    \)
    \State generate \(Y_{i+1}\) from the latent bridge with terminal law
    \(\widehat\eta^Y_i\)
    \State
    \(
    \ell^Y_{\cdot}
    \leftarrow
    \operatorname{Update}_Y(\ell^Y_{\cdot};Y_{i+1})
    \)

    \Statex \emph{State step.}
    \State 
    \(
    A_i^\varepsilon \leftarrow A_i^\varepsilon(Y_{i+1})
    \)
    \State 
    \(
    \bar\ell^X_{\cdot}
    =
    (1-\rho_X)\ell^X_{\cdot}
    +
    \rho_X\ell^Y_{\cdot}
    \)
    \State
    \(
    w^X_{\cdot}
    \leftarrow
    \operatorname{softmax}(\bar\ell^X_{\cdot}),
    \quad
    \widehat\eta^X_i
    \leftarrow
    \sum_j w^X_j\,\delta_{\Delta X^{(j)}_{i+1}}
    \)
    \State \(X_i^{(0)}\leftarrow X_i\)
    \For{\(r=0,\ldots,n_{\mathrm{in}}-1\)}
        \State
        \(
        b_{i,r}
        \leftarrow
        A_i^\varepsilon
        \nabla_x\log \widehat H_i^{\varepsilon,M}
        (\tau_r,X_i^{(r)};\widehat\eta^X_i)
        \)
        \State
        \(
        X_i^{(r+1)}
        \leftarrow
        X_i^{(r)}
        +
        b_{i,r}\Delta\tau_r
        +
        (A_i^\varepsilon)^{1/2}\sqrt{\Delta\tau_r}\,\xi_{i,r}
        \)
    \EndFor
    \State \(X_{i+1}\leftarrow X_i^{(n_{\mathrm{in}})}\)
    \State 
    \(
    \ell^X_{\cdot},\ell^{X,i}_{\cdot}
    \leftarrow
    \operatorname{Update}_X(\ell^X_{\cdot},\ell^{X,i}_{\cdot};X_{i+1})
    \)

\EndFor
\State \Return \((X_0,Y_0),\ldots,(X_{H-1},Y_{H-1})\)
\end{algorithmic}
\end{spacing}
\end{algorithm}

\subsection{Hyperparameter selection}
\label{subsec:computational-approximation-status}

The free parameters (kernel bandwidths, memory lengths, spectral floors,
selector dimensions, and coupling weights \(\rho_X,\rho_Y,\alpha\)) are
selected by predictive validation against the energy score of
Appendix~\ref{sec:appendix-predictive-energy-score}. The covariance
descriptor is validated first, then the state component, then the coupling
weights, each at a progressively longer predictive horizon. The reference
family used at each latent level is selected, and its closed-loop emission
validated, through the entropic surrogate of
Appendix~\ref{sec:appendix-entropic-reference-selection-and-validation}.

\section{Numerical experiments}
\label{sec:numerical-experiments}

We test TR-SBTS along two complementary axes. The first
(Section~\ref{ssec:low-rank-dimensional-stress}) probes the
\emph{dimensional} robustness of the kernel-PCR conditioning: a
diffusion whose true dynamics live on a low-dimensional subspace is
embedded in increasingly high ambient dimension, and we ask whether a
PCR-truncated TR-SBTS keeps tracking the true subspace once the
ambient dimension is much larger than the intrinsic one. The second
(Section~\ref{ssec:heston-parameter-recovery}) examines the
\emph{stochastic-volatility} case: synthetic paths from a calibrated
Heston model are used as ground truth, TR-SBTS is fitted on them, and
the per-parameter Heston QMLE distribution recovered from the
synthetic paths is compared with the QMLE distribution recovered from
the real paths.

\subsection{Low-rank dimensional stress}
\label{ssec:low-rank-dimensional-stress}

We embed a non-trivial two-dimensional motion in an ambient space of
growing dimension and ask whether the estimator continues to
recover its predictive distribution as the codimension grows. The
informative motion is a Hopf oscillation on \((X_{1},X_{2})\) — a stable
limit cycle of unit radius with non-linear (cubic) drift, dominantly
tangential, so that every cycle phase has a different conditional
mean and a globally smoothed kernel cannot match the ground truth. The
remaining \(d-2\) coordinates are an independent fast Ornstein--Uhlenbeck
process whose contribution to the conditional law of any future state
is, in the asymptotic limit \(\lambda_{\perp}\!\to\!\infty\), a pure
isotropic Brownian increment over the macro step. The trajectory is
generated by an exact transition kernel on a trading-day grid
(\(dt\!=\!1/250\), \(20\) calendar years per run); we sweep
\(d\in\{4,8,16,32,64,128,256,512\}\) with three seeds per dimension and
keep the training-set size fixed at \(M=4500\), so as \(d\) grows the
ratio \(M/d\) collapses from \(1125\) to \(8.8\). The signal subspace
is taken canonical, \(\mathrm{span}(e_{1},e_{2})\); under any random
orthonormal loading the same conclusions hold up to a fixed rotation,
but the canonical choice lets the test scoring be performed on the
first two coordinates without a data-driven projection.

The kernel weights are the \emph{quartic compact-support kernel}
of \cite{hamdouche2023generative} (eq.~4.2),
\(K_{h}(u)=(1-\|u\|^{2}/h^{2})\,\mathbf{1}_{\{\|u\|\le h\}}\), used as a
multiplicative factor at every past time point. This is the
architectural choice that exposes the conditioning effect: a Gaussian
softmax cancels noise-axis contributions in its own normalisation and
masks the curse of dimensionality on this DGP, while a hard support
cutoff lets the perpendicular squared-distance accumulate freely and
push every candidate outside the support as \(d\) grows. We compare two
variants: \texttt{classic\_no\_pcr}, which conditions the kernel on the
full \(d\)-dimensional state, and \texttt{classic\_pcr}, which conditions
it on the validated PCA-truncated state — and the validated truncation
lands on \(k\!=\!2\) at every \(d\), since the position covariance has a
clean step at rank \(2\) by construction. Hyperparameters are validated
per dimension on a held-out validation block; the test metric is the
one-step predictive multivariate energy score on the test block,
sliced to the first two ambient coordinates and normalised by
\(\sqrt{q}\) so that values are directly comparable across \(d\). With
the chosen DGP parameters the score is bounded below by the Bayes-optimal
floor \(\widetilde{\mathrm{ES}}^{\star}\!\approx\!0.040\) (the entropy of
the one-step Brownian transition on the signal subspace) and above by
the predict-the-marginal-mean floor of order \(0.7\); reporting both
the absolute score and its excess over the Bayes-optimal floor isolates
the curse-of-dimensionality contribution to the predictive error.

\begin{figure}[h]
  \centering
  \begin{tikzpicture}
    \begin{axis}[
      width=0.85\linewidth, height=6.5cm,
      xmode=log, log basis x=2, xtick={4,8,16,32,64,128,256,512},
      xticklabels={4,8,16,32,64,128,256,512},
      xlabel={ambient dimension \(d\)},
      ylabel={\(\widetilde{\mathrm{ES}}_{d}\)},
      ymin=0.030, ymax=0.16,
      grid=both, grid style={gray!20},
      legend pos=north west, legend cell align=left,
    ]
      \addplot[dashed, gray, thick, domain=4:512] {0.040};
      \addlegendentry{Bayes-optimal floor};
      \addplot+[mark=*, color=red!70!black, thick, error bars/.cd,
               y dir=both, y explicit]
        coordinates {
          (4,0.066) +- (0,0.003)
          (8,0.069) +- (0,0.003)
          (16,0.092) +- (0,0.021)
          (32,0.114) +- (0,0.017)
          (64,0.116) +- (0,0.009)
          (128,0.107) +- (0,0.007)
          (256,0.109) +- (0,0.009)
          (512,0.129) +- (0,0.014)
        };
      \addlegendentry{\texttt{classic\_no\_pcr}};
      \addplot+[mark=square*, color=blue!70!black, thick, error bars/.cd,
               y dir=both, y explicit]
        coordinates {
          (4,0.066) +- (0,0.005)
          (8,0.059) +- (0,0.003)
          (16,0.053) +- (0,0.007)
          (32,0.067) +- (0,0.017)
          (64,0.061) +- (0,0.003)
          (128,0.063) +- (0,0.006)
          (256,0.067) +- (0,0.007)
          (512,0.078) +- (0,0.010)
        };
      \addlegendentry{\texttt{classic\_pcr}};
    \end{axis}
  \end{tikzpicture}
  \caption{Predictive energy score on the canonical signal subspace as
    a function of the ambient dimension \(d\). The PCR variant stays
    within \(2\times\) the Bayes-optimal floor at every \(d\); the
    no-PCR baseline degrades from \(\approx\!1.7\times\) the floor at
    \(d=4\) to \(\approx\!3.2\times\) at \(d=512\). Error bars are one
    standard deviation across three independent seeds.}
  \label{fig:lowrank-energy-vs-d}
\end{figure}
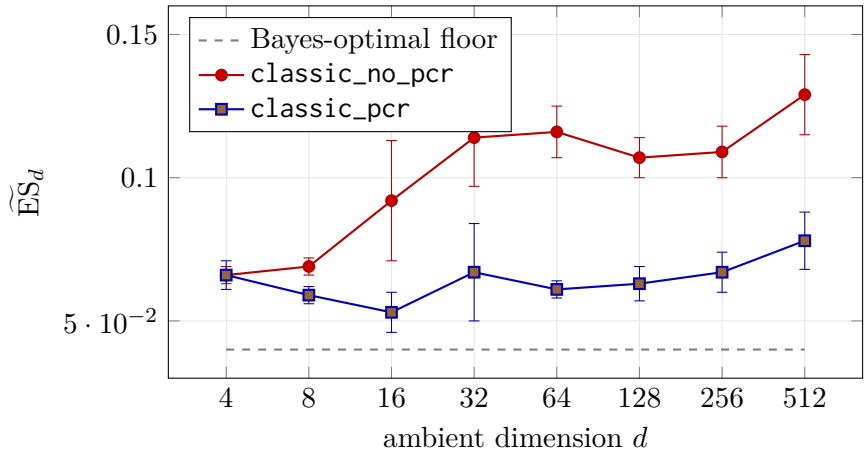

Figure~\ref{fig:lowrank-energy-vs-d} summarises the experiment. The
PCR variant is essentially flat in \(d\): its excess over the
Bayes-optimal floor stays in a narrow \([0.013,0.039]\) band across
two orders of magnitude in the codimension. The no-PCR baseline
degrades monotonically; its excess over the same floor grows by roughly
a factor of three from \(d=4\) to \(d=512\), with the PCR/no-PCR ratio
on the excess climbing to \(\approx\!4\) at the intermediate
dimensions and stabilising around \(2.5\) at the largest. The two
variants coincide at \(d=4\), where there are no perpendicular axes
to exploit, and diverge as soon as the codimension is comparable to
the signal dimension. This is the regime the construction is designed
to absorb: a kernel that is contracted on the data manifold by PCR
keeps contact with the data at every ambient dimension, while a
kernel that sees the full ambient state loses contact as the noise
dimension grows.

\subsection{Heston parameter recovery}
\label{ssec:heston-parameter-recovery}

The second experiment runs the three-layer TR-SBTS pipeline end-to-end
on a canonical stochastic-volatility model. Let \((S_{t},V_{t})\) solve
the standard Heston system
\begin{equation}
\label{eq:heston-system-numexp}
\begin{aligned}
dS_{t} &= \mu\, S_{t}\, dt + \sqrt{V_{t}}\, S_{t}\, dW_{t}^{S}, \\
dV_{t} &= \kappa\,(\theta - V_{t})\, dt + \xi\,\sqrt{V_{t}}\, dW_{t}^{V},
\qquad d\langle W^{S},W^{V}\rangle_{t} = \rho\, dt,
\end{aligned}
\end{equation}
with parameters \((\kappa,\theta,\xi,\rho)\) drawn per path from a fixed
prior and a horizon of \(T=252\) coarse increments at \(\Delta t = 1/T\).
The observable state is \(\zeta_{t}:=(\log S_{t},V_{t})\). The
experiment is a stress test on how informative a TR-SBTS reference
constructed \emph{purely from data} — without any parametric Heston
assumption inside the model — can be made, when the data is the
diffusion of \(\zeta_{t}\) itself. Throughout this subsection,
\(X\), \(Y\), and \(Z\) denote the three layers of the pipeline,
from the observable level downwards; in particular, \(Z\) here is the
tertiary layer state, not the augmented state of
Section~\ref{sec:reference-based-sbts-projection}.

\subsubsection*{Three layers and what they encode}

The pipeline is the three-layer TR-SBTS construction of
Section~\ref{subsec:joint-tr-sbts-generation}, instantiated with the
following causal interpretation of the layers:
\begin{description}
\item[Primary (\(X\)-layer).] State
  \(X_{t}\equiv\zeta_{t}\in\mathbb{R}^{2}\), the observable pair
  of log-price and variance. This is the only level whose generated
  output the user consumes.
\item[Secondary (\(Y\)-layer).] State \(Y_{t}\in\mathbb{R}^{3}\),
  the packed lower triangle of the \emph{cumulative average} of the
  outer products of past coarse increments of \(\zeta_{t}\),
  \[
    M_{t} \;=\; \frac{1}{t\,\Delta t}\sum_{i=1}^{t}
    \Delta\zeta_{i}\,(\Delta\zeta_{i})^{\top},
  \]
  packed in lower-triangular form. \(M_{t}\) is a fully causal,
  parameter-free running estimate of the path's average covariance
  rate; the secondary models its time evolution.
\item[Tertiary (\(Z\)-layer).] State \(Z_{t}\) is a
  \emph{hybrid-frame parametrisation} of what we call the path's
  average covariance \emph{ribbon}: the path-specific curve in
  \(\mathrm{Sym}^{+}_{2}\) along which the cumulative covariance
  accumulates. Concretely, \(Z_{t}\) is a finite-dimensional vector
  whose first component is the dominant principal direction of the
  path's cumulative covariance estimator (after normalisation by its
  \(X\)-variance), and whose remaining components are the
  normalisations of the second and first principal components of
  that same estimator. The tertiary thus encodes the
  \emph{geometry} of where the diffusion lives, independently of the
  magnitude controlled by \(V_{t}\).
\end{description}

The reference at each level is produced by a deterministic
\emph{backward map} from the state of the level immediately below,
applied at every closed-loop step. The map at the upper edge
(tertiary~$\rightarrow$~secondary) sends a hybrid-frame point back
into the cone of symmetric PSD matrices, picking out the rank-one
covariance ribbon parametrised by that point; the map at the lower
edge (secondary~$\rightarrow$~primary) is the unpacking of the
\(3\)-vector into a \(2\times 2\) symmetric matrix used as the primary's
runtime diffusion covariance. Both maps are closed-form and
state-independent inside each coarse interval, in line with the
intervalwise-frozen reference structure of
Section~\ref{sec:exact-intervalwise-dynamics}.

\subsubsection*{Data-driven references}

The most delicate aspect of the experiment is the choice of references
the upper layers expose to the lower ones. Each is read off the data
without ever invoking a parametric form of the Heston diffusion:
\begin{itemize}
\item The secondary's training stream is the cumulative-average
packed covariance \(M_{t}\) of \(\zeta_{t}\) defined above. No
smoothing kernel, no half-life, no tuning: the running mean of the
squared increments is the most parsimonious causal estimator of
where the increments are diffusing on average up to time \(t\).
\item The tertiary's training stream is the hybrid-frame
parametrisation of the path's terminal cumulative covariance
\(M_{T}\), obtained by a principal-component analysis on \(\zeta\)
and a backward-map reduction into the ribbon coordinates that the
secondary will consume as reference.
\end{itemize}
Both streams are computed per path on the training pool only, and
each layer is then fitted by the kernel-regression estimator of
Section~\ref{subsec:statistical-estimator} against its own stream as
if it were the ground truth. No information about the Heston law
enters the model.

\subsubsection*{Implementation}

The model implementing this pipeline is a three-layer TR-SBTS as
described in Section~\ref{sec:computational-architecture}, run on
\(N_{\mathrm{train}}\) training paths and validated on a held-out pool
of \(N_{\mathrm{val}}\) paths. The closed-loop generation warm-starts
on the path's real history of length \(H = 64\) and then evolves the joint
\((Z,Y,X)\) state by the joint TR-SBTS generator for the remaining
\(T-H\) coarse steps. The secondary and tertiary bandwidths,
history lengths, and integration controls are selected by a
predictive-energy validation on the same pool.

\subsubsection*{Results}

We evaluate the closed-loop synthetic paths through the same per-path
Heston-like estimators of \((\kappa,\theta,\xi,\rho)\) that one would
apply to the real series, and report their cross-path distributions on
the held-out pool. Figure~\ref{fig:heston-densities} shows, for the
selected configuration, the four parameter densities on the synthetic
trajectories overlaid on the densities obtained on the held-out real
trajectories and on those produced by the SBTS-Classic baseline of
\cite{alouadi2025robust}, which fits a single kernel-regression layer
against a frozen, state-independent Brownian reference. The classical
baseline recovers the mean-reversion level \(\theta\) and the long-run
drift, but the variance-of-variance \(\xi\) and the leverage \(\rho\)
collapse onto the empirical priors of those parameters, as the frozen
Brownian reference carries no information about the path-specific
diffusion direction. The TR-SBTS pipeline, by contrast, replaces that
frozen reference with the variable, intervalwise reference produced by
the tertiary~$\rightarrow$~secondary backward map, and as a consequence
recovers \emph{also} the geometry-defining pair \((\xi,\rho)\)
distributionally; the long-run level \(\theta\) is recovered within
the empirical finite-sample fluctuation of the per-path estimator;
the mean-reversion rate \(\kappa\) is the least informative channel at
this horizon and bandwidth, consistent with its weak identifiability
from a fixed-length single trajectory. Allowing the reference to vary
along the path is what enables the recovery of the diffusion
structure that an isochronous Brownian reference does not encode.

\begin{figure}[h]
\centering
\includegraphics[width=0.95\textwidth]{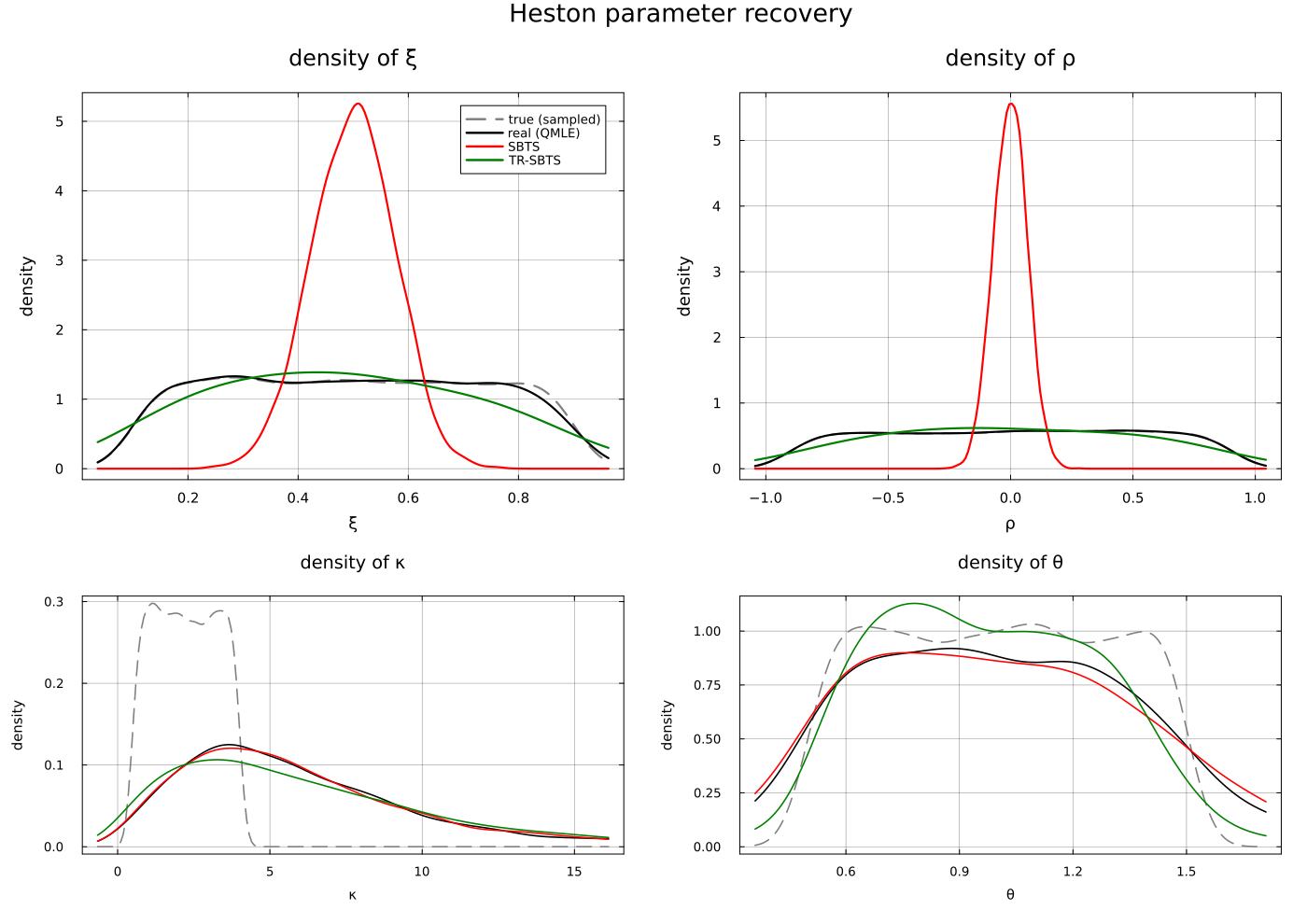}
\caption{Heston parameter recovery on closed-loop synthetic
trajectories. Real held-out paths, TR-SBTS synthetic paths under the
selected configuration, and SBTS-Classic baseline. SBTS-Classic
recovers \(\theta\) but collapses \((\xi,\rho)\) onto the empirical
priors. TR-SBTS additionally recovers the geometry pair
\((\xi,\rho)\) distributionally; the mean-reversion rate \(\kappa\) is
the least identifiable channel at this horizon.}
\label{fig:heston-densities}
\end{figure}

\paragraph{Hardware and software environment.}
\label{par:hardware-software-env}
The experiments above were performed on a laptop with the following
characteristics: Intel(R) Core(TM) i7-8565U CPU at 1.80\,GHz (turbo
up to 1.99\,GHz), 4~physical cores and 8~logical processors,
16\,GB of RAM, and an integrated Intel UHD Graphics 620 GPU (no
discrete accelerator), running Microsoft Windows~11 Pro
(build~10.0.26200). All algorithms were implemented in Julia
(version~1.12.3). The core TR-SBTS solver depends only on the Julia
standard library; the elementary aggregations used at run time come
from the \texttt{Statistics}~1.11.1 standard module. The synthetic
datasets are generated using
\texttt{DifferentialEquations.jl}~7.17.0,
\texttt{Distributions.jl}~0.25.123 and \texttt{StatsBase.jl}~0.34.9;
the Heston experiment additionally uses \texttt{Optim.jl}~1.13.3 for
the calibration step. All figures are produced with
\texttt{Plots.jl}~1.41.4. The TR-SBTS solver runs on the CPU only;
no GPU acceleration is used.

\section{Conclusion and outlook}
\label{sec:conclusion}

We have developed \emph{TR-SBTS}, a conservative extension of the
Schrödinger Bridge for Time Series framework: the entropy-minimisation
problem on path space is unchanged and only the reference is modified. The new
reference is triangular and intervalwise frozen: it incorporates a
volatility structure depending on the coarse past and on a latent
environment, while remaining state-independent inside each coarse
interval, and iterating the construction over the latent environment
yields a hierarchy of conditioning levels, each of which is itself a
TR-SBTS conditional on a higher-level volatility signal. The whole
construction is a single entropy projection on the augmented state
space, with the constraint imposed jointly across time and levels and
unfolded by the disintegration of entropy; in particular, arbitrary
cross-dependence between price and volatility paths is allowed.

This choice of reference preserves the transparency of the original
framework. The minimiser is identified explicitly as an
\(h\)-transform of the reference law, and, conditionally on the
coarse past and on the latent environment, each interval carries a
Gaussian leaf problem, possibly degenerate. Its optimal dynamics are
described by backward heat potentials on the affine leaves generated
by the active covariance directions, yielding the intrinsic
logarithmic-gradient drift \(A\,\nabla\log H\), which remains
meaningful when the covariance is degenerate or of variable rank.

The freezing step is itself controlled: the discrepancy between the
ideal volatility-informed reference and its intervalwise frozen
approximation is measured by a cumulant interpolation error, and the
corresponding bridges are stable
(Proposition~\ref{prop:piecewise-frozen-bridge-stability}).

On the statistical side, the only conditional object that must be
estimated is the next-increment law, and convergence of the
regularised kernel estimators is established for the associated heat
potentials and drifts. The
implementation is non-parametric and reference-aware: a
finite-dimensional conditioning map (PCR-based, or built on the
local Mahalanobis geometry induced by the frozen covariance
cumulants) feeds a coupled state--covariance bridge step in which
each latent level produces a dynamic reference for the level above,
summarised by a covariance descriptor; the symmetric square-root
factorisation enters only at simulation time, for density evaluation
and local geometric normalisation. The numerical experiments of
Section~\ref{sec:numerical-experiments} illustrate this behaviour in
a low-rank dimensional stress test and in a Heston parameter-recovery
study.

\medskip

Several directions remain open.

\paragraph{Jumps.}
The present formulation is continuous-path. A natural extension
combines the intervalwise frozen diffusion reference with a jump
component, in the spirit of Schrödinger bridges with jumps. This
would let the model treat stochastic volatility and discontinuous
path behaviour within a single framework, and is empirically
relevant for financial and energy time series where abrupt moves,
heavy tails, and regime switches are central.

\paragraph{Conditional generation.}
The explicit Brownian noise injected at each inner step, together
with the path-dependent bridge mechanism, makes the framework
structurally suited to conditional generation under pathwise
events---barrier- or drawdown-type constraints, scenario
reconstruction on rare-event subsets, and more generally the
simulation of trajectories subject to path-dependent conditions.
We intend to pursue this direction in future work.

\paragraph{Neural estimators.}
The kernel-based estimator used here is interpretable but exposed
to bandwidth sensitivity and to the curse of dimensionality. A
controlled introduction of neural components into the bridge
mechanism, preserving the analytic stability guarantees and the
transparency of the path-space construction, is a natural direction
of investigation.

\paragraph{Multi-resolution data.}
The cumulant stability result of
Section~\ref{subsec:volatility-informed-freezing} is itself a
multi-resolution statement: refining the volatility grid inside each
coarse interval recovers the continuous volatility limit
(Proposition~\ref{prop:piecewise-frozen-bridge-stability}). The same
principle extends to settings in which the volatility process is
observed or updated at a finer rate than the state, and more
generally to interacting processes with non-identical observation
grids, where each component contributes information when it is
observed while the latent continuous-time interpolation keeps the
joint dynamics coherent.

\section*{Acknowledgements}

The author thanks Arakne S.r.l.\ for providing the working
environment in which this research was carried out and for the
original suggestion to look into the SBTS framework.

The author also thanks Jacopo Stortini for sustained collaboration on
the volatility-related components of the construction throughout its
development, and for serving as a critical interlocutor whose
questions led to several methodological refinements.

The author is particularly indebted to Professor Alessio Porretta,
whose analytical training and sustained guidance underpin the
approach taken here. The choice to handle the local conditional
dynamics through a backward heat-equation regularisation, rather
than via direct drift estimation, traces back to his methodological
influence, as does the introduction to the surrounding literature on
mean-field Schrödinger bridges and finite-entropy Girsanov theory.

Finally, the author thanks Professor Lidia Aceto for several
stimulating exchanges that proved decisive in reorienting the
numerical approach in a non-trivial direction, and for the
perspective that her expertise in numerical analysis brought to bear
on the ill-conditioned aspects of the problem.

\appendix

\section{Technical proof of the statistical approximation theorem}
\label{sec:appendix-statistical-approximation}
This appendix contains the proof of the statistical approximation Theorem~\ref{thm:approximate-intervalwise-dynamics}.
The argument is split into three layers. 
First, we replace the degenerate leafwise heat kernel by its spectrally
floored full-rank extension.
Second, we prove consistency for compactly supported regression kernels.
Third, we transfer the result to the Gaussian kernels used in the
implementation by a truncation argument.

\subsection{Spectral regularisation of the leafwise backward kernel}

\begin{proposition}[Spectral regularisation of the leafwise backward kernel]
\label{prop:spectral-regularisation-leafwise-kernel}
Adopt the setup of Section~\ref{sec:exact-intervalwise-dynamics},
let \(\mu\) satisfy the finite-entropy
condition~\eqref{eq:finite-entropy-assumption}, and fix
\(i\in\{0,\dots,N-1\}\). 

For \(\mu_i\)-a.e.
\((\mathbf x_i,y_{i+1})\) the following hold.

\begin{enumerate}
\item[(i)] \textnormal{\textsc{Smoothness.}}
For every \(\varepsilon>0\), the regularised kernel
\(p_i^\varepsilon(\cdot,\cdot;t_{i+1},\cdot)\)
of~\eqref{eq:def-p-eps} is smooth and strictly positive on
\([t_i,t_{i+1})\times\R^d\times\R^d\), and the regularised
potential \(H_i^\varepsilon\) of~\eqref{eq:def-H-eps} is smooth in
\(x\) on \(\R^d\) and continuous in \(t\) on \([t_i,t_{i+1})\).

\item[(ii)] \textnormal{\textsc{Global collapse.}}
As \(\varepsilon\downarrow 0\),
\begin{equation}
\label{eq:H-eps-equals-H-app}
H_i^\varepsilon(t,x)
\longrightarrow
H_i(t,x)\,\mathbf 1_{\mathcal L_i}(x)
\qquad
\text{pointwise on }(t_i,t_{i+1})\times\R^d,
\end{equation}
locally uniformly on compact subsets of
\((t_i,t_{i+1})\times\mathcal L_i\) (where the limit is the exact
leafwise potential \(H_i\)) and on every compact set in
\((t_i,t_{i+1})\times\R^d\) disjoint from
\((t_i,t_{i+1})\times\mathcal L_i\) (where the limit is \(0\)).

\item[(iii)] \textnormal{\textsc{Drift convergence on compacta.}}
As \(\varepsilon\downarrow 0\), for every compact
\(K\Subset(t_i,t_{i+1})\times\mathcal L_i\),
\begin{equation}
\label{eq:regularised-drift-equals-exact-on-leaf}
\sup_{(t,x)\in K}
\bigl|A_i^\varepsilon\,\nabla_x\log H_i^\varepsilon(t,x)-b_i^*(t,x)\bigr|
\longrightarrow 0,
\end{equation}
in the intrinsic Mahalanobis geometry of \(A_i\).

\item[(iv)] \textnormal{\textsc{Drift convergence at the boundary
\((t_i,x_i)\).}}
For every \(\varepsilon>0\) sufficiently small, the diagonal limit
\begin{equation}
\label{eq:b-eps-boundary-value-app}
b_i^\varepsilon(t_i,x_i)
:=
\lim_{t\downarrow t_i}A_i^\varepsilon\,\nabla_x\log H_i^\varepsilon(t,x_i)
\end{equation}
exists, and \(b_i^\varepsilon(t_i,x_i)\to b_i^*(t_i,x_i)\) as
\(\varepsilon\downarrow 0\), where \(b_i^*(t_i,x_i)\) is the starting
drift of Proposition~\ref{prop:local-intervalwise-minimiser}.
\end{enumerate}
\end{proposition}

\begin{proof}
Fix \((\mathbf x_i,y_{i+1})\) in a full \(\mu_i\)-measure set for which
Corollary~\ref{cor:terminal-conditional-law} and
Proposition~\ref{prop:local-intervalwise-minimiser} hold. Write
\(A:=A_i\), \(r:=r_i\), \(k:=d-r\), and \(L:=\mathcal L_i=\operatorname{Ran}A\)
(linear subspace of admissible increments). Let \(P\), \(P^\perp\)
denote the orthogonal projections onto \(L\) and \(\ker A\), and set
\[
\Pi x := x_i + P(x-x_i),
\quad
d_\perp(x):=|P^\perp(x-x_i)|,
\quad
\alpha:=t_{i+1}-t,
\quad
\beta:=t_{i+1}-t_i,
\]
so that \(0<\alpha<\beta\) on \((t_i,t_{i+1})\). By
Corollary~\ref{cor:terminal-conditional-law}, \(\eta_i^\mu\) is
supported on \(L\); throughout, \(\delta\in L\) denotes the
increment variable.

\paragraph{Step 1: factorisation along leaf-tangent and leaf-transverse parts.}
Since the conditioning pair \((\mathbf x_i,y_{i+1})\) is fixed, the
active threshold \(\lambda_i^{\min,+}>0\) is a.s.\ strictly positive
and \(\varepsilon\) eventually falls below it; we conduct the
analysis on this eventual range
\(\varepsilon\in(0,\lambda_i^{\min,+})\), the only regime that
matters in the limit \(\varepsilon\downarrow 0\). There the floor
leaves the active eigenvalues unchanged and replaces the kernel
directions by \(\varepsilon\). Since \(\delta\in L\) and
\(x_i+\delta-x_i=\delta\) has no transverse component, while
\(x_i+\delta-x=\delta-P(x-x_i)-P^\perp(x-x_i)\) splits into active
and transverse parts; writing \(\delta':=\delta-P(x-x_i)\),
\begin{align*}
\bigl\langle (A_i^\varepsilon)^{-1}(x_i+\delta-x),x_i+\delta-x\bigr\rangle
&=
\bigl\langle A_i^+\delta',\delta'\bigr\rangle
+\frac{d_\perp(x)^2}{\varepsilon},
\\
\bigl\langle (A_i^\varepsilon)^{-1}\delta,\delta\bigr\rangle
&=
\bigl\langle A_i^+\delta,\delta\bigr\rangle.
\end{align*}
With \(\det A_i^\varepsilon=(\det{}'A)\varepsilon^k\), the
\(\varepsilon\)-dependent normalisation
\((2\pi(\cdot))^{-d/2}\allowbreak(\det A_i^\varepsilon)^{-1/2}\)
cancels between numerator and denominator. Combining the remaining factors,
\begin{equation}
\label{eq:eps-kernel-factorisation}
\Phi_i^\varepsilon(t,x,\delta)
:=
\frac{p_i^\varepsilon(t,x;t_{i+1},\delta)}
     {p_i^\varepsilon(t_i,x_i;t_{i+1},\delta)}
=
\exp\!\Bigl(
-\frac{d_\perp(x)^2}{2\alpha\varepsilon}
\Bigr)\,
\frac{p_i(t,\Pi x;t_{i+1},\delta)}
     {p_i(t_i,x_i;t_{i+1},\delta)},
\qquad \delta\in L,
\end{equation}
where \(p_i\) is the unfloored leafwise heat kernel of \(A_i\) as a
density on the increment. Integrating~\eqref{eq:eps-kernel-factorisation}
against \(\eta_i^\mu(d\delta)\) yields the global factorisation
\begin{equation}
\label{eq:H-eps-factorisation}
H_i^\varepsilon(t,x)
=
\exp\!\Bigl(
-\frac{d_\perp(x)^2}{2\alpha\varepsilon}
\Bigr)\,
H_i(t,\Pi x).
\end{equation}

\paragraph{Step 2: Gaussian-ratio domination via the \(\alpha<\beta\) gap.}
For \(\delta\) ranging in \(\R^d\) (not only on \(L\)), completing
the square in \(\delta\) in the joint quadratic form of numerator
and denominator gives
\begin{equation}
\label{eq:eps-kernel-ratio-after-square}
\Phi_i^\varepsilon(t,x,\delta)
=
\Bigl(\tfrac{\beta}{\alpha}\Bigr)^{d/2}
\exp\!\Bigl(
-\tfrac{1}{2}\bigl(\tfrac{1}{\alpha}-\tfrac{1}{\beta}\bigr)\,
\|\delta-m_{t,x}\|^2_{(A_i^\varepsilon)^{-1}}
+R_{t,x}^\varepsilon
\Bigr),
\end{equation}
with
\(m_{t,x}:=(\tfrac{1}{\alpha}-\tfrac{1}{\beta})^{-1}\,\tfrac{1}{\alpha}(x-x_i)\)
and \(R_{t,x}^\varepsilon\) collecting the \(\delta\)-independent
residual (which embeds the transverse exponential
of~\eqref{eq:eps-kernel-factorisation}). The crucial point is the
gap \(\alpha<\beta\): the \emph{numerator} \(p_i^\varepsilon\)
diffuses over the residual time \(\alpha\) while the \emph{prior}
\(p_i^\varepsilon(t_i,x_i;\cdot)\) diffuses over the full interval
\(\beta\), so the combined quadratic coefficient
\(\tfrac{1}{\alpha}-\tfrac{1}{\beta}>0\) is strictly positive. Since
\((A_i^\varepsilon)^{-1}\succeq(\|A_i\|+\varepsilon)^{-1}I\), the
quadratic form in \(\delta\)
of~\eqref{eq:eps-kernel-ratio-after-square} is strictly
negative-definite uniformly on \(\{t\le t_{i+1}-\tau\}\) and on
bounded \(x\)-sets, yielding the Gaussian-ratio bound
\[
\Phi_i^\varepsilon(t,x,\delta)
\le
C_{\tau,B}\,
\exp\!\bigl(-c_{\tau,B}\,\|\delta-m_{t,x}\|^2_{(A_i^\varepsilon)^{-1}}\bigr),
\]
with constants independent of small \(\varepsilon\) on the active
eigenspace. This is the dominated integrand used throughout the
appendix.

\paragraph{Step 3: smoothness, on-leaf identity, off-leaf decay.}
Smoothness of \(p_i^\varepsilon\) on
\([t_i,t_{i+1})\times\R^d\times\R^d\) is immediate from the explicit
Gaussian form, for every \(\varepsilon>0\). Smoothness of
\(H_i^\varepsilon\) in \(x\) on \(\R^d\) follows by differentiation
under the integral, justified by the Gaussian-ratio domination of
Step~2.

On the leaf \(x\in L\) one has \(\Pi x=x\) and \(d_\perp(x)=0\), so
\eqref{eq:H-eps-factorisation} reduces to
\(H_i^\varepsilon(t,x)=H_i(t,x)\) eventually (as soon as
\(\varepsilon<\lambda_i^{\min,+}\)) on \((t_i,t_{i+1})\times L\); in
particular this gives the locally-uniform-on-compacta convergence
stated in~(ii). Off the leaf, \(d_\perp(x)>0\) and
\eqref{eq:H-eps-factorisation} gives
\[
H_i^\varepsilon(t,x)
\le C\,\exp\!\bigl(-d_\perp(x)^2/(2\alpha\varepsilon)\bigr)\to 0
\]
as \(\varepsilon\downarrow 0\), uniformly on every compact set where
\(d_\perp\ge\rho>0\). Combining the two regimes yields the global
pointwise collapse~\eqref{eq:H-eps-equals-H-app}.

\paragraph{Step 4: active directional derivatives.}
Differentiating~\eqref{eq:def-H-eps} in \(x\) along an eigenvector
\(q_{i,j}\) inserts a polynomial-in-\((\delta-(x-x_i))\) factor against
the same dominated integrand
of~\eqref{eq:eps-kernel-ratio-after-square}; dominated convergence
under the Gaussian-ratio bound of Step~2 yields, for every compact
\(K\Subset(t_i,t_{i+1})\times(x_i+\mathcal L_i)\),
\begin{equation}
\label{eq:H-eps-kernel-derivative-vanishes}
\sup_{(t,x)\in K}
\bigl|D_{q_{i,j}}H_i^\varepsilon(t,x)
-D_{q_{i,j}}H_i(t,x)\bigr|
\longrightarrow 0,
\qquad j=1,\dots,d.
\end{equation}

\paragraph{Step 5: drift convergence on compacta.}
Strict positivity of \(H_i\) on \((t_i,t_{i+1})\times L\) and the
uniform-on-compact convergence of both \(H_i^\varepsilon\) and its
active directional derivatives
\eqref{eq:H-eps-kernel-derivative-vanishes} give
\(A_i^\varepsilon\nabla_x\log H_i^\varepsilon\to A_i\nabla_x\log H_i\)
on compact subsets of \((t_i,t_{i+1})\times L\). The eigenvalue
weights are bounded by
\(\max(\lambda_{i,j},\varepsilon)\le\|A_i\|+\varepsilon\) uniformly
in \(\varepsilon\); the kernel-direction contributions are
multiplied by \(A_i^\varepsilon\) and absorbed by the off-leaf
collapse of Step~3. This
establishes~\eqref{eq:regularised-drift-equals-exact-on-leaf}.

\paragraph{Step 6: drift convergence at \((t_i,x_i)\).}
By Step~1, eventually (i.e.\ as soon as
\(\varepsilon<\lambda_i^{\min,+}\)) \(x_i\in L\) gives
\(d_\perp(x_i)=0\), so \(H_i^\varepsilon(t,x_i)=H_i(t,x_i)\) and
\(A_i^\varepsilon\nabla_x\log H_i^\varepsilon(t,x_i)
=b_i^*(t,x_i)\) for every \(t\in(t_i,t_{i+1})\). It therefore
suffices to show that the diagonal limit
\(\lim_{t\downarrow t_i}b_i^*(t,x_i)\) exists; the value
\(b_i^\varepsilon(t_i,x_i)\) in~\eqref{eq:b-eps-boundary-value-app}
then equals that limit for all sufficiently small \(\varepsilon\),
yielding \(b_i^\varepsilon(t_i,x_i)\to b_i^*(t_i,x_i)\) as
\(\varepsilon\downarrow 0\).

Differentiating the coherent ratio gives
\(\nabla_x\log\Phi_i^\varepsilon(t,x,\delta)
=\frac{1}{\alpha}(A_i^\varepsilon)^{-1}(\delta-(x-x_i))\); hence the
kernel-ratio average representation, at \(x=x_i\),
\begin{equation}
\label{eq:drift-kernel-ratio-at-xi}
A_i^\varepsilon\,\nabla_x\log H_i^\varepsilon(t,x_i)
=
\frac{1}{\alpha}\,
\frac{\int_{\mathcal L_i}\delta\,
\Phi_i^*(t,x_i,\delta)\,\eta_i^\mu(d\delta)}
{\int_{\mathcal L_i}\Phi_i^*(t,x_i,\delta)\,\eta_i^\mu(d\delta)},
\end{equation}
where, by Step~1 with \(d_\perp(x_i)=0\),
\(\Phi_i^\varepsilon(t,x_i,\delta)\) reduces on \(\mathcal L_i\) to
the unfloored leafwise ratio
\[
\Phi_i^*(t,x_i,\delta)
:=
\Bigl(\tfrac{\beta}{\alpha}\Bigr)^{r_i/2}
\exp\!\Bigl(
-\tfrac{1}{2}\bigl(\tfrac{1}{\alpha}-\tfrac{1}{\beta}\bigr)\,
\|\delta\|^2_{A_i^+}
\Bigr),
\qquad \delta\in\mathcal L_i,
\]
independent of \(\varepsilon\). For any
\(\delta_0\in(0,\beta)\) and \(t\in(t_i,t_i+\delta_0]\), the
\(\alpha<\beta\) gap of Step~2 still holds (with shrinking but
strictly positive coefficient \(\tfrac{1}{\alpha}-\tfrac{1}{\beta}\)),
so the Gaussian exponential is bounded by~\(1\); the prefactor
\((\beta/\alpha)^{r_i/2}\) is uniformly bounded by a constant
\(C_{\delta_0}\). Hence
\(\Phi_i^*(t,x_i,\delta)\le C_{\delta_0}\) uniformly in
\((t,\delta)\in(t_i,t_i+\delta_0]\times\mathcal L_i\), and
\(\Phi_i^*(t,x_i,\delta)\to 1\) pointwise as \(t\downarrow t_i\).

The denominator of~\eqref{eq:drift-kernel-ratio-at-xi} converges to
\(1\) by dominated convergence (dominated by \(C_{\delta_0}\)). For the
numerator we use that \(\eta_i^\mu\) has finite first
\(A_i^+\)-Mahalanobis moment: by the finite-entropy
hypothesis~\eqref{eq:finite-entropy-assumption} and disintegration,
\(H(\eta_i^\mu\mid k_i^R)<\infty\) for \(\mu_i\)-a.e.\
\((\mathbf x_i,y_{i+1})\), and the Donsker--Varadhan variational
inequality applied to the Gaussian-Mahalanobis test function gives
\(\int_{\mathcal L_i}|\delta|_{A_i^+}\,\eta_i^\mu(d\delta)<\infty\).
Hence \(C_{\delta_0}\,|\delta|\) is \(\eta_i^\mu\)-integrable, and
dominated convergence yields
\[
\int_{\mathcal L_i}\delta\,\Phi_i^*(t,x_i,\delta)\,\eta_i^\mu(d\delta)
\longrightarrow
\int_{\mathcal L_i}\delta\,\eta_i^\mu(d\delta)
\quad\text{as }t\downarrow t_i.
\]
Combining with \(\alpha\to\beta=t_{i+1}-t_i\), the right-hand side
of~\eqref{eq:drift-kernel-ratio-at-xi} converges to
\[
\frac{1}{t_{i+1}-t_i}\int_{\mathcal L_i}\delta\,\eta_i^\mu(d\delta)
=:
b_i^*(t_i,x_i),
\]
the starting drift of
Proposition~\ref{prop:local-intervalwise-minimiser}. The limit is
independent of \(\varepsilon\) for all sufficiently small
\(\varepsilon\); this
establishes~\eqref{eq:b-eps-boundary-value-app}.
\end{proof}
\subsection{Compact-kernel estimator and local Gaussian domination}
\label{app:compact-kernel-estimators}

We now introduce the auxiliary regression scheme used in the proof. This
scheme is not yet the Gaussian one used in the implementation. It is a
compactly supported Nadaraya--Watson scheme with \(M\)-dependent kernel
profiles. The support may expand in the rescaled variable, but the
effective conditioning neighbourhood shrinks because \(h_MR_M\to0\). This
form is needed later to compare with truncated Gaussian weights.

Fix \(i\in\{0,\dots,N-1\}\) and a conditioning pair
\((\mathbf x_i,y_{i+1})\) with
\(\mathbf x_i=(x_0,\dots,x_i)\) and \(y_{i+1}\in\mathcal Y\) in the full
\(\mu_i\)-measure regular set where
Corollary~\ref{cor:terminal-conditional-law} and
Proposition~\ref{prop:local-intervalwise-minimiser} hold and where the
eigenstructure of \(A_i\) is locally continuous. Write
\[
\mathbf X_i:=(X_{t_0},\dots,X_{t_i}),
\qquad
q_i:=\dim(\mathbf X_i)=(i+1)d,
\qquad
q_Y:=\dim\mathcal Y,
\]
so the augmented past \((\mathbf X_i,Y)\) lives in
\(\R^{q_i+q_Y}\). Throughout this appendix the compact-kernel
estimator is written in the Euclidean chart of the macro-variable
\(\mathcal U_i^\ell\) (Remark~\ref{rem:reference-aware-kernel-rationale}),
in which the effective dimension is
\(q_{\mathrm{eff}}=q_i+q_Y\); the abstract small-ball formulation of
(S1) gives the same conclusions on replacing \(q_i+q_Y\) by
\(q_{\mathrm{eff}}\) in all bandwidth and denominator estimates.
Let
\[
\bigl(\mathbf X_i^{(m)},Y^{(m)},\Delta X_{i+1}^{(m)}\bigr),
\qquad
\Delta X_{i+1}^{(m)}:=X_{t_{i+1}}^{(m)}-X_{t_i}^{(m)},
\qquad m=1,\dots,M,
\]
be i.i.d.\ samples from \(\mu\) on the augmented state space, with
the realised next-interval increment as response and the latent
realisation \(Y^{(m)}\) observed.

Let \(K_M:\R^{q_i+q_Y}\to[0,\infty)\) be a sequence of continuous
regression profiles on the augmented past. We assume
\[
\int_{\R^{q_i+q_Y}}K_M(u)\,du=1,
\qquad
\supp K_M\subset B(0,R_M),
\qquad
R_M\uparrow\infty,
\]
together with a radial integrable envelope: there exists a bounded
measurable \(\overline K:[0,\infty)\to[0,\infty)\) such that
\[
0\le K_M(u)\le \overline K(|u|^2)
\quad\text{for all }u\in\R^{q_i+q_Y},\ M,
\qquad
\int_{\R^{q_i+q_Y}}\overline K(|u|^2)\,du<\infty.
\]
Let \(h_M\downarrow0\) be a bandwidth sequence with
\(h_MR_M\to 0\). The rescaled kernel and Nadaraya--Watson weights
localising on the augmented past at the runtime point
\((\mathbf x_i,y_{i+1})\) are
\begin{align*}
K_{M,h_M}(u)&:=h_M^{-(q_i+q_Y)}K_M(u/h_M),
\\
\omega_{i,m}^{M,K}(\mathbf x_i,y_{i+1})
&:=
\frac{K_{M,h_M}\!\bigl((\mathbf x_i,y_{i+1})-(\mathbf X_i^{(m)},Y^{(m)})\bigr)}
     {\sum_{n=1}^M K_{M,h_M}\!\bigl((\mathbf x_i,y_{i+1})-(\mathbf X_i^{(n)},Y^{(n)})\bigr)}.
\end{align*}

Throughout, a generic conditioning vector for the principal past is
denoted \(\boldsymbol\xi_i=(\xi_0,\dots,\xi_i)\in(\R^d)^{i+1}\) with
last component \(\xi_i\). The
\emph{compact-kernel empirical regularised backward potential} is
defined by
\begin{equation}
\label{eq:empirical-H-eps-MK}
\widehat H_i^{\varepsilon,M,K}(t,x)
:=
\sum_{m=1}^M
\omega_{i,m}^{M,K}(\mathbf x_i,y_{i+1})\,
\Phi_i^\varepsilon\!\bigl(t,x,\Delta X_{i+1}^{(m)}\bigr),
\end{equation}
where
\[
\Phi_i^\varepsilon(t,x,z)
:=
\frac{p_i^\varepsilon(t,x;t_{i+1},z)}
     {p_i^\varepsilon(t_i,x_i;t_{i+1},z)}
\]
is the coherently regularised kernel-ratio response of
\eqref{eq:def-H-eps}, with the same spectrally floored covariance
\(A_i^\varepsilon\) in numerator and denominator. The estimator
\eqref{eq:empirical-H-eps-MK} is the Nadaraya--Watson estimator of
the conditional expectation
\begin{equation}
\label{eq:H-eps-conditional-expectation}
H_i^\varepsilon(t,x)
=
\mathbb E_\mu\!\left[
\Phi_i^\varepsilon\!\bigl(t,x,X_{t_{i+1}};\mathbf x_i,y_{i+1}\bigr)
\,\middle|\,
\mathbf X_i=\mathbf x_i,Y=y_{i+1}
\right],
\end{equation}
which is~\eqref{eq:def-H-eps-cond} written under \(\mu\).

\paragraph{Statistical assumptions for the compact-kernel argument.}
The Nadaraya--Watson regression localises on the macro-conditioning
variable \(U_i^\ell\) of~\eqref{eq:macro-conditioning} at the
runtime query point \(u_0\), in the geometry of the pseudo-distance
\(d_\varepsilon^\ell\) of~\eqref{eq:macro-pseudo-distance};
\(B_\varepsilon(u_0,r):=\{v:d_\varepsilon^\ell(u_0,v)<r\}\). The
assumptions (S1)--(S6) of Section~\ref{subsec:statistical-assumptions-app}
apply. For ease of reference:
\begin{enumerate}[(S1)]
\item \emph{Local lower mass.} There exist \(r_0,c_{u_0},q_{\mathrm{eff}}>0\)
with \(\mu_U(B_\varepsilon(u_0,r))\ge c_{u_0}r^{q_{\mathrm{eff}}}\)
for \(0<r<r_0\);
\item \emph{Conditional continuity.}
\(v\mapsto\eta_i^{\mu,\ell}(\cdot\mid v)\) weakly continuous near
\(u_0\); \(v\mapsto A_i^\ell(v)\) operator-norm continuous at
\(u_0\);
\item \emph{i.i.d.\ sample.}
\((U_i^{\ell,(m)},\Delta X_{i+1}^{(m)})_{m=1}^M\) i.i.d.;
\item \emph{Bandwidth.} \(h_M\downarrow 0\),
\(Mh_M^{q_{\mathrm{eff}}}\to\infty\); for compact-kernel
truncation \(R_M\uparrow\infty\) with \(h_MR_M\to 0\);
\item \emph{Strong bandwidth.}
\(Mh_M^{q_{\mathrm{eff}}}/\log M\to\infty\) for a.s.\ statements;
\item \emph{Local conditional entropy admissibility.} For some
\(\rho>0\),
\[
\sup_{v\in B_\varepsilon(u_0,\rho)}
H\!\bigl(\eta_i^{\mu,\ell}(\cdot\mid v)\,\big|\,k_i^{R,\ell}(\cdot\mid v)\bigr)
<\infty.
\]
\end{enumerate}
The reference-comparison factor of
Lemma~\ref{lem:reference-aware-comparison} satisfies
\(\kappa_\varepsilon^\ell\le\exp(d_\varepsilon^\ell/2)\) since
\(d_\varepsilon^\ell\) already contains the reference log-spectral
distance~\eqref{eq:macro-reference-distance}.

\paragraph{Consequence of (S6): local first moment of the increment.}
By Donsker--Varadhan applied to the intrinsic Mahalanobis quadratic
test against the frozen reference,
\begin{equation}
\label{eq:local-intrinsic-first-moment}
\sup_{v\in C_\rho(u_0)}
\int |\delta|_{C_i^\ell(v)^+}\,
\eta_i^{\mu,\ell}(d\delta\mid v)<\infty;
\end{equation}
the reference-comparison
estimate~\eqref{eq:reference-comparison-factor} transports this
intrinsic control to the runtime/query geometry,
\begin{multline}
\label{eq:local-runtime-first-moment}
\sup_{v\in C_\rho(u_0)}
\int |\delta|_{C_{i,\varepsilon}^\ell(u_0)^{-1}}\,
\eta_i^{\mu,\ell}(d\delta\mid v)
\\
\le
\bigl(\sup_{v\in C_\rho(u_0)}\kappa_\varepsilon^\ell(u_0,v)\bigr)\,
\sup_{v\in C_\rho(u_0)}
\int |\delta|_{C_i^\ell(v)^+}\,\eta_i^{\mu,\ell}(d\delta\mid v)
<\infty.
\end{multline}
On the active leaf the Euclidean first-moment bound follows from
\(|\delta|\le\lambda_{\max}(A_i(v))^{1/2}|\delta|_{C_i^\ell(v)^+}\):
\begin{equation}
\label{eq:local-euclidean-first-moment}
\sup_{v\in C_\rho(u_0)}
\int |\delta|\,\eta_i^\mu(d\delta\mid v)<\infty.
\end{equation}
This bound is used only for the empirical \emph{initial-boundary}
drift estimator below; the variational construction of
Proposition~\ref{prop:local-intervalwise-minimiser} does not need
it.
At fixed \(\varepsilon\in(0,\lambda_i^{\min,+})\), the regularised
covariance \(A_i^\varepsilon\succeq\varepsilon I\) with
\(\det A_i^\varepsilon\ge\varepsilon^d\), so the finite-\(\varepsilon\)
Gaussian estimators of this appendix are well defined and uniformly
nondegenerate; the rank \(r_i\) and the active leaf \(\mathcal L_i\)
are fixed by the conditioning and enter only in the
\(\varepsilon\downarrow 0\) geometric step
(Proposition~\ref{prop:spectral-regularisation-leafwise-kernel}).

\begin{lemma}[Kernel denominator and localisation from the
small-ball lower mass]
\label{lem:kernel-denominator-localisation}
Under~(S1), \(D_h(u_0):=\int W_h(u_0,v)\,\mu_U(dv)\) satisfies
\(D_h(u_0)\gtrsim h^{q_{\mathrm{eff}}}\) as \(h\downarrow 0\), so
that \(M D_{h_M}(u_0)\to\infty\) under (S4) and a.s.\ under (S5).
Moreover the normalised kernel measures concentrate at \(u_0\): for
every \(\rho>0\),
\begin{equation}
\label{eq:s7-kernel-mass-localisation}
\frac{\int_{\{d_\varepsilon(u_0,v)>\rho\}}W_h(u_0,v)\,\mu_U(dv)}{D_h(u_0)}
\longrightarrow 0
\qquad\text{as }h\downarrow 0.
\end{equation}
\end{lemma}

\begin{proof}
For the compact case, the kernel is positive on a neighbourhood of
the origin in each factor, so for some \(c,c'>0\),
\(W_h(u_0,v)\ge c\,\mathbf 1_{B_\varepsilon(u_0,c'h)}(v)\); hence
\(D_h(u_0)\ge c\,\mu_U(B_\varepsilon(u_0,c'h))\ge cc_{u_0}(c'h)^{q_{\mathrm{eff}}}\).
For the Gaussian case,
\(W_h(u_0,v)\ge\exp(-c'^2/2)\,\mathbf 1_{B_\varepsilon(u_0,c'h)}(v)\),
giving the same lower bound up to constants. In both cases the
numerator in~\eqref{eq:s7-kernel-mass-localisation} is bounded by
\(\sup W_h\) times either the kernel-tail mass (compact: zero past
\(c''h\); Gaussian: \(e^{-\rho^2/(2h^2)}\)), so the ratio is at most
\(e^{-\rho^2/(2h^2)}/(c h^{q_{\mathrm{eff}}})\to 0\).
\end{proof}

\begin{lemma}[Reference-aware localisation and Mahalanobis comparison]
\label{lem:reference-aware-comparison}
For every increment
\(\delta\in\operatorname{Ran}C_i^\ell(v)\),
\begin{equation}
\label{eq:reference-comparison-factor}
|\delta|_{C_{i,\varepsilon}^\ell(u_0)^{-1}}
\le
\kappa_\varepsilon^\ell(u_0,v)\,|\delta|_{C_i^\ell(v)^+},
\qquad
\kappa_\varepsilon^\ell(u_0,v)
\le
\exp\!\Bigl(\tfrac12 d_{\mathrm{ref},\varepsilon}^\ell(u_0,v)\Bigr).
\end{equation}
In particular, on the kernel-localised set
\(\{d_{\mathrm{ref},\varepsilon}^\ell(u_0,v)\le r\}\) the
historical/intrinsic Mahalanobis norm \(|\delta|_{C_i^\ell(v)^+}\)
controls the runtime/query Mahalanobis norm
\(|\delta|_{C_{i,\varepsilon}^\ell(u_0)^{-1}}\) with constant
\(e^{r/2}\); compact reference-aware kernels give uniform
comparability on their support, and Gaussian reference-aware kernels
penalise the comparison-factor growth.
\end{lemma}

\begin{proof}
Write \(\delta=C_i^\ell(v)^{1/2}\zeta\) for some \(\zeta\in\R^d\).
Then
\[
|\delta|_{C_{i,\varepsilon}^\ell(u_0)^{-1}}
=\bigl|C_{i,\varepsilon}^\ell(u_0)^{-1/2}C_i^\ell(v)^{1/2}\zeta\bigr|
\le
\kappa_\varepsilon^\ell(u_0,v)\,|\zeta|
=\kappa_\varepsilon^\ell(u_0,v)\,|\delta|_{C_i^\ell(v)^+}.
\]
The bound \(\kappa_\varepsilon^\ell\le\exp(d_{\mathrm{ref},\varepsilon}^\ell/2)\)
is the spectral form of the quadratic-form
sandwich~\eqref{eq:macro-reference-distance}, using
\(C_i^\ell(v)\preceq C_{i,\varepsilon}^\ell(v)\).
\end{proof}

\medskip

The next lemma is the core technical estimate of the appendix: it
controls the integrand pointwise on compact internal cylinders, and
compact-uniform convergence follows by Ascoli--Arzel\`a.

\begin{lemma}[Gaussian-ratio domination]
\label{lem:local-gaussian-domination}
For every \(\varepsilon\in(0,\lambda_i^{\min,+})\), every compact
\(J\Subset(t_i,t_{i+1})\), \(B\Subset(x_i+\mathcal L_i)\), there
exist a compact neighbourhood \(C\Subset C_i\) of
\((\mathbf x_i,y_{i+1})\), a measurable centre map
\((t,x,\boldsymbol\xi_i)\mapsto m_{t,x,\boldsymbol\xi_i}\in\R^d\),
and constants \(c,C_{J,B}>0\) such that
\begin{multline}
\label{eq:local-gaussian-domination}
\Phi_i^\varepsilon(t,x,\delta;\boldsymbol\xi_i)
\le
C_{J,B}\,\exp\!\bigl(-c\,\|\delta-m_{t,x,\boldsymbol\xi_i}\|^2_{(A_i^\varepsilon)^{-1}}\bigr)
\\
\forall\,t\in J,\,x\in B,\,\boldsymbol\xi_i\in C,\,\delta\in\R^d.
\end{multline}
The same bound holds for every active directional derivative
\(D_{q_{i,j}}\Phi_i^\varepsilon\), \(j=1,\dots,r_i\).
\end{lemma}

\begin{proof}
Write \(\alpha:=t_{i+1}-t\), \(\beta:=t_{i+1}-t_i\), so \(0<\alpha\le\beta-\tau\) on \(J\) for some \(\tau>0\). Both densities in
\(\Phi_i^\varepsilon\) share the same covariance
\(A_i^\varepsilon(\boldsymbol\xi_i)\) and the same
\(\varepsilon\)-dependent normalisation
\((2\pi)^{-d/2}(\det A_i^\varepsilon)^{-1/2}\); the latter is
\(\delta\)-independent and cancels in the ratio up to the factor
\((\beta/\alpha)^{d/2}\). The remaining exponent in \(\delta\) is
\[
-\tfrac{1}{2\alpha}\|x_i+\delta-x\|^2_{(A_i^\varepsilon)^{-1}}
+\tfrac{1}{2\beta}\|\delta\|^2_{(A_i^\varepsilon)^{-1}},
\]
a strictly negative-definite quadratic form in \(\delta\) since
\(\alpha<\beta\). Completing the square yields a centre
\(m_{t,x,\boldsymbol\xi_i}\) and residual
\(R_{t,x,\boldsymbol\xi_i}^\varepsilon\) such that
\[
\Phi_i^\varepsilon(t,x,\delta;\boldsymbol\xi_i)
=
\bigl(\tfrac{\beta}{\alpha}\bigr)^{d/2}
\exp\!\bigl(
-\tfrac12\bigl(\tfrac1\alpha-\tfrac1\beta\bigr)
\|\delta-m_{t,x,\boldsymbol\xi_i}\|^2_{(A_i^\varepsilon)^{-1}}
+R_{t,x,\boldsymbol\xi_i}^\varepsilon
\bigr).
\]
On \(J\times B\) and on a sufficiently small compact \(C\Subset C_i\)
of \((\mathbf x_i,y_{i+1})\), continuity of
\((\boldsymbol\xi_i,t,x)\mapsto m_{t,x,\boldsymbol\xi_i}\) and of
\(A_i(\boldsymbol\xi_i)\) gives a uniform upper bound on the
residual, and the \(\alpha<\beta\) gap gives
\(\tfrac1\alpha-\tfrac1\beta\ge\tau/(\beta(\beta-\tau))>0\). The
quadratic decay rate \(c\) on the active eigenspace is bounded
below uniformly in \(\varepsilon\in(0,\lambda_i^{\min,+})\) since
\((A_i^\varepsilon)^{-1}\succeq(\|A_i\|+\varepsilon)^{-1}I\) there,
while the transverse coefficient \(1/(\varepsilon\alpha)\) only adds
decay; setting \(c:=\tfrac12(\tfrac1\alpha-\tfrac1\beta)\) and
absorbing \((\beta/\alpha)^{d/2}\), \(\exp(R)\) and continuity
prefactors into a finite \(C_{J,B}\)
proves~\eqref{eq:local-gaussian-domination}.

The directional derivative \(D_{q_{i,j}}\log\Phi_i^\varepsilon\)
along \(q_{i,j}\) is \(\alpha^{-1}\langle (A_i^\varepsilon)^{-1}(\delta-(x-x_i)),q_{i,j}\rangle\); chain rule inserts a polynomial-in-\((\delta-(x-x_i))\) prefactor
against the same Gaussian envelope, absorbed into the exponential
domination at the price of slightly smaller \(c\) and larger
\(C_{J,B}\) via \(|P(u)|\,e^{-c|u|^2}\le C'\,e^{-c'|u|^2}\) for
every polynomial \(P\) and \(0<c'<c\).
\end{proof}

\begin{lemma}[Asymptotic nondegeneracy of the regression denominator]
\label{lem:nw-denominator-lower-bound}
Assume \((\mathrm S1)\), \((\mathrm S3)\), \((\mathrm S4)\). Then
\[
\frac1M\sum_{m=1}^M K_{M,h_M}(\mathbf x_i-\mathbf X_i^{(m)})
\longrightarrow
\pi_i(\mathbf x_i)
\]
in probability; under \((\mathrm S5)\), almost surely. In particular,
since \(\pi_i(\mathbf x_i)>0\),
\(\sum_{m=1}^M K_{M,h_M}(\mathbf x_i-\mathbf X_i^{(m)})\ge
M(\pi_i(\mathbf x_i)-\eta)\) eventually for every
\(\eta\in(0,\pi_i(\mathbf x_i))\), in the same stochastic mode.
\end{lemma}

\begin{proof}
The bias term is
\[
\int_{\R^{q_i+q_Y}}K_M(u)\,\pi_i(\mathbf x_i-h_Mu)\,du
\longrightarrow\pi_i(\mathbf x_i),
\]
by continuity of \(\pi_i\), \(\int K_M=1\), and dominated convergence
against \(\overline K\) using \(h_MR_M\to0\). The variance is bounded
by \(\|\overline K\|_\infty h_M^{-(q_i+q_Y)}\), so
\[
\operatorname{Var}\!\Bigl(\tfrac1M\sum_m K_{M,h_M}(\mathbf x_i-\mathbf X_i^{(m)})\Bigr)
=O\!\bigl(1/(Mh_M^{q_i+q_Y})\bigr)\longrightarrow 0.
\]
The almost-sure version follows from
exponential inequalities for bounded empirical processes under
\((\mathrm S5)\) and Borel--Cantelli.
\end{proof}

\subsection{Pointwise and compact-uniform consistency}
\label{app:compact-uniform-consistency}

We now combine standard pointwise Nadaraya--Watson consistency with the
Gaussian domination of
Lemma~\ref{lem:local-gaussian-domination} to obtain uniform
convergence on every compact subset of the open intervalwise cylinder.

\begin{proposition}[Pointwise consistency]
\label{prop:pointwise-NW-H-eps}
Assume \((\mathrm S1)\)--\((\mathrm S4)\) and fix
\(\varepsilon\in(0,\lambda_i^{\min,+})\). For every fixed
\((t,x)\in(t_i,t_{i+1})\times\mathcal L_i\),
\[
\widehat H_i^{\varepsilon,M,K}(t,x)
\longrightarrow
H_i^\varepsilon(t,x)
\]
in probability as \(M\to\infty\). Under \((\mathrm S5)\), the
convergence is almost sure. The same conclusion holds for the
directional derivative \(D_{q_{i,j}}\widehat H_i^{\varepsilon,M,K}\),
\(j=1,\dots,d\), towards
\(D_{q_{i,j}}H_i^\varepsilon\), along every eigenvector
of \(A_i\).
\end{proposition}

\begin{proof}
Write the empirical potential as a Nadaraya--Watson estimator on the
augmented past of the conditional
expectation~\eqref{eq:H-eps-conditional-expectation}. For fixed
\((t,x)\) and \(\varepsilon\) the response
\[
(\boldsymbol\xi_i,y)\longmapsto
\Phi_{t,x}^\varepsilon(\boldsymbol\xi_i,y)
:=
\mathbb E_\mu\!\left[
\Phi_i^\varepsilon\!\bigl(t,x,X_{t_{i+1}}\bigr)
\,\middle|\,\mathbf X_i=\boldsymbol\xi_i,\,Y=y
\right]
\]
is continuous at \((\boldsymbol\xi_i,y)=(\mathbf x_i,y_{i+1})\), since
\(\eta_i^\mu(\cdot\mid\boldsymbol\xi_i,y)\) depends weakly continuously on
\((\boldsymbol\xi_i,y)\) by \((\mathrm S2)\) and the Gaussian-ratio
response \(\Phi_i^\varepsilon\) is dominated by
Lemma~\ref{lem:local-gaussian-domination} via the
\(\alpha<\beta\) gap of the coherent kernel ratio. The denominator
is asymptotically nondegenerate by
Lemma~\ref{lem:nw-denominator-lower-bound}. Standard Nadaraya--Watson
consistency for compactly supported regression kernels under
\((\mathrm S3)\)--\((\mathrm S4)\) (cf.\ \cite{Cui2013StrongCO}) yields
\(\widehat H_i^{\varepsilon,M,K}(t,x)\to
H_i^\varepsilon(t,x)\) in probability, and almost surely
under \((\mathrm S5)\). Differentiating \(\Phi_i^\varepsilon\) in
\(x\) along any eigenvector \(q_{i,j}\), \(j=1,\dots,d\),
inserts a polynomial-in-\((z-x)\) factor against the same dominated
Gaussian-ratio integrand of
Lemma~\ref{lem:local-gaussian-domination}. The
same Nadaraya--Watson argument applies to give pointwise convergence
of \(D_{q_{i,j}}\widehat H_i^{\varepsilon,M,K}\) for
every \(j=1,\dots,d\).
\end{proof}

\begin{proposition}[Compact-uniform consistency]
\label{prop:compact-uniform-consistency-H-eps}
Assume \((\mathrm S1)\)--\((\mathrm S4)\) and fix
\(\varepsilon\in(0,\lambda_i^{\min,+})\). For every compact
\[
K=J\times B\Subset (t_i,t_{i+1})\times\mathcal L_i,
\]
\[
\sup_{(t,x)\in K}
\bigl|\widehat H_i^{\varepsilon,M,K}(t,x)
-H_i^\varepsilon(t,x)\bigr|
\longrightarrow 0
\qquad\text{in probability.}
\]
Under \((\mathrm S5)\), the convergence is almost sure. The same
uniform convergence holds for every directional derivative
\(D_{q_{i,j}}\), \(j=1,\dots,d\), along the
eigenvectors of \(A_i\), both active and kernel
directions included.
\end{proposition}

\begin{proof}
By Lemma~\ref{lem:local-gaussian-domination}, the integrand defining
\(\widehat H_i^{\varepsilon,M,K}(t,x)\) is, on the event
where every sample with nonzero NW-weight satisfies
\(\mathbf X_i^{(m)}\in C\), bounded by
\(C_{J,B}\,e^{-c|\Delta X_{i+1}^{(m)}|^2}\) uniformly for
\((t,x)\in K\). Combined with the denominator lower bound of
Lemma~\ref{lem:nw-denominator-lower-bound}, this yields a uniform
bound
\[
\sup_{(t,x)\in K}\bigl|\widehat H_i^{\varepsilon,M,K}(t,x)\bigr|
\le
\tilde C_{J,B}
\]
in the relevant stochastic mode, for all sufficiently large \(M\). The
same bound holds for the exact \(H_i^\varepsilon\).

For equicontinuity in \((t,x)\in K\), the explicit Gaussian form of
\(p_i^\varepsilon\) shows that, for any \((t,x),(t',x')\in J\times B\),
\begin{multline*}
\bigl|p_i^\varepsilon(t,x;t_{i+1},x_i+\delta)
-p_i^\varepsilon(t',x';t_{i+1},x_i+\delta)\bigr| \\
\le
L_{J,B}\,(|t-t'|+|x-x'|)\,
\mathrm{e}^{-c'|\delta|^2},
\end{multline*}
for constants \(L_{J,B},c'>0\) uniform on \(J\times B\) and
\(\delta\in\R^d\) with \(\xi_i+\delta\in\mathcal L_i^{\boldsymbol\xi_i,y}\),
for \((\boldsymbol\xi_i,y)\in C\); on this neighbourhood the regularised
covariance \(A_i^\varepsilon(\boldsymbol\xi_i,y)\) is uniformly elliptic
by \((\mathrm S2)\). The same
domination~\eqref{eq:local-gaussian-domination} thus controls the
difference, and summing against the NW weights gives
\[
\bigl|\widehat H_i^{\varepsilon,M,K}(t,x)
-\widehat H_i^{\varepsilon,M,K}(t',x')\bigr|
\le
\widetilde L_{J,B}\,(|t-t'|+|x-x'|)
\]
in the same stochastic mode, for all large \(M\). The exact
\(H_i^\varepsilon\) satisfies the same Lipschitz estimate on \(K\).

The family
\(\{(t,x)\mapsto\widehat H_i^{\varepsilon,M,K}(t,x)\}_M\)
is therefore eventually equibounded and equi-Lipschitz on the compact
metric space \(K\). Ascoli--Arzel\`a applied to subsequences identifies
the only possible limit, by Proposition~\ref{prop:pointwise-NW-H-eps},
with \(H_i^\varepsilon\), and the whole sequence converges uniformly on
\(K\) in the stated stochastic mode.

For the active derivatives \(D_{q_{i,j}}\), the same
argument applies: the differentiated kernel
\(D_{q_{i,j}}p_i^\varepsilon(t,x;t_{i+1},\nu)\)
acquires a polynomial-in-\((\nu-x)\) factor against the Gaussian, which
is dominated by~\eqref{eq:local-gaussian-domination}, and a Lipschitz
estimate on \((t,x)\in K\) follows from joint smoothness of the
Gaussian density. Compact-uniform convergence on \(K\) then follows by
the same Ascoli--Arzel\`a argument applied to the derivative
estimator.
\end{proof}

\subsection{Gaussian-kernel transfer and drift consistency}
\label{app:gaussian-kernel-estimators}

The Gaussian weights of~\eqref{eq:reference-aware-kernel} have
unbounded support; we transfer the compact-kernel consistency by a
truncation argument carried out directly in the reference-aware
metric \(d_\varepsilon\), without any change of chart.

\paragraph{Truncation in the reference-aware metric.}
The Gaussian-kernel case is obtained from the compactly supported
case by truncation, without changing the reference-aware metric. For
\(R>0\), define
\[
K_{h,R}(u,v):=\frac{K_h(u,v)\,\mathbf 1_{\{d_\varepsilon(u,v)\le Rh\}}}{Z_R(u,h)},
\qquad
Z_R(u,h):=\!\!\int_{\{d_\varepsilon(u,v)\le Rh\}}\!\!K_h(u,v)\,\mu_U(dv).
\]
For fixed \(R\), this is a compactly supported kernel in the metric
\(d_\varepsilon\), hence the compact-kernel consistency result
applies. It remains to show that the estimators built with \(K_h\)
and \(K_{h,R}\) differ by a term going to zero as \(R\to\infty\),
uniformly on compact subsets of the interior time-space domain. This
follows from the Gaussian tail of \(K_h\), the positivity of the
regression denominator, and the integrable domination of the
response \(\Phi_i^\varepsilon\) and of its active derivatives.
Concretely,
\begin{equation}
\label{eq:gaussian-truncation-bound}
\sup_{(t,x)\in K}
\bigl|\widehat H_i^{\varepsilon,M,\mathrm G}(t,x)
-\widehat H_i^{\varepsilon,M,R}(t,x)\bigr|
\le
C_K\,
\frac{\displaystyle\sum_{m=1}^M K_h^{\mathrm G}(u,U_m)\,
\mathbf 1_{\{d_\varepsilon(u,U_m)>Rh\}}\,
\bigl(1+\Psi(\Delta X_{i+1}^{(m)})\bigr)}
{\displaystyle\sum_{m=1}^M K_h^{\mathrm G}(u,U_m)}
+o_R(1),
\end{equation}
where \(\Psi\) is integrable under the conditional law and the bound
uses the same Gaussian domination already established for
\(\Phi_i^\varepsilon\) and its active derivatives. We keep the
notation \(\mathbf x_i,\mathbf X_i^{(m)},y_{i+1},Y^{(m)}\) for the
conditioning variables; in the truncation envelope below,
\(K_h^{\mathrm G}=\Gamma_h\) and \(K_{h,R}=\Gamma_{h,R}\).

Set \(q:=q_{\mathrm{eff}}\), the effective dimension of the
finite-dimensional regressors of the reference-aware metric chart, and
\[
\Gamma(u):=(2\pi)^{-q/2}e^{-|u|^2/2},
\quad
\Gamma^{(R)}(u):=\frac{\Gamma(u)\mathbf 1_{\{|u|\le R\}}}{Z_R},
\quad
Z_R:=\!\!\int_{|v|\le R}\!\!\Gamma(v)\,dv.
\]
Here \(|u|\) denotes the reference-aware metric displacement
\(d_\varepsilon\), so that the truncation \(\mathbf 1_{\{|u|\le R\}}\)
is exactly the metric ball used in~\eqref{eq:gaussian-truncation-bound}.
The scaled truncated and genuine Gaussian regression kernels are
\[
\Gamma_{h,R}(u):=h^{-q}\Gamma^{(R)}(u/h),
\qquad
\Gamma_h(u):=h^{-q}\Gamma(u/h).
\]
The truncation \(\Gamma^{(R_M)}\) satisfies the compact-kernel
assumptions \((\mathrm S3)\) with \(R_M\) the truncation radius and
\(\overline K\) the radial Gaussian envelope.

Let \(\omega_{i,m}^{M,R_M}(\mathbf x_i,y_{i+1})\) and
\(\omega_{i,m}^{M,\mathrm G}(\mathbf x_i,y_{i+1})\) be the Nadaraya--Watson
weights built from \(\Gamma_{h_M,R_M}\) and \(\Gamma_{h_M}\),
respectively, and let
\(\widehat H_i^{\varepsilon,M,R_M}\) and
\(\widehat H_i^{\varepsilon,M,\mathrm G}\) be the corresponding
empirical potentials.

\begin{lemma}[Nondegeneracy of the Gaussian regression denominator]
\label{lemma:nondeg-gaussian-regr}
Assume \((\mathrm S1)\),
\(h_M\downarrow0\), \(Mh_M^{q_{\mathrm{eff}}}\to\infty\). Then
\(M^{-1}\sum_{m=1}^M\Gamma_{h_M}(\mathbf x_i-\mathbf X_i^{(m)})
\to\pi_i(\mathbf x_i,y_{i+1})\) in probability, and almost surely under
\((\mathrm S5)\).
\end{lemma}

\begin{proof}
For the lower bound, fix \(\eta\in(0,\pi_i(\mathbf x_i,y_{i+1}))\) and choose
\(R\) with \(Z_R\pi_i(\mathbf x_i,y_{i+1})\ge\pi_i(\mathbf x_i,y_{i+1})-\eta/2\). The
inequality \(\Gamma(u)\ge Z_R\Gamma^{(R)}(u)\) gives
\(D_M^G(\mathbf x_i)\ge Z_RD_M^{(R)}(\mathbf x_i)\), and
Lemma~\ref{lem:nw-denominator-lower-bound} applied to the fixed
compactly supported kernel \(\Gamma^{(R)}\) yields the lower bound.
The matching upper bound follows from the standard Parzen--Rosenblatt
estimate for the Gaussian kernel under the bandwidth conditions.
\end{proof}

Under the coherent ratio formulation of~\eqref{eq:def-H-eps}, the
empirical estimator integrates \(\Phi_i^\varepsilon\) directly: the
same floored covariance \(A_i^\varepsilon(\mathbf x_i,y_{i+1})\)
appears in numerator and denominator and the volume factor
\((\det A_i^\varepsilon)^{-1/2}\) cancels, so no sample-specific
volume term has to be controlled. What still has to be finite, with
high probability over the sample cloud, is the \emph{runtime}
Mahalanobis evaluated on the sample increment,
\(|\delta_m|_{(A_i^\varepsilon(\mathbf x_i,y_{i+1}))^{-1}}\). To
control it, define the \emph{sample-conditioned intrinsic increment
norm}
\begin{equation}
\label{eq:intrinsic-endpoint-norm}
r_i(\boldsymbol\xi_i,\delta)
:=
\bigl|C_i(\boldsymbol\xi_i)^{+/2}\delta\bigr|,
\qquad
C_i(\boldsymbol\xi_i):=(t_{i+1}-t_i)A_i(\boldsymbol\xi_i),
\end{equation}
where \(C_i(\boldsymbol\xi_i)^{+/2}\) is the principal symmetric
square root of the Moore--Penrose pseudo-inverse of
\(C_i(\boldsymbol\xi_i)\); \(r_i(\boldsymbol\xi_i,\delta)\) is the
intrinsic norm of \(\delta\) under the sample's own frozen
reference. The finite-entropy
hypothesis~\eqref{eq:finite-entropy-assumption} controls
\(r_i(\boldsymbol\xi_i,\delta_m)\) with high probability via entropy
disintegration relative to the
sample-conditioned reference density
\(q_i(\delta\mid\boldsymbol\xi_i):=p_i(t_i,\xi_i;t_{i+1},\xi_i+\delta\mid\boldsymbol\xi_i)\)
(the unfloored leafwise density of \(\Delta X_{i+1}\) under \(R\)
given \(\mathbf X_i=\boldsymbol\xi_i\)), and the runtime Mahalanobis
is equivalent to \(r_i(\boldsymbol\xi_i,\cdot)\) on the localised
sample cloud by continuity of
\((\boldsymbol\xi_i,y)\mapsto A_i(\boldsymbol\xi_i,y)\) at
\((\mathbf x_i,y_{i+1})\) and equivalence of norms on the
finite-dimensional matrix space \(\mathbb S^d\).

\begin{proposition}[High-probability intrinsic localisation for Gaussian weights]
\label{prop:hp-intrinsic-localisation}
Assume \((\mathrm S1)\)--\((\mathrm S4)\) and the finite-entropy
hypothesis~\eqref{eq:finite-entropy-assumption}. Fix
\(\varepsilon\in(0,\lambda_i^{\min,+})\) and a compact
\(K\Subset(t_i,t_{i+1})\times(x_i+\mathcal L_i)\). For every
\(\rho>0\) there exist \(R<\infty\) and \(M_0\) such that, with
\begin{multline*}
\mathcal I_M(R)
:=
\bigl\{m\le M:\,
|(\mathbf X_i^{(m)},Y^{(m)})-(\mathbf x_i,y_{i+1})|\le h_MR_M,\\
r_i(\mathbf X_i^{(m)},\Delta X_{i+1}^{(m)})\le R\bigr\},
\end{multline*}
\begin{equation}
\label{eq:hp-localisation-bound}
\sup_{(t,x)\in K}
\Bigl|\sum_{m\notin\mathcal I_M(R)}
\omega_{i,m}^{M,\mathrm G}(\mathbf x_i,y_{i+1})\,
\Phi_i^\varepsilon(t,x,\Delta X_{i+1}^{(m)})\Bigr|
\le \rho
\end{equation}
holds with probability \(\ge 1-\rho\) for all \(M\ge M_0\). The same
holds for \(D_{q_{i,j}}\Phi_i^\varepsilon\), \(j=1,\dots,r_i\).
\end{proposition}

\begin{proof}
Decompose
\(\mathcal I_M(R)^c=\mathcal B_M(R)\cup\mathcal A_M(R)\) into the
endpoint-cost tail \(\mathcal B_M(R):=\{m:r_i>R\}\) and the
conditioning-far set
\(\mathcal A_M(R):=\{m:|(\mathbf X_i^{(m)},Y^{(m)})-(\mathbf x_i,y_{i+1})|>h_MR_M,\,r_i\le R\}\),
and bound each.

\emph{Endpoint-cost tail.} Entropy disintegration along the past
gives
\begin{equation}
\label{eq:entropy-chain-rule}
H(\mu\mid\mu_T^R)
=
H(\mu_i\mid\mu_i^R)
+\int H\!\left(
\eta_i^\mu(\cdot\mid\boldsymbol\xi_i)\,\big|\,k_i^R(\cdot\mid\boldsymbol\xi_i)
\right)\mu_i(d\boldsymbol\xi_i).
\end{equation}
By~\eqref{eq:finite-entropy-assumption} and Markov,
\(E_\rho:=\{\boldsymbol\xi_i:H(\eta_i^\mu\mid k_i^R)\le L_\rho\}\)
satisfies \(\mu_i(E_\rho)\ge 1-\rho\) for some \(L_\rho<\infty\). On
\(E_\rho\), under the intrinsic parametrisation
\(\delta\mapsto C_i(\boldsymbol\xi_i)^{+/2}\delta\) the reference
\(k_i^R\) is standard \(r_i\)-Gaussian with tail
\(k_i^R(r_i>R\mid\boldsymbol\xi_i)\le C_{r_i}e^{-R^2/(2r_i)}\);
Donsker--Varadhan transfers the tail to \(\eta_i^\mu\):
\begin{equation}
\label{eq:entropy-tail-transfer}
\sup_{\boldsymbol\xi_i\in E_\rho}
\eta_i^\mu(r_i>R\mid\boldsymbol\xi_i)
\le
\frac{L_\rho+\log 2}{R^2/(2r_i)-\log C_{r_i}}
\xrightarrow[R\to\infty]{}0.
\end{equation}
For \(R=R(\rho,L_\rho)\) large enough this is \(\le\rho^2\); Markov
on \(\#\mathcal B_M(R)/M\) bounds its empirical mass by \(\rho\)
w.h.p.

\emph{Conditioning-far set.} On \(\mathcal A_M(R)\),
Lemma~\ref{lem:local-gaussian-domination} bounds \(\Phi_i^\varepsilon\)
uniformly (constants \(c,C_{J,B}\) of the lemma absorb the runtime
\(\varepsilon\)-dependence); by
Lemma~\ref{lemma:nondeg-gaussian-regr} and the Gaussian tail
\(\Gamma_{h_M}(u)\le e^{-R_M^2/2}\|\Gamma_{h_M}\|_\infty\) for
\(|u|>h_MR_M\),
\(\omega_{i,m}^{M,\mathrm G}\le 2e^{-R_M^2/2}\|\Gamma_{h_M}\|_\infty/(M\pi_i(\mathbf x_i,y_{i+1}))\)
for large \(M\); the factor \(e^{-R_M^2/2}\to 0\) kills the
contribution.

\emph{Runtime vs.\ sample Mahalanobis.} The bound from
Lemma~\ref{lem:local-gaussian-domination} on \(\Phi_i^\varepsilon\)
is in the runtime norm
\(\|\cdot\|_{(A_i^\varepsilon(\mathbf x_i,y_{i+1}))^{-1}}\), while
\(r_i(\boldsymbol\xi_i^{(m)},\delta_m)\) controls the sample norm.
On the localised cloud, operator-norm continuity of
\((\boldsymbol\xi_i,y)\mapsto A_i^\varepsilon(\boldsymbol\xi_i,y)\)
by~\((\mathrm S2)\) and equivalence of norms on
\(\mathbb S^d\) give
\(\|A_i^\varepsilon(\boldsymbol\xi_i,y)-A_i^\varepsilon(\mathbf x_i,y_{i+1})\|\to 0\),
so the two intrinsic norms are equivalent up to a multiplicative
factor approaching \(1\); the sample bound \(r_i\le R\) transfers to
a runtime bound \(\|\delta_m\|_{(A_i^\varepsilon(\mathbf x_i,y_{i+1}))^{-1}}
\le R(1+o(1))\) on the event of the two preceding steps.

Combining yields~\eqref{eq:hp-localisation-bound}. The
active-derivative case is identical: differentiating
\(\Phi_i^\varepsilon\) along \(q_{i,j}\) inserts a
polynomial-in-\((\delta-(x-x_i))\) factor absorbed into the same
Gaussian envelope of Lemma~\ref{lem:local-gaussian-domination}.
\end{proof}

\begin{remark}[Sample-reference compatibility]
\label{rem:sample-reference-compatibility}
Bound~\eqref{eq:hp-localisation-bound} is the statistical analogue
of the intrinsic endpoint tightness of
Section~\ref{subsec:volatility-informed-freezing}: the endpoint
displacement of the sample is measured in the geometry of its own
covariance cumulant, and finite
entropy~\eqref{eq:entropy-chain-rule} transfers the Gaussian
endpoint tails of the frozen reference to the target conditional
law. Under the coherent ratio formulation, the volume factor
cancels in numerator and denominator and only the intrinsic
\(r_i\)-tail and the sample-to-runtime Mahalanobis equivalence on the
localised cloud are needed.
\end{remark}

\begin{proposition}[Compact-uniform consistency of the Gaussian estimator]
\label{prop:gaussian-compact-uniform}
Under \((\mathrm S1)\)--\((\mathrm S4)\), the
finite-entropy hypothesis~\eqref{eq:finite-entropy-assumption},
\(\varepsilon\in(0,\lambda_i^{\min,+})\) fixed, and
\(R_M\uparrow\infty\) with \(h_MR_M\to 0\): for every compact
\(K\Subset(t_i,t_{i+1})\times(x_i+\mathcal L_i)\),
\[
\sup_{(t,x)\in K}
\bigl|\widehat H_i^{\varepsilon,M,\mathrm G}(t,x)-H_i^\varepsilon(t,x)\bigr|
\longrightarrow 0
\quad\text{in probability,}
\]
and the same uniform convergence holds for every active directional
derivative \(D_{q_{i,j}}\), \(j=1,\dots,r_i\). The truncated
auxiliary estimator \(\widehat H_i^{\varepsilon,M,R_M}\) converges
almost surely on its own under \((\mathrm S5)\).
\end{proposition}

\begin{proof}
The truncated profile \(K_M:=\Gamma^{(R_M)}\) satisfies the
compact-kernel assumptions of
Section~\ref{app:compact-kernel-estimators}; by
Proposition~\ref{prop:compact-uniform-consistency-H-eps},
\(\widehat H_i^{\varepsilon,M,R_M}\to H_i^\varepsilon\) uniformly on
\(K\) in probability under \((\mathrm S4)\), a.s.\ under
\((\mathrm S5)\).

For the passage to the full Gaussian estimator, with
\(S_M:=\sum_m\Gamma_{h_M}(\mathbf x_i-\mathbf X_i^{(m)})\) and
\(T_M:=\sum_m\Gamma_{h_M}(\mathbf x_i-\mathbf X_i^{(m)})\mathbf 1_{\{|\mathbf X_i^{(m)}-\mathbf x_i|>h_MR_M\}}\)
(suppressing the \(Y\)-component, identical by the same truncation argument),
\[
\sum_{m=1}^M
\bigl|\omega_{i,m}^{M,R_M}-\omega_{i,m}^{M,\mathrm G}\bigr|
=
\frac{2T_M}{S_M}.
\]
By Lemma~\ref{lemma:nondeg-gaussian-regr} \(S_M/M\to\pi_i>0\); the
Gaussian tail with \(R_M\to\infty\) gives \(T_M/M\to 0\). On the
event of Proposition~\ref{prop:hp-intrinsic-localisation} the
integrand is uniformly bounded by
Lemma~\ref{lem:local-gaussian-domination}, whence
\[
\sup_{(t,x)\in K}
\bigl|\widehat H_i^{\varepsilon,M,R_M}-\widehat H_i^{\varepsilon,M,\mathrm G}\bigr|
\le
C_{J,B}\,\frac{2T_M}{S_M}\to 0.
\]
The active-derivative case is identical: the differentiated
integrand is Gaussian-dominated by
Lemma~\ref{lem:local-gaussian-domination} via the polynomial-prefactor
absorption.
\end{proof}

\paragraph{Compact-uniform drift convergence.}

\begin{proposition}[Compact-uniform drift convergence]
\label{prop:drift-compact-uniform}
Under \((\mathrm S1)\)--\((\mathrm S4)\), the
finite-entropy hypothesis~\eqref{eq:finite-entropy-assumption}, and
\(\varepsilon\in(0,\lambda_i^{\min,+})\) fixed, for every compact
\(K\Subset(t_i,t_{i+1})\times(x_i+\mathcal L_i)\),
\[
\sup_{(t,x)\in K}
\bigl|A_i^\varepsilon\,\nabla_x\log\widehat H_i^{\varepsilon,M,\mathrm G}(t,x)
-A_i^\varepsilon\,\nabla_x\log H_i^\varepsilon(t,x)\bigr|
\longrightarrow 0
\quad\text{in probability,}
\]
and almost surely under \((\mathrm S5)\) for the truncated auxiliary
estimator.
\end{proposition}

\begin{proof}
Write \(\widehat H:=\widehat H_i^{\varepsilon,M,\mathrm G}\) and
\(H:=H_i^\varepsilon\). Strict positivity of \(H\) on \(K\) gives
\(\inf_K H\ge\delta_K>0\); by
Proposition~\ref{prop:gaussian-compact-uniform},
\(\widehat H\to H\) and \(D_{q_{i,j}}\widehat H\to D_{q_{i,j}}H\)
uniformly on \(K\), so \(\inf_K\widehat H\ge\delta_K/2\) for all
large \(M\) in the stated stochastic mode. The identity
\begin{multline*}
D_{q_{i,j}}\log\widehat H-D_{q_{i,j}}\log H
=\frac{D_{q_{i,j}}\widehat H-D_{q_{i,j}}H}{\widehat H}
+D_{q_{i,j}}H\!\left(\frac1{\widehat H}-\frac1{H}\right),
\\
\Bigl|\tfrac{1}{\widehat H}-\tfrac{1}{H}\Bigr|
\le\tfrac{2}{\delta_K^2}|\widehat H-H|,
\end{multline*}
combined with the uniform compact-uniform convergence of both
\(\widehat H\) and \(D_{q_{i,j}}\widehat H\) and with smoothness of
\(H\) on the open cylinder, gives
\(\sup_K|D_{q_{i,j}}\log\widehat H-D_{q_{i,j}}\log H|\to 0\).
Summing the spectral-floor expansion
\(A_i^\varepsilon\nabla_x\log H=\sum_{j}\max(\lambda_{i,j},\varepsilon)\,
D_{q_{i,j}}\log H\,q_{i,j}\) with deterministically bounded
coefficients \(\max(\lambda_{i,j},\varepsilon)\le\|A_i\|+\varepsilon\)
yields the claim.
\end{proof}

By Proposition~\ref{prop:spectral-regularisation-leafwise-kernel},
the regularised drift converges to the exact leafwise drift on every
compact \(K\Subset(t_i,t_{i+1})\times(x_i+\mathcal L_i)\) as
\(\varepsilon\downarrow 0\):
\(A_i^\varepsilon\nabla_x\log H_i^\varepsilon\to b_i^*\). A diagonal
extraction in \((\varepsilon_n,M_n,h_{M_n})\) combines this with
Proposition~\ref{prop:drift-compact-uniform} to give the
compact-uniform drift convergence of part~(i) of
Theorem~\ref{thm:approximate-intervalwise-dynamics}.

\subsection{Initial-boundary drift and diagonal consistency}
\label{app:boundary-drift}

The interior arguments above give convergence on compacta
\(K\Subset(t_i,t_{i+1})\times(x_i+\mathcal L_i)\) of the open
cylinder, where the strict time gap \(\alpha<\beta\) provides the
Gaussian-ratio domination of
Lemma~\ref{lem:local-gaussian-domination}. At the initial diagonal
point \((t_i,x_i)\), \(\alpha\uparrow\beta\), the exponential
penalisation disappears, and the response becomes the unbounded
increment \(\Delta X_{i+1}\); we collect here the dominated- and
Donsker--Varadhan-based arguments needed at the boundary. Set
\(\alpha:=t_{i+1}-t\), \(\beta:=t_{i+1}-t_i\). The leafwise kernel
ratio at \(x=x_i\) reduces to
\(\Phi_i^*(t,x_i,\delta)
=(\beta/\alpha)^{r_i/2}\exp\!\bigl(-\tfrac12(\tfrac1\alpha-\tfrac1\beta)\|\delta\|^2_{A_i^+}\bigr)\)
for \(\delta\in\mathcal L_i\).

\begin{lemma}[Analytic boundary drift]
\label{lem:analytic-boundary-drift}
For every conditioning point \(u_0\) at which the conditional
entropy \(H(\eta_i^\mu(\cdot\mid u_0)\,|\,k_i^R(\cdot\mid u_0))\) is
finite --- the full \(\mu_i\)-measure set provided by
disintegration~\eqref{eq:entropy-chain-rule} of the global
finite-entropy assumption~\eqref{eq:finite-entropy-assumption} ---
and for every \(\varepsilon\in(0,\lambda_i^{\min,+})\),
\[
\lim_{t\downarrow t_i}b_i^\varepsilon(t,x_i;u_0)
=
\frac{1}{t_{i+1}-t_i}\,
\mathbb E_\mu[\Delta X_{i+1}\mid \mathbf U_i=u_0]
=:b_i^*(t_i,x_i;u_0),
\]
and the limit is independent of \(\varepsilon\). No local
admissibility assumption (S6) is needed for this analytic step.
\end{lemma}

\begin{proof}
Differentiating the coherent ratio at \(x=x_i\) gives the
kernel-ratio average
\[
b_i^\varepsilon(t,x_i;u_0)
=
\frac{1}{\alpha}\,
\frac{\int\delta\,\Phi_i^*(t,x_i,\delta)\,\eta_i^\mu(d\delta\mid u_0)}
{\int\Phi_i^*(t,x_i,\delta)\,\eta_i^\mu(d\delta\mid u_0)}.
\]
For \(\alpha\in[\beta/2,\beta]\) the gap
\(\tfrac1\alpha-\tfrac1\beta\ge 0\), the exponential is bounded by
\(1\) and the prefactor by \(2^{r_i/2}\); hence
\(\Phi_i^*(t,x_i,\delta)\le 2^{r_i/2}\) uniformly in \((t,\delta)\)
and \(\Phi_i^*(t,x_i,\delta)\to 1\) pointwise as
\(\alpha\uparrow\beta\) (the convergence is \emph{not} monotone:
only pointwise plus dominated convergence is used). The
denominator tends to \(1\).

The numerator requires an \(\eta_i^\mu(\cdot\mid u_0)\)-integrable
dominant for \(|\delta|\). Since \(\eta_i^\mu(\cdot\mid u_0)\) is
supported on the active leaf \(\mathcal L_i(u_0)\) of \(A_i(u_0)\)
(Corollary~\ref{cor:terminal-conditional-law}), it suffices to
control \(|\delta|\) on that leaf. By the Donsker--Varadhan
variational inequality applied to the intrinsic Mahalanobis
quadratic test against the reference \(k_i^R(\cdot\mid u_0)\),
finite conditional entropy at \(u_0\) implies the
\emph{intrinsic \(A_i^+\)-Mahalanobis first moment on the active
leaf}
\[
\int_{\mathcal L_i(u_0)}|\delta|_{C_i(u_0)^+}\,
\eta_i^\mu(d\delta\mid u_0)<\infty,
\qquad C_i(u_0):=\Delta_i A_i(u_0).
\]
The Euclidean implication uses the leaf inequality
\[
|\delta|\le\lambda_{\max}(A_i(u_0))^{1/2}\,|\delta|_{C_i(u_0)^+}
\qquad\forall\,\delta\in\mathcal L_i(u_0),
\]
which gives
\(\int|\delta|\,\eta_i^\mu(d\delta\mid u_0)<\infty\). Hence
\(2^{r_i/2}|\delta|\) is an integrable dominant; dominated
convergence and \(\alpha\to\beta\) yield
\(b_i^\varepsilon(t,x_i;u_0)\to\beta^{-1}\int\delta\,\eta_i^\mu(d\delta\mid u_0)\).
The argument is \(\varepsilon\)-independent: on the leaf
\(\Phi_i^*\) is the unfloored leafwise ratio.
\end{proof}

\begin{lemma}[Compact reference-aware boundary drift]
\label{lem:compact-boundary-drift}
Under (S1)--(S6), with the compact choice
\(W_h(u_0,v)=K(d_\varepsilon^\ell(u_0,v)/h)\) and
\(\operatorname{supp}K\subset[0,1]\), the empirical regularised
boundary drift admits the finite-sample limit
\begin{multline*}
\widehat b_i^{M,K}(t_i,x_i)
:=\lim_{t\downarrow t_i}\widehat b_i^{\varepsilon,M,K}(t,x_i)
\\
=\frac{1}{\beta}\,
\frac{\sum_{m=1}^M W_{h_M}(u_0,\mathbf U_i^{(m)})\,\Delta X_{i+1}^{(m)}}
{\sum_{m=1}^M W_{h_M}(u_0,\mathbf U_i^{(m)})},
\qquad\beta=t_{i+1}-t_i,
\end{multline*}
independent of \(\varepsilon\), and
\(\widehat b_i^{M,K}(t_i,x_i)\to b_i^*(t_i,x_i;u_0)\) in
probability; a.s.\ under (S5).
\end{lemma}

\begin{proof}
\emph{Diagonal limit.} As \(t\downarrow t_i\),
\(\Phi_i^\varepsilon(t,x_i,\delta)\to 1\) uniformly in the sample
increments (dominated by \(2^{r_i/2}\); cf.\
Lemma~\ref{lem:analytic-boundary-drift}); the
\(\Phi_i^\varepsilon\) factors cancel between numerator and
denominator, leaving the NW estimator of \(\Delta X_{i+1}\) at
\(u_0\) with kernel \(W_{h_M}\).
\emph{Denominator non-collapse.}
Lemma~\ref{lem:kernel-denominator-localisation} gives
\(MD_{h_M}(u_0)\to\infty\) and concentration of
\(W_{h_M}(u_0,\cdot)/D_{h_M}(u_0)\) at \(u_0\).
\emph{Local uniform integrability.} On the compact support
\(\{d_\varepsilon(u_0,v)\le h_M\}\), in particular
\(d_{\mathrm{ref},\varepsilon}^\ell(u_0,v)\le h_M\), and
Lemma~\ref{lem:reference-aware-comparison} gives
\(\kappa_\varepsilon^\ell(u_0,\cdot)\le e^{h_M/2}\); together with
(S6) and Donsker--Varadhan
(\eqref{eq:local-intrinsic-first-moment}--\eqref{eq:local-euclidean-first-moment}),
this yields local uniform integrability of \(\Delta X_{i+1}\) under
the normalised kernel measure.
\emph{Numerator.} Continuity of \(v\mapsto\int\delta\,\eta_i^\mu(d\delta\mid v)\)
at \(u_0\) (from (S2)) plus the uniform integrability above give NW
consistency of the numerator.
\emph{Conclusion.} The ratio converges to
\(\beta^{-1}\int\delta\,\eta_i^\mu(d\delta\mid u_0)=b_i^*(t_i,x_i;u_0)\).
\end{proof}

\begin{lemma}[Gaussian reference-aware boundary drift]
\label{lem:gaussian-boundary-drift}
Under (S1)--(S6), with the Gaussian choice
\(W_h(u_0,v)=\exp(-d_\varepsilon^\ell(u_0,v)^2/(2h^2))\), the
empirical boundary drift satisfies
\[
\widehat b_i^{M,\mathrm G}(t_i,x_i)
:=
\lim_{t\downarrow t_i}\widehat b_i^{\varepsilon,M,\mathrm G}(t,x_i)
\longrightarrow
b_i^*(t_i,x_i;u_0)
\quad\text{in probability,}
\]
a.s.\ under (S5).
\end{lemma}

\begin{proof}
The diagonal limit and the cancellation of the \(\Phi_i^\varepsilon\)
factors are as in Lemma~\ref{lem:compact-boundary-drift}, leaving
the Gaussian NW estimator with product kernel \(W_{h_M}\) and
reference factor
\(\Gamma^{\mathrm{ref}}_{h}(u_0,v)=\exp(-r(u_0,v)^2/(2h^2))\),
\(r(u_0,v):=d_{\mathrm{ref},\varepsilon}^\ell(u_0,v)\).

\emph{Reference factor absorbs the comparison factor.} The growth
of \(\kappa_\varepsilon^\ell(u_0,v)^2\le e^{r(u_0,v)}\) is uniformly
absorbed by the Gaussian penalty:
\begin{equation}
\label{eq:gaussian-absorption}
\Gamma^{\mathrm{ref}}_{h}(u_0,v)\,\kappa_\varepsilon^\ell(u_0,v)^2
\le
\exp\!\Bigl(-\tfrac{r(u_0,v)^2}{2h^2}+r(u_0,v)\Bigr)
\le
\exp\!\Bigl(\tfrac{h^2}{2}\Bigr),
\end{equation}
by maximising the exponent over \(r\ge 0\). Hence the reference
factor times the comparison factor is uniformly bounded by
\(\exp(h_M^2/2)\to 1\).

\emph{Normalised kernel concentration.} The normalised localiser
\(\nu_h^{u_0}(dv):=W_h(u_0,v)\,\mu_U(dv)/D_h(u_0)\) concentrates at
\(u_0\) by
Lemma~\ref{lem:kernel-denominator-localisation}: for every
\(\rho>0\), \(\nu_h^{u_0}(\{d_\varepsilon>\rho\})\to 0\) as
\(h\downarrow 0\).

\emph{Inner region \(\{d_\varepsilon(u_0,v)\le\rho\}\).} For
\(\rho\) small enough this is contained in the (S6)-neighbourhood
\(C_\rho(u_0)\); Donsker--Varadhan
gives~\eqref{eq:local-runtime-first-moment}, a uniform first-moment
bound in the runtime/query Mahalanobis norm
\(|\cdot|_{C_{i,\varepsilon}^\ell(u_0)^{-1}}\), and the Euclidean
implication on the active leaf is~\eqref{eq:local-euclidean-first-moment}.
The NW numerator restricted to this region is thus uniformly
integrable against \(\nu_h^{u_0}\) and converges by weak continuity
of \(v\mapsto\int\delta\,\eta_i^\mu(d\delta\mid v)\) at \(u_0\) to
\(\int\delta\,\eta_i^\mu(d\delta\mid u_0)\).

\emph{Outer region \(\{d_\varepsilon(u_0,v)>\rho\}\): direct
exponential tail.} The integrand in the numerator is dominated by
\(\kappa_\varepsilon^\ell(u_0,v)\,m_{\mathrm{intr}}(v)\) with
\(m_{\mathrm{intr}}(v):=\int|\delta|_{C_i^\ell(v)^+}\,\eta_i^\mu(d\delta\mid v)\).
Using \(d_{\mathrm{ref},\varepsilon}^\ell(u_0,v)\le d_\varepsilon(u_0,v)\)
and the reference-comparison
estimate~\eqref{eq:reference-comparison-factor},
\[
W_h(u_0,v)\,\kappa_\varepsilon^\ell(u_0,v)
\lesssim
\exp\!\Bigl(-\tfrac{d_\varepsilon(u_0,v)^2}{2h^2}
+\tfrac12 d_\varepsilon(u_0,v)\Bigr)
\le
\exp\!\Bigl(-\tfrac{\rho^2}{4h^2}\Bigr)
\]
on \(\{d_\varepsilon>\rho\}\), for \(h\) small enough that
\(\tfrac{d_\varepsilon^2}{2h^2}-\tfrac12 d_\varepsilon\ge\tfrac{d_\varepsilon^2}{4h^2}\)
when \(d_\varepsilon>\rho\). The intrinsic moment is globally
integrable under \(\mu_U\):
\begin{equation}
\label{eq:global-intrinsic-moment}
\int m_{\mathrm{intr}}(v)\,\mu_U(dv)
=
\mathbb E_\mu\!\left[\int |\delta|_{C_i(U_i)^+}\,\eta_i^\mu(d\delta\mid U_i)\right]
<\infty,
\end{equation}
since under the conditioned reference \(k_i^R(\cdot\mid v)\) the
intrinsic Mahalanobis norm \(|\delta|_{C_i^\ell(v)^+}\) has
uniform-in-\(v\) exponential moments by construction, and finite
global entropy~\eqref{eq:finite-entropy-assumption} transfers these
moments via Donsker--Varadhan to \(\eta_i^\mu\) with a global bound.
Combining with the denominator lower bound
\(D_h(u_0)\gtrsim h^{q_{\mathrm{eff}}}\) of
Lemma~\ref{lem:kernel-denominator-localisation},
\begin{multline*}
\frac{\int_{\{d_\varepsilon>\rho\}}
W_h(u_0,v)\,\kappa_\varepsilon^\ell(u_0,v)\,m_{\mathrm{intr}}(v)\,\mu_U(dv)}
{D_h(u_0)}
\\
\le
\frac{\exp(-\rho^2/(4h^2))}{c\,h^{q_{\mathrm{eff}}}}
\int m_{\mathrm{intr}}\,d\mu_U
\longrightarrow 0
\end{multline*}
as \(h\downarrow 0\). The outer contribution to the NW numerator
therefore vanishes.

\emph{Conclusion.} Adding the inner and outer estimates, the NW
numerator converges to \(\int\delta\,\eta_i^\mu(d\delta\mid u_0)\);
the denominator stays bounded below by
Lemma~\ref{lem:kernel-denominator-localisation}; the ratio tends to
\(b_i^*(t_i,x_i;u_0)\) in probability, a.s.\ under~(S5).
\end{proof}

The interior consistency of
Proposition~\ref{prop:drift-compact-uniform} together with
Lemmas~\ref{lem:analytic-boundary-drift}--\ref{lem:gaussian-boundary-drift}
gives convergence on every compact subset of the open intervalwise
cylinder \emph{and} at the single boundary diagonal point
\((t_i,x_i)\), where the empirical boundary drift is defined as the
finite-sample diagonal limit
\(\widehat b_i^M(t_i,x_i):=\lim_{t\downarrow t_i}\widehat b_i^{\varepsilon,M}(t,x_i)\).

\subsection{Proof of the main statistical convergence theorem}

\begin{proof}[Proof of Theorem~\ref{thm:approximate-intervalwise-dynamics}]
Fix \(i\), the query \(u_0=(\mathbf x_i,y_{i+1})\), and a compact
\(K\Subset(t_i,t_{i+1})\times(x_i+\mathcal L_i)\) in the backward
stratum \(\{t\le t_{i+1}-\tau\}\).

\smallskip
\noindent\emph{Interior drift convergence (part~(i)).} Under
(S1)--(S5), decompose
\begin{equation}
\label{eq:drift-two-step-decomposition}
A_i^\varepsilon\nabla_x\log\widehat H_i^{\varepsilon,M,\mathrm G}
\;\xrightarrow{M\to\infty}\;
A_i^\varepsilon\nabla_x\log H_i^\varepsilon
\;\xrightarrow{\varepsilon\downarrow 0}\;
A_i\nabla_x\log H_i = b_i^*,
\end{equation}
the first arrow by Proposition~\ref{prop:drift-compact-uniform}
(statistical step at fixed \(\varepsilon\)), the second by
Proposition~\ref{prop:spectral-regularisation-leafwise-kernel}
(analytical step in \(\varepsilon\)). Picking any
\(\varepsilon_n\downarrow 0\),
Proposition~\ref{prop:gaussian-compact-uniform} at each
\(\varepsilon_n\) provides \(M_n\to\infty\) and admissible
\(h_{M_n}\downarrow 0\) with compact-uniform \(\tfrac1n\)-control of
\(\widehat H_i^{(n)}\) and its active gradient in the stated
stochastic mode; composing the two arrows
in~\eqref{eq:drift-two-step-decomposition} gives
\(\sup_K|\widehat b_i^{(n)}-b_i^*|\to 0\) in probability (a.s.\
under~(S5)). The strict interior \(\alpha<\beta\) gap furnishes the
Gaussian-ratio domination of
Lemma~\ref{lem:local-gaussian-domination} and no boundary
admissibility is invoked.

\smallskip
\noindent\emph{Boundary-diagonal drift convergence (part~(ii)).}
Under (S1)--(S6), the diagonal-limit value at \((t_i,x_i)\) is
supplied directly by the boundary
Section~\ref{app:boundary-drift}: the analytic limit by
Lemma~\ref{lem:analytic-boundary-drift} (which uses only finite
conditional entropy at \(u_0\), automatic from
disintegration~\eqref{eq:entropy-chain-rule} for \(\mu_i\)-a.e.\
\(u_0\)), and the empirical limit by either
Lemma~\ref{lem:compact-boundary-drift} (compact-kernel regression)
or Lemma~\ref{lem:gaussian-boundary-drift} (Gaussian-kernel
regression), both using (S6) for the local entropy admissibility
and (S1) for the small-ball denominator non-collapse via
Lemma~\ref{lem:kernel-denominator-localisation}. Both deliver
\(\widehat b_i^{(n)}(t_i,x_i)\to b_i^*(t_i,x_i;u_0)\) in
probability, a.s.\ under~(S5), with the empirical boundary drift
defined as the finite-sample diagonal limit
\(\widehat b_i^{(n)}(t_i,x_i):=\lim_{t\downarrow t_i}\widehat b_i^{\varepsilon_n,M_n}(t,x_i)\).

\end{proof}

\begin{remark}[Uniformity with respect to the conditioning point]
If \((\mathbf x_i,y_{i+1})\) ranges in a compact set
\(C\Subset C_i\) where \(\pi_i\) is bounded away from \(0\) and the
eigenstructure of \(A_i\) is continuous, the Gaussian-domination
constants of Lemma~\ref{lem:local-gaussian-domination} can be chosen
uniformly in \((\mathbf x_i,y_{i+1})\in C\), and the same diagonal
triple \((\varepsilon_n,M_n,h_{M_n})\) delivers the conclusions of
Theorem~\ref{thm:approximate-intervalwise-dynamics} uniformly on
\(C\).
\end{remark}

\begin{remark}[Reference-weighted Hilbert structure]
\label{rem:reference-weighted-hilbert}
The natural global Hilbert structure on the active subspace
\(\mathcal L_i\) of increments is not \(L^2(\mathcal L_i,d\ell_i)\)
but the reference-weighted space \(L^2(k_i^R)\) with
\(k_i^R(d\delta):=p_i(t_i,x_i;t_{i+1},\delta)\,d\ell_i(\delta)\): if
\(f_i\in L^2(k_i^R)\), the backward reference semigroup is
contractive in the corresponding norm, and Gaussian smoothing gives
active-gradient estimates in the same weighted spaces on compact
subintervals of \((t_i,t_{i+1})\). The consistency theorem above
does not rely on this stronger global integrability: it uses only compact convergence for the heat drift.
\end{remark}

\section{Predictive validation: energy score and conditional kernel score}
\label{sec:appendix-predictive-energy-score}

This appendix describes the predictive validation score used to select and
compare model configurations. The score is designed to test the conditional
predictive law generated by the model. The model is first given a real past
trajectory as memory; it then generates several possible future continuations,
which are compared with the realised future of the validation path.

Let
\[
Z_1,\dots,Z_n \in \R^q
\]
be the validation path in the coordinates used by the model. In the white-only
case one may simply take \(Z_i=X_i\). In a joint state-volatility model,
\(Z_i\) may instead denote the same-grid joint state, for instance
\[
Z_i=(X_i,Y_i),
\]
where \(Y_i\) is the volatility-factor state. The definition below is
independent of this choice.

Fix a memory length \(p_{\mathrm{mem}}\ge 1\), a predictive horizon
\(K\ge 1\), a number \(L\ge 1\) of Monte Carlo predictive continuations, and a
validation stride \(s\ge 1\). For every admissible validation index \(i\), define
the real memory block
\[
\mathcal H_i
:=
(Z_{i-p_{\mathrm{mem}}+1},\dots,Z_i)
\in (\R^q)^{p_{\mathrm{mem}}},
\]
and the realised future block
\[
Z_{i+1:i+K}
:=
(Z_{i+1},\dots,Z_{i+K})
\in (\R^q)^K.
\]
We identify the future block with its time-major vectorisation
\[
z_i^{\mathrm{obs}}
:=
\operatorname{vec}_{\mathrm{time}}(Z_{i+1},\dots,Z_{i+K})
\in \R^{Kq}.
\]

Conditionally on the real memory \(\mathcal H_i\), the fitted generative model
produces \(L\) independent predictive continuations
\[
\widehat Z^{(\ell)}_{i,1:K}
=
\bigl(
\widehat Z^{(\ell)}_{i,1},
\dots,
\widehat Z^{(\ell)}_{i,K}
\bigr),
\qquad
\ell=1,\dots,L.
\]
For \(K>1\), the continuation is generated autoregressively. After each
generated step, the model memory is shifted and the generated value is inserted
as the newest state. Equivalently, setting
\[
\widehat{\mathcal H}^{(\ell)}_{i,0}:=\mathcal H_i,
\]
one samples recursively
\[
\widehat Z^{(\ell)}_{i,r}
\sim
\widehat P_\theta
\bigl(
\cdot \mid \widehat{\mathcal H}^{(\ell)}_{i,r-1}
\bigr),
\qquad r=1,\dots,K,
\]
and then updates
\[
\widehat{\mathcal H}^{(\ell)}_{i,r}
=
\operatorname{shift}
\bigl(
\widehat{\mathcal H}^{(\ell)}_{i,r-1},
\widehat Z^{(\ell)}_{i,r}
\bigr).
\]
We again use the time-major vectorisation
\[
\widehat z^{(\ell)}_i
:=
\operatorname{vec}_{\mathrm{time}}
\bigl(
\widehat Z^{(\ell)}_{i,1},\dots,\widehat Z^{(\ell)}_{i,K}
\bigr)
\in \R^{Kq}.
\]

\subsection*{Energy score and aggregation}

The per-window score is the multivariate energy score
\begin{equation}
\label{eq:predictive-energy-local-score}
\widehat{\mathrm{ES}}_i
:=
\frac{1}{L}\sum_{\ell=1}^{L}
\bigl\|\widehat z^{(\ell)}_i - z_i^{\mathrm{obs}}\bigr\|
-
\frac{1}{2L^{2}}
\sum_{\ell=1}^{L}\sum_{\ell'=1}^{L}
\bigl\|\widehat z^{(\ell)}_i - \widehat z^{(\ell')}_i\bigr\|,
\end{equation}
and the global score is the empirical mean over admissible validation
windows~\(\mathcal I\):
\begin{equation}
\label{eq:predictive-energy-population-score}
\widehat{\mathrm{ES}}
:=
\frac{1}{|\mathcal I|}\sum_{i\in\mathcal I}
\widehat{\mathrm{ES}}_i.
\end{equation}

\subsection*{Conditional transition-kernel score}
\label{subsec:appendix-conditional-kernel-score}

The energy
score~\eqref{eq:predictive-energy-local-score}--\eqref{eq:predictive-energy-population-score}
tests the \emph{global} generative quality of the model on
increments. We complement it with a Gaussian-kernel score that
targets a different object: the quality of the approximation of the
\emph{conditional} transition kernel itself,
\(K_i^\mu(u,\cdot):=\mathcal L_\mu(\Delta X_{i+1}\mid \mathbf U_i=u)\),
of which \(\widehat K_i^{(n)}(u,\cdot)\) is the NW estimator of
Theorem~\ref{thm:approximate-intervalwise-dynamics}. The two scores
play complementary roles: the energy score evaluates marginal/aggregate
generative fit, the conditional kernel score evaluates the statistical
operator that the method estimates.

Given a conditioning query \(u_q\) from the validation set, let
\[
\widehat K_i^{\mathrm{train}}(u_q,\cdot)
:=
\sum_{m=1}^M w_m(u_q)\,\delta_{Z_m}
\]
denote the empirical transition kernel produced by the runtime
regression weights \(w_m(u_q):=\omega_{i,m}^{M,\mathrm G}(u_q)\) of
Section~\ref{subsec:statistical-estimator} on the training set,
\(Z_m:=\Delta X_{i+1}^{(m)}\) being the future increment associated
with the \(m\)-th training sample. Let \(z_q\) denote the realised
validation increment at \(u_q\). The Gaussian-kernel conditional
score at \(u_q\) is
\begin{equation}
\label{eq:cmmd-conditional-score}
S_k(u_q,z_q)
:=
\sum_{m,\ell=1}^M w_m(u_q)\,w_\ell(u_q)\,k(Z_m,Z_\ell)
-2\sum_{m=1}^M w_m(u_q)\,k(Z_m,z_q),
\end{equation}
with characteristic Gaussian kernel
\begin{equation}
\label{eq:cmmd-kernel}
k(a,b)
:=
\exp\!\Bigl(-\tfrac{|S^{-1}(a-b)|^2}{2\sigma_k^2}\Bigr),
\end{equation}
\(S\) the normalisation operator on the increment coordinates and
\(\sigma_k>0\) a kernel bandwidth. The global validation score is
the robust upper quantile over a validation query set \(\mathcal Q\),
\begin{equation}
\label{eq:cmmd-aggregate}
\mathcal S_k^{(\alpha)}
:=
\operatorname{Quantile}_{\alpha}\!\bigl\{S_k(u_q,z_q):u_q\in\mathcal Q\bigr\},
\qquad\alpha\in(0,1),
\end{equation}
typically \(\alpha=0.9\). \(S_k(u_q,z_q)\) is a kernel discrepancy
between \(\widehat K_i^{\mathrm{train}}(u_q,\cdot)\) and the
realised validation transition; the characteristic Gaussian kernel
metrises weak convergence of probability measures, so the score
provides a practical validation criterion for the statistical
approximation of the conditional transition operator itself.

\subsection*{Enriched score for coupled \((X,Y)\) validation}

For coupled validation we use the same global score template
applied to features that combine positions with a packed squared-increment
descriptor. Let \(K\ge 1\) be the predictive horizon as above and let
\(\Delta^{\otimes 2}_r X\in\R^{d_X^2}\) denote the time-\(r\) packed
squared-increment block, i.e.\ the vectorisation of
\(\Delta X_r\Delta X_r^{\top}\). Define a per-step feature
\[
\phi_r(X)
:=
\begin{pmatrix}
X_r \\[3pt]
\bigl[\Delta^{\otimes 2}X\bigr]_r
\end{pmatrix}
\in\R^{d_X+d_{\mathrm{inc},X}},
\qquad d_{\mathrm{inc},X}:=d_X^2,
\]
with an analogous definition for \(Y\). Stack-rescale by component
standard deviations
\[
\phi_r(X)
:=
\begin{pmatrix}
(\sigma_{\mathrm{pos}}\sqrt{d_X})^{-1}\, X_r \\[3pt]
(\sigma_{\mathrm{inc}}\sqrt{d_{\mathrm{inc}}})^{-1}\,
\bigl[\Delta^{\otimes 2} X\bigr]_r
\end{pmatrix}
\in \R^{d_X + d_{\mathrm{inc}}},
\]
where \(\sigma_{\mathrm{pos}}\) and \(\sigma_{\mathrm{inc}}\) are the
root-mean-square standard deviations of, respectively, the position
components and the packed squared-increment components over the training set.
The coupled enriched comparison vector is
\[
\tilde z_i
:=
\operatorname{vec}_{\mathrm{time}}
\Bigl(
\bigl[\phi_r(X),\,\phi_r(Y)\bigr]_{r=1,\dots,K}
\Bigr)
\in \R^{K(d_X + d_{\mathrm{inc},X} + d_Y + d_{\mathrm{inc},Y})},
\]
and the enriched global score \(\widehat{\mathrm{ES}}_{\mathrm{enriched}}\)
is obtained by substituting \(\tilde z_i\) and \(\widetilde{\widehat z}^{(\ell)}_i\)
for \(z_i^{\mathrm{obs}}\) and \(\widehat z^{(\ell)}_i\)
in~\eqref{eq:predictive-energy-local-score}--\eqref{eq:predictive-energy-population-score}.

\section{Entropic reference selection and validation}
\label{sec:appendix-entropic-reference-selection-and-validation}

The construction of
Sections~\ref{subsec:volatility-informed-freezing}
and~\ref{subsec:joint-tr-sbts-generation} requires, at each latent
level, a reference family under which the data-generating law has
finite relative entropy on every coarse interval. We use a single
finite-sample surrogate of that condition in two roles: selecting a
reference family for a given latent level, and validating the
closed-loop emission of the level. 

\paragraph{Per-step empirical Gaussian negative log-likelihood.}
Fix a coarse step \([t,t+\Delta t]\) and consider, at the relevant
latent level, a candidate zero-mean Gaussian reference for the
level's state increment \(\Delta\zeta_t\in\R^d\),
\[
R_{\Sigma,\Delta t}(d\zeta)
\;=\;
\mathcal N\!\bigl(0,\;\Sigma\,\Delta t\bigr)(d\zeta),
\qquad \Sigma\in\mathrm{Sym}^+_d,
\]
with density \(r_{\Sigma,\Delta t}\): the entropic surrogate is
evaluated against the local-martingale part of the reference, whose
drift is identically zero. Given an observed increment
\(\Delta\zeta^{\mathrm{obs}}_{p,t}\) on a training path \(p\), the
per-step Gaussian negative log-likelihood under
\(R_{\Sigma,\Delta t}\) is
\begin{multline}
\label{eq:entropic-per-step-empirical-nll}
\mathcal S^{(p)}_t(\Sigma)
\;:=\;
-\log r_{\Sigma,\Delta t}\!\bigl(\Delta\zeta^{\mathrm{obs}}_{p,t}\bigr)
\\
=\;
\tfrac{1}{2}\log\det\!\bigl(\Sigma\,\Delta t\bigr)
+
\tfrac{1}{2}
\bigl(\Delta\zeta^{\mathrm{obs}}_{p,t}\bigr)^{\!\top}
\!\bigl(\Sigma\,\Delta t\bigr)^{-1}
\Delta\zeta^{\mathrm{obs}}_{p,t}
+\;
\tfrac{d}{2}\log(2\pi).
\end{multline}
The first \(\Sigma\)-dependent term in
\eqref{eq:entropic-per-step-empirical-nll} is the log-volume of the
reference; the second is the squared Mahalanobis norm of the
observed increment in the metric induced by \(\Sigma\). The
\(\Sigma\)-independent additive constant \(\tfrac{d}{2}\log(2\pi)\)
cancels in any comparison between candidates evaluated on the same
\(\Delta\zeta^{\mathrm{obs}}_{p,t}\).

When the candidate family produces rank-deficient covariances by
construction (for instance a directional rank-one parametrisation),
\eqref{eq:entropic-per-step-empirical-nll} is evaluated on the
spectrally floored covariance \(\Sigma^{(f),\varepsilon}_t\)
of~\eqref{eq:psd-floor}, so the inverse exists and the perpendicular
component of \(\Delta\zeta^{\mathrm{obs}}_{p,t}\) enters as the
leaf-aware penalty
\(\tfrac{1}{2}\|\Delta\zeta^{\mathrm{obs},\perp}_{p,t}\|^2/(\varepsilon\Delta t)\)
discussed in
Section~\ref{subsec:stable-coefficient-database}. Comparison across
families with different ranks accordingly uses the same
\(\varepsilon\)-floor on all candidates.

\paragraph{Selecting a reference family for a latent level.}
Let \(\{\Sigma^{(f)}_t\}_f\) be a finite family of candidate frozen
reference trajectories indexed by a hyperparameter \(f\) (smoothing
length, history window, rank constraint, parametric form). For each
training path \(p\), aggregate
\eqref{eq:entropic-per-step-empirical-nll} along the path's step set
\(\mathcal T_p\):
\[
\overline{\mathcal S}_p\!\bigl(\Sigma^{(f)}\bigr)
\;:=\;
\frac{1}{|\mathcal T_p|}
\sum_{t\in\mathcal T_p}\mathcal S^{(p)}_t\!\bigl(\Sigma^{(f),\varepsilon}_t\bigr).
\]
The selection rule is the cross-path \(q_\alpha\)-quantile of these
path-wise scores,
\begin{equation}
\label{eq:entropic-reference-selection-rule}
\widehat f
\;:=\;
\operatorname*{arg\,min}_{f}\;
q_\alpha\!\left(\bigl\{\overline{\mathcal S}_p\!\bigl(\Sigma^{(f)}\bigr)\bigr\}_p\right),
\qquad \alpha\in(0,1),
\end{equation}
typically with \(\alpha=0.9\). The quantile suppresses tail paths on
which any candidate would be a poor description and orders families
by how well their reference explains the bulk of the observed
increments; no closed-loop generation is required at this stage.

\paragraph{Forward entropic validation of a generated latent level.}
For a held-out validation path \(p\), let
\(\Sigma^{\mathrm{gen}}_{p,t}\in\mathrm{Sym}^+_d\) denote the
closed-loop generation of the level on \(p\), PSD-projected and
spectrally floored at the common \(\varepsilon\)
via~\eqref{eq:psd-projection}--\eqref{eq:psd-floor}. The per-path
entropic-adherence score is the same NLL aggregate as for selection,
applied now to the closed-loop generated reference,
\begin{equation}
\label{eq:entropic-validation-score}
\overline{\mathcal S}^{\mathrm{val}}_p\!\bigl(\Sigma^{\mathrm{gen}}\bigr)
\;:=\;
\frac{1}{|\mathcal T_p|}
\sum_{t\in\mathcal T_p}
\mathcal S^{(p)}_t\!\bigl(\Sigma^{\mathrm{gen},\varepsilon}_{p,t}\bigr),
\end{equation}
the empirical Gaussian negative log-likelihood of the observed
increment under the level's own generated reference, averaged over
the validation step set. The cross-path scalar score is the
\(q_\alpha\)-quantile of
\(\{\overline{\mathcal S}^{\mathrm{val}}_p\}_p\), with the same
\(\alpha\) and the same \(\varepsilon\)-floor used at selection.

The two scores
\eqref{eq:entropic-reference-selection-rule}--\eqref{eq:entropic-validation-score}
share a single objective: the empirical entropy of the
data-generating law under the level's Gaussian model. Selection
applies it to candidate frozen references, validation applies it to
the closed-loop generated reference on held-out paths, and no
intermediate proxy enters between the two stages.

\nocite{*}
\bibliographystyle{plain}
\bibliography{biblio}
% Suggested keys used or reserved in this section:
% \cite{hamdouche2023sbts}
% \cite{alouadi2025robustsbts}
% \cite{demarco2026jumpssbts}

\end{document}